\documentclass[twoside]{article}

\usepackage[accepted]{aistats2023}
\usepackage[utf8]{inputenc} % allow utf-8 input
\usepackage[T1]{fontenc}    % use 8-bit T1 fonts
\usepackage{hyperref}       % hyperlinks
\usepackage{url}            % simple URL typesetting
\usepackage{booktabs}       % professional-quality tables
\usepackage{amsfonts}       % blackboard math symbols
\usepackage{nicefrac}       % compact symbols for 1/2, etc.
\usepackage{microtype}      % microtypography
\usepackage{xcolor}         % colors
\usepackage{amsthm}
\usepackage{amssymb}

\usepackage{mathrsfs}
\usepackage{enumerate}
\usepackage{graphicx} % more modern
\usepackage{caption}

\usepackage{tablefootnote}
\usepackage{color}
\usepackage{xcolor}
\usepackage{natbib}
\newtheorem{thm}{Theorem}[section]
\newtheorem{lem}{Lemma}[section]
\newtheorem{cor}{Corollary}[section]
\newtheorem{prop}{Proposition}[section]
\newtheorem{asmp}{Assumption}[section]
\newtheorem{defn}{Definition}[section]

\newtheorem{rem}{Remark}[section]

\hypersetup{
	colorlinks=true,
	filecolor=blue,
	citecolor = blue,
	urlcolor=cyan,
}
\usepackage{amsmath,amsfonts,bm}

\def\seps{\boldsymbol{\eps}}

\def\floor#1{\lfloor #1 \rfloor}
\def\1{\bm{1}}

\def\eps{{\epsilon}}

\def\rr{{\textnormal{r}}}

\DeclareMathAlphabet{\mathsfit}{\encodingdefault}{\sfdefault}{m}{sl}
\SetMathAlphabet{\mathsfit}{bold}{\encodingdefault}{\sfdefault}{bx}{n}

\def\sI{\boldsymbol{I}}
\def\sG{\boldsymbol{G}}
\def\sD{\boldsymbol{D}}
\def\mthH{\mathrm{H}}

\def\su{\boldsymbol{u}}
\def\sU{\boldsymbol{U}}
\def\sv{\boldsymbol{v}}
\def\su{\boldsymbol{u}}

\def\sG{\boldsymbol{G}}

\def\AM{{\mathcal A}}
\def\BM{{\mathcal B}}

\def\DM{{\mathcal D}}

\def\FM{{\mathcal F}}

\def\HM{{\mathcal H}}

\def\LM{{\mathcal L}}
\def\NM{{\mathcal N}}
\def\OM{{\mathcal O}}

\def\PM{{\mathcal P}}
\def\SM{{\mathcal S}}
\def\TM{{\mathcal T}}
\def\UM{{\mathcal U}}

\def\ZM{{\mathcal Z}}

\def\RB{{\mathbb R}}
\def\EB{{\mathbb E}}

\def\PB{{\mathbb P}}

\def\varepsi{\mbox{\boldmath$\varepsilon$\unboldmath}}

\def\De{\mbox{\boldmath$\Delta$\unboldmath}}

\def\argmax{\mathop{\rm argmax}}

\def\tr{\mathrm{tr}}

\def\diag{\mathrm{diag}}

\newcommand{\softmax}{\mathrm{softmax}}

\newcommand{\Var}{\mathrm{Var}}

\def\GammaSym{{\boldsymbol \Gamma}}

\def\A{{\boldsymbol A}}

\def\B{{\boldsymbol B}}
\def\bb{{\boldsymbol b}}
\def\C{{\boldsymbol C}}

\def\D{{\boldsymbol D}}

\def\e{{\boldsymbol e}}

\def\h{{\boldsymbol h}}
\def\G{{\boldsymbol G}}
\def\H{{\boldsymbol H}}
\def\I{{\boldsymbol I}}

\def\M{{\boldsymbol M}}

\def\N{{\boldsymbol N}}

\def\PP{{\boldsymbol P}}
\def\Q{{\boldsymbol Q}}

\def\R{{\boldsymbol R}}
\def\rr{{\boldsymbol r}}

\def\U{{\boldsymbol U}}

\def\V{{\boldsymbol V}}
\def\v{{\boldsymbol v}}
\def\W{{\boldsymbol W}}

\def\X{{\boldsymbol X}}
\def\x{{\boldsymbol x}}
\def\Y{{\boldsymbol Y}}
\def\y{{\boldsymbol y}}
\def\Z{{\boldsymbol Z}}

\def\0{{\boldsymbol 0}}
\def\1{{\boldsymbol 1}}
\def\DDelta{{\mathbf \Delta}}
\def\BPi{{\mathbf \Pi}}
\def\BDelta{{\bar{\mathbf{\Delta}}}}
\def\Bphi{{\boldsymbol \phi}}
\def\eeta{{\widetilde{\eta}}}
\def\tim{{\widetilde{m}}}
\def\gap{{\mathrm{gap}}}
\def\TOM{{\widetilde{\OM}}}
\def\TQ{{\widetilde{\Q}}}
\def\TA{{\widetilde{\A}}}
\def\TZ{{\widetilde{\Z}}}
\def\TTM{{\widetilde{\TM}}}
\def\TDe{{\widetilde{\De}}}
\def\TDe{{\widetilde{\De}}}
\def\Tphi{{\widetilde{\Bphi}}}
\def\Tpi{{\widetilde{\pi}}}
\def\VQ{\mathrm{Var}_\Q}
\def\TVQ{\widetilde{\mathrm{Var}}_\Q}
\def\VV{\mathrm{Var}_\V}
\def\Tr{{\lfloor Tr \rfloor}}
\def\BC{{\mathsf{C}([0, 1], \RB^d)}}
\def\BD{{\mathsf{D}([0, 1], \RB^d)}}
\def\BDD{{\mathsf{D}([0, 1], \RB^D)}}
\def\BCD{{\mathsf{C}([0, 1], \RB^D)}}
\def\BDone{{\mathsf{D}([0, 1], \RB)}}
\def\BDM{\DM([0, 1], \RB^d)}
\def\Bpsi{{\boldsymbol \psi}}
\newcommand{\ssum}[3]{\sum\limits_{{#1}={#2}}^{#3}}

\begin{document}


\onecolumn

\runningauthor{Xiang Li, Wenhao Yang, Jiadong Liang, Zhihua Zhang, Michael I. Jordan}

\twocolumn[

\aistatstitle{
A Statistical Analysis of Polyak-Ruppert Averaged Q-learning}

\aistatsauthor{ Xiang Li \And Wenhao Yang \And  Jiadong Liang }

\aistatsaddress{ \url{lx10077@pku.edu.cn}\\ Peking University \And  \url{yangwenhaosms@pku.edu.cn} \\ Peking University \And
\url{jdliang@pku.edu.cn} \\ Peking University 
}

\aistatsauthor{ Zhihua Zhang \And Michael I. Jordan }

\aistatsaddress{
\url{zhzhang@math.pku.edu.cn}\\ Peking University  \And
\url{jordan@cs.berkeley.edu} \\UC Berkeley
}
]

\begin{abstract}%
We study Q-learning with Polyak-Ruppert averaging in a discounted Markov decision process in synchronous and tabular settings. 
Under a Lipschitz condition, we establish a functional central limit theorem for the averaged iteration $\bar{\boldsymbol{Q}}_T$ and show that its standardized partial-sum process converges weakly to a rescaled Brownian motion.
The functional central limit theorem implies a fully online inference method for reinforcement learning.
 Furthermore, we show that $\bar{\boldsymbol{Q}}_T$ is the regular asymptotically linear (RAL) estimator for the optimal Q-value function $\boldsymbol{Q}^*$ that has the most efficient influence function. 
We present a nonasymptotic analysis for the $\ell_{\infty}$ error, $\mathbb{E}\|\bar{\boldsymbol{Q}}_T-\boldsymbol{Q}^*\|_{\infty}$, showing that it matches the instance-dependent lower bound for polynomial step sizes.
Similar results are provided for entropy-regularized Q-learning without the Lipschitz condition.
% In short, our theoretical analysis shows that averaged Q-learning is statistically efficient. 
\end{abstract}

\section{INTRODUCTION}
Q-learning~\citep{watkins1989learning}, as a model-free approach seeking the optimal Q-function of a Markov decision process (MDP), is perhaps the most widely deployed algorithm in reinforcement learning (RL)~\citep{sutton2018reinforcement}.
Unlike policy evaluation where the underlying structure is linear in nature and the goal is essentially to solve a linear system, Q-learning is nonlinear, nonsmooth and nonstationary.
%\footnote{Here the nonstationarity means the maintained policy $\pi_t$ changes with iteration $t$. By contrast, $\pi_t \equiv \pi_b$ for all iterations $t$ in policy evaluation where $\pi_b$ is the target policy.}
Theoretical analysis for Q-learning ranges from asymptotic convergence~\citep{singh1993convergence,tsitsiklis1994asynchronous,borkar2000ode,szepesvari1998asymptotic} to nonasymptotic rates~\citep{even2003learning,beck2012error,chen2020finite,li2021q,li2020sample}.
Variants of Q-learning~\citep{lattimore2014near,sidford2018near,sidford2018variance,wainwright2019variance} have been proposed that achieve the minimax lower bound of sample complexity established in~\citep{azar2013minimax}. 

On the other hand, Q-learning can be viewed through the lens of stochastic approximation (SA)~\citep{konda2002actor}, a general iterative framework for solving root-finding problems~\citep{robbins1951stochastic}.
It is a particular instance of SA that targets the Bellman fixed-point equation, $\TM Q^* = Q^*$, where $\TM$ is the population Bellman operator (see Eq.~\eqref{eq:T} for the definition).

The last-iterate behavior of Q-learning has been analyzed thoroughly within the nonlinear SA framework. In particular, on the asymptotic side, the ODE approach~\citep{kushner2003stochastic,abounadi2002stochastic,borkar2009stochastic,gadat2018stochastic,borkar2021ode} establishes a functional central limit theorem (functional CLT), showing that the interpolated process that connects rescaled last iterates converges weakly to the solution of a specific SDE.
From the nonasymptotic side, specific nonlinear SA convergence analyses have been tailored for Q-learning, capturing its nonasymptotic convergence rate~\citep{chen2020finite,chen2021lyapunov,qu2020finite}.

An important gap in this literature is the behavior of Q-learning under averaging, specifically Polyak-Ruppert averaging~\citep{polyak1992acceleration}.  Polyak-Ruppert averaging provides a general tool for stabilizing and accelerating SA algorithms.  It is known to accelerate policy evaluation~\citep{mou2020linear,mou2020optimal} and exhibits superior empirical performance in various RL problems~\citep{lillicrap2016continuous,anschel2017averaged}.
However, a theoretical understanding of Q-learning with Polyak-Ruppert averaging is not yet available.

In this paper, we analyze averaged Q-learning in the setting of a discounted infinite-horizon MDP and in the synchronous setting where a generative model produces independent samples for all state-action pairs in every iteration~\citep{kearns2002sparse}.
We provide both asymptotic and nonasymptotic analyses.
On the asymptotic side, we establish an functional CLT for averaged Q-learning, showing that the partial-sum process, $\Bphi_T(r) := \frac{1}{\sqrt{T}}  \sum_{t=1}^{\Tr} (\Q_t - \Q^*)$, converges weakly to a rescaled Brownian motion, namely $\VQ^{1/2}\B_D(r)$, where $r \in [0, 1]$ is the fraction of data used, $\lfloor \cdot \rfloor$ is the floor function, $\VQ$ (see Eq.~\eqref{eq:opt-variance}) is the asymptotic variance, and $\B_D(\cdot)$ is a standard $D$-dimensional Brownian motion on $[0, 1]$. 
Such a functional result for partial-sum processes has not been presented previously in the RL literature.
This allows us to construct an asymptotically pivotal statistic using information from the whole function $\Bphi_T(\cdot)$ (see Proposition~\ref{prop:pivotal}).
This obviates the need to estimate the asymptotic variance in providing asymptotically valid confidence intervals for $\Q^*$, which is required by~\citep{chen2020statistical,zhu2021online,hao2021bootstrapping,shi2020statistical,khamaru2022instance}.
It opens a door to online statistical inference for RL.

As a complementary result, we establish a semiparametric efficiency lower bound for any regular asymptotically linear (RAL) estimator (see Definition~\ref{def:RAL} for details) of the optimal Q-value function $\Q^*$. 
Given the $r$-th fraction of data, we further show that $\Bphi_T(r)$ is the most efficient RAL estimator with the smallest asymptotic variance, confirming its optimality in the asymptotic regime.

On the nonasymptotic side, we provide the first finite-sample error analysis of $\EB\|\bar{\Q}_T-\Q^*\|_{\infty}$ in the $\ell_{\infty}$-norm for both linearly rescaled and polynomial step sizes.
The error is dominated by $\OM(
\sqrt{\|\diag(\VQ)\|_{\infty}}\sqrt{\frac{\ln |\SM \times \AM|}{T}} )$ for polynomial step sizes given a sufficiently large $T$, which matches the instance-dependent lower bound established by~\citep{khamaru2021instance}.
This, together with the worst-case bound $\|\diag(\VQ)\|_{\infty} = \OM((1-\gamma)^{-3})$, implies that averaged Q-learning already achieves the optimal minimax sample complexity $\TOM\left( \frac{|\SM \times \AM|}{(1-\gamma)^3\varepsilon^2} \right)$ established by~\citep{azar2013minimax}.
Those lower bounds have only been shown to hold for a complicated variance-reduced version of Q-learning in this setting~\citep{wainwright2019variance,khamaru2021instance}.

From a technical perspective, we carefully decompose the partial sum process, $\Bphi_T(r)$, into several processes, each of which either has a nice structure (e.g., a sum of i.i.d.\ variables) or vanishes in the $\ell_{\infty}$-norm with probability one.
In this way, the nonasymptotic analysis reduces to careful examination of these diminishing rates.
To underpin the functional CLT, we develop a new lemma that shows that a certain residual error converges to zero in probability (see Lemma~\ref{lem:ignore-error}).
Generalizing an existing result from~\citet{lee2021fast,li2021statistical}, this technical lemma may be of independent interest.
Finally, while both our asymptotic and nonasymptotic analyses rely on a Lipschitz condition, stated in Assumption~\ref{asmp:gap},
we find that averaged Q-learning regularized by entropy achieves a similar functional CLT and instance-dependent bound without the Lipschitz assumption.

%Therefore, we develop a new analysis for averaged Q-learning that establishes its asymptotic statistical properties and provides a nonasymptotic analysis for finite samples.
%In particular, we develop a ``sandwich'' argument to decompose the error, $\bar{\Q}_T - \Q^*:=\frac{1}{T} \sum_{t=1}^T \Q_t - \Q^*$, into several terms, each of which either has a nice structure (e.g., a sum of i.i.d.\ variables) or vanishes in the $\ell_{\infty}$-norm with probability one.
%In this way, the nonasymptotic analysis reduces to careful examination of these diminishing rates.
%This analysis method may be of independent interest.

\paragraph{Paper organization.}
The remainder of this paper is organized as follows. 
In Section~\ref{sec:notation}, we introduce our notation and preliminaries on RL.
We present the formal functional CLT in Section~\ref{sec:asym-ave-Q} and the semiparametric efficiency lower bound in Section~\ref{sec:infolb}.
In Section~\ref{sec:nonasymptotic}, we show the nonasymptotic convergence bound and contrast it with previous work. 
We summarize our results and discuss future research directions in Section~\ref{sec:conclusion}.
We provide additional discussion of related work, and all proof details, in the appendix.

\section{PRELIMINARIES}
\label{sec:notation}

%$\rr_t \in \RB^D$ collects all random rewards generated at each $(s, a)$ and $\PP_t \in \RB^{D \times S}$ gathers all empirical transitions following the probability $\PP_{s, a}:=P(\cdot|s, a)$ starting from each $(s, a)$.
%Specifically, the $(s, a)$-th entry of $\rr_t$ is an independent copy of $R(s, a)$, while the $(s, a)$-th row of $\PP_t$ is a one-hot random vector with a single nonzero entry.

\paragraph{Discounted infinite-horizon MDPs.}
An infinite-horizon MDP is represented by a tuple $\mathcal{M}=(\SM, \AM, \gamma, P, R, r)$.
Here $\SM$ is the state space, $\AM$ is the action space, and $\gamma \in (0, 1)$ is the discount factor.
For simplicity, we define $D = |\SM \times \AM|=SA$.
We use $P\colon \SM \times \AM \to \Delta(\SM)$ to represent the probability transition kernel with $P(s'|s, a)$ the probability of transiting to $s'$ from a given state-action pair $(s, a)\in \SM \times \AM$.
Let $R\colon \SM \times \AM \to [0, \infty)$ stand for the random reward, i.e., $R(s, a)$ is the immediate reward collected in state $s \in \SM$ when action $a \in \AM$ is taken.
Unlike previous works~\citep{wainwright2019stochastic,li2021q} which assume the immediate reward $R$ is deterministic, we consider a general setting where $R$ itself is a random function with $r = \EB R$ the expected reward.
A policy $\pi$ maps each $s \in \SM$ to a probability over $\AM$.
In a $\gamma$-discounted MDP, a common objective is to maximize the expected long-term reward. 
For a given policy $\pi\colon \SM \to \Delta(\AM)$, the expected long-term reward is measured by the Q-function $Q^{\pi}$ defined as follows
\begin{gather*}
	Q^{\pi}(s, a) = \EB_{\pi} \left[ \sum_{t=0}^\infty \gamma^t r(s_t, a_t) \bigg| s_0=s, a_0=a\right],
\end{gather*}
and its companion value function is defined via $V^{\pi}(s) =\sum_{a \in \AM} \pi(a|s) Q^{\pi}(s, a)$.
Here $\EB_{\pi}(\cdot) $ is taken with respect to the randomness of the trajectory of the MDP induced by the policy $\pi$.
The optimal value function $V^*$ and optimal Q-function $Q^*$ are defined as $
V^*(s) = \max_{\pi} V^{\pi}(s)
\ \text{and} \
Q^*(s, a) = \max_{\pi} Q^{\pi}(s, a)$.
For simplicity, we employ the vectors
$\V^{\pi}, \V^* \in \RB^S$ and $ \Q^\pi, \Q^*, \Q_t, \bar{\Q}_t \in \RB^D$ to denote evaluations of the functions $V^\pi, V^*, Q^\pi, Q^*, Q_t, \bar{Q}_t$.

A generative model is assumed~\citep[cf.][]{kearns1999finite,sidford2018near,li2021q}.
In iteration $t$, we collect independent samples of rewards $r_t(s, a)$ and the next state $s_t(s, a) \sim P(\cdot |s, a)$ for every state-action pair $(s, a)\in \SM \times \AM$.
We summarize the observations into the reward vector $\rr_t = (r_t(s, a))_{(s, a)} \in \RB^{D}$ and the empirical transition matrix $\PP_t = (\e_{s_t(s, a)})_{(s, a)} \in \RB^{D\times S}$ with each row a one-hot vector.
We introduce the transition matrix $\PP \in \RB^{D \times S}$ to represent the probability transition kernel $P$, whose $(s, a)$-th row $\PP_{s, a}$ is a probability vector representing $P(\cdot|s, a)$.
The square probability transition matrix $\PP^{\pi} \in \RB^{D \times D}$ (resp. $\PP_{\pi} \in \RB^{S \times S}$) induced by the deterministic policy $\pi$ over the state-action pairs (resp. states) is
\begin{equation}
	\label{eq:P-matrix}
	\PP^\pi := \PP \boldsymbol{\Pi}^\pi
	\quad  \text{and}  \quad
	\PP_\pi := \boldsymbol{\Pi}^\pi\PP,
\end{equation}
where $\boldsymbol{\Pi}^\pi \in \RB^{S \times D}$ is a projection matrix associated with a given policy $\pi$:
\begin{equation}
	\label{eq:project-pi}
	\boldsymbol{\Pi}^{\pi}=
	\diag\{\pi(\cdot|1)^{\top} , \cdots, \pi(\cdot|S)^{\top}   \},
	%
	%\left(\begin{array}{llll}
		%\e_{\pi(1)}^{\top} & & & \\
		%& \e_{\pi(2)}^{\top} & & \\
		%& & \ddots & \\
		%& & & \e_{\pi(S)}^{\top}
		%\end{array}\right)
	\end{equation}
	where $\pi(\cdot|s) \in \RB^A$ is the policy vector at state $s$.
	
	\paragraph{Q-learning.}
The synchronous Q-learning algorithm maintains a Q-function vector, $\Q_t \in \RB^D$, for all $t \ge 0$ and updates its entries via the following update rule:
\begin{equation}
	\label{eq:Q-update}
		\Q_t = (1-\eta_t)\Q_{t-1} + \eta_t(\rr_t + \widehat{\TM}_t \Q_{t-1}),
%	\Q_t=(1-\eta_t) \Q_{t-1} + \eta_t \widehat{\TM}_t(\Q_{t-1}),
\end{equation}
where $\eta_t \in (0, 1]$ is the step size in the $t$-th iteration and $\widehat{\TM}_t: \RB^D \to \RB^D$ is the empirical Bellman operator constructed by samples collected in the $t$-th iteration:
\begin{equation}
	\label{eq:em-T}
	(\widehat{\TM}_t\Q)(s, a) = r_t(s, a) + \gamma \max_{a' \in \AM} \Q(s_t, a'),
	% \  \text{with} \
	% r_t(s, a) \sim R(s, a)
	% \  \text{and} \
	% s_t = s_t(s, a) \sim P(\cdot|s, a),
\end{equation} 
with $r_t(s, a) \sim R(s, a)$ and $s_t = s_t(s, a) \sim P(\cdot|s, a)$
for each state-action pair $(s, a)\in \SM \times \AM$.
In matrix form, $\widehat{\TM}_t \Q_{t-1} = \PP_t \V_{t-1}$ where $\V_{t-1}(s) = \max_{a } \Q_{t-1}(s ,a)$ is the greedy value.
Clearly, $\widehat{\TM}_t$ is an unbiased estimate of the Bellman operator $\TM: \RB^D \to \RB^D$ given by
\begin{equation}
	\label{eq:T}
	(\TM\Q)(s, a) = r(s, a) + \gamma \EB_{s' \sim P(\cdot|s, a)}\max_{a' \in \AM} \Q(s', a').
\end{equation}
The optimal $\Q^*$ is the unique fixed point of the Bellman operator, $\TM\Q^* = \Q^*$.
Let $\pi_t$ be the greedy policy w.r.t.\ $Q_t$; i.e., $\pi_t(s) \in \arg\max_{a \in \AM} Q_t(s, a)$ for $s \in \SM$ and $\pi^*$ the optimal policy.

\paragraph{Averaged Q-learning.}
\citet{ruppert1988efficient} and \citet{polyak1992acceleration} showed that averaging the iterates generated by a stochastic approximation (SA) algorithm has favorable asymptotic statistical properties. There is a line of work which has adapted Polyak-Ruppert averaging to the problem of policy evaluation in RL~\citep{bhandari2018finite,khamaru2021temporal,mou2020linear}.
Q-learning is different than policy evaluation due to the nonstationarity  (i.e., $\pi_t$ changes over time) and the nonlinearity of $\TM$.
The averaged Q-learning iterate has the form 
\[
\bar{Q}_T = \frac{1}{T} \sum_{t=1}^T Q_t
\]
with $\{ Q_t \}_{t \ge 0}$ updated as in Eq.~\eqref{eq:Q-update} and $T$ is the number of iterates.
When we conduct inference, we use the average estimate $\bar{Q}_T$ rather than the last iterative value $Q_T$ given an iteration budget $T$.
The application of Polyak-Ruppert averaging in deep RL has been shown empirically to have benefits in terms of error reduction and stability~\citep{lillicrap2016continuous,anschel2017averaged}.

\paragraph{Bellman noise.}
Let $\Z_t \in \RB^D$ be the Bellman noise at the $t$-th iteration, whose $(s, a)$-th entry is
\begin{equation}
\label{eq:Z}
Z_t(s, a) = \widehat{\TM}_t(Q^*)(s, a) -\TM(Q^*)(s, a).
\end{equation}
In matrix form, the Bellman noise at iteration $t$ can be equivalently presented as $ \Z_t = (\rr_t-\rr) + \gamma (\PP_t-\PP) \V^* $.
The Bellman noise $\Z_t$ reflects the noise present in the empirical Bellman operator~\eqref{eq:em-T} using samples collected at iteration $t$ as an estimate of the population Bellman operator~\eqref{eq:T}.
%~\citep{wainwright2019stochastic} terms $\Z_t$ as effective noise variables.

In our synchronous setting, $\rr_t$ and $\PP_t$ are independent of each other and the past history.
%the generator produces the reward $r_t(s, a)$ and next-state $s_t(s, a)$ independent over the historical samples $\{ (r_l, s_l) \}_{l < t}$ and over all state-action pairs.
Therefore, $\{\Z_t\}$ is an i.i.d.\ random vector sequence with coordinates that are mean zero and mutually independent.
When it is clear from the context, we drop the dependence on $t$ and use $\Z$ to denote an independent copy of $\Z_t$.
We refer to $\Z$ as the Bellman noise (vector).
Finally, an important quantity in our analysis is the covariance matrix of $\Z$:
\begin{equation}
\label{eq:V-bellman}
\Var(\Z) = \EB_{r_t, s_t} \Z \Z^\top \in \RB^{D \times D},
\end{equation}
where the expectation $\EB_{r_t, s_t}(\cdot)$ is taken over the randomness of rewards $r_t$ and states $s_t$.
Clearly, $\Var(\Z)$ is a diagonal matrix with the $(s, a)$-th diagonal entry given by $\EB Z_t^2(s, a)$.

%We use $\|\A\|_{2 \to \infty}$ denote two-to-infinity norm, i.e., $\|\A\|_{2 \to \infty} = \max_{\|\x\|_{2}=1} \|\A\x\|_{\infty} = \max_i\sqrt{\sum_j|\A(i, j)|^2}$.
%Equivalently, $\|\A\|_{2 \to \infty}$ is the maximum $\ell_2$-norm of row vectors in $\A$.
%These matrix norms have the relation $\|\A\|_{\max} \le \|\A\|_{2\to \infty} \le \|\A\|_{\infty}$~\citep{cape2019two}.
\vspace{-0.1in}

%\section{Main Results}
%\label{sec:main}

%\subsection{Bellman Noise}

\section{FUNCTIONAL CENTRAL LIMIT THEOREM FOR PARTIAL-SUM AVERAGED Q-LEARNING}
\label{sec:asym-ave-Q}

Our main result is a functional central limit theorem for the partial-sum process of averaged Q-learning.
To that end, we make three assumptions.
The first is that all random rewards have uniformly bounded fourth moments (Assumption~\ref{asmp:reward}).
Though typical in the SA literature~\citep{borkar2009stochastic}, it is weaker than the uniform boundedness assumption which is often used for nonasymptotic analysis in RL.
It is required for a technical reason (that we should ensure a residual error vanishes uniformly in probability, a result which is one of our technical contributions).

The second is a Lipschitz condition (Assumption~\ref{asmp:gap}) over a specific optimal policy $\pi^* \in \Pi^*$, where $\Pi^*$ collects all optimal policies.
The condition is true when $|\Pi^*| = 1$ (See Lemma~\ref{lem:gap} for the reason).
Similar assumptions have been adapted for asymptotic analysis for general nonlinear SA~\citep{mokkadem2006convergence}, and nonasymptotic analysis for both variance reduced Q-learning~\citep{khamaru2021instance} and policy iteration~\citep{puterman1979convergence}.
The condition implies that when $\Q_t \approx \Q^*$ the asymptotic behavior of averaged Q-learning is captured by a linear system up to a high-order approximation error.
As a result, we can explicitly formulate the asymptotic variance matrix.
The approach of approximating a nonlinear SA by a specific linear SA and analyzing the approximation errors is also standard in the SA literature~\citep{polyak1992acceleration,mokkadem2006convergence,lee2021fast,li2021statistical}.

The last assumption (Assumption~\ref{asmp:lr}) requires that the step size decays at a sufficiently slow rate; this is necessary in order to establish asymptotic normality~\citep{polyak1992acceleration,su2018uncertainty,chen2020statistical,li2021statistical}.
A typical example satisfying Assumption~\ref{asmp:lr} is the polynomial step size, $\eta_t = t^{-\alpha}$ with $\alpha \in (0.5, 1)$.

\begin{asmp}
	\label{asmp:reward}
	We assume $ \EB |R(s, a)|^4 < \infty$ for all $(s, a) \in \SM \times \AM$.
\end{asmp}

%\begin{asmp}[Unique optimal policy]
%	\label{asmp:gap}
%	The optimal policy is unique, which we denote by $\pi^*$.
%	It implies a positive optimality gap defined by $\gap := \min\limits_{s \in \SM} \min\limits_{a \neq \pi^*(s)}| V^*(s)-Q^{*}(s, a)| > 0.$
%\end{asmp}

\begin{asmp}
	\label{asmp:gap}
	There exists $\pi^* \in \Pi^*$ such that for any $Q$-function estimator $\Q \in \RB^D$, $	\|(\PP^{\pi_{\Q}}-\PP^{\pi^*})(\Q - \Q^*)\|_{\infty} \le L \|\Q - \Q^*\|_{\infty}^2$ where $\pi_{Q}(s) := \arg\max_{a \in \AM} Q(s, a)$ is the greedy policy w.r.t. $Q$.
\end{asmp}

	\begin{asmp}
	\label{asmp:lr}
	Assume
	\emph{(i)} $0 \le \sup_t \eta_t \le 1, \eta_t \downarrow 0$ and $t \eta_t \uparrow \infty$; \emph{(ii)} $\frac{\eta_{t-1} - \eta_{t}}{\eta_{t-1}} =o(\eta_{t-1})$; \emph{(iii)} $\frac{1}{\sqrt{T}} \sum_{t=0}^T \eta_{t} \to 0$ for all $t \ge 1$; \emph{(iv)} $\frac{\sum_{t=0}^T \eta_t}{ T\eta_T} \le C$ for all $T \ge 1$.
\end{asmp}

We now present the functional CLT for averaged Q-learning under the same conditions.
Define the standardized partial-sum processes associated with $\{ \Q_t \}_{t \ge 0}$ as follows:
\begin{equation}
\label{eq:partial}
\Bphi_T(r) := \frac{1}{\sqrt{T}}  \sum_{t=1}^{\Tr} (\Q_t - \Q^*),
\end{equation}
where $r \in [0, 1]$ is the fraction of the data used to compute the partial-sum process and $\floor{\cdot}$ returns the largest integer smaller than or equal to the input number.
%For simplicity, we use $\Bphi_T:=\{\Bphi_T(r)\}_{r \in [0, 1]}$ or $\Bphi_T(\cdot)$ to denote the whole function.

\begin{thm}
\label{thm:fclt}
Under Assumptions~\ref{asmp:reward},~\ref{asmp:gap} and~\ref{asmp:lr}, we have 
\begin{equation}
\label{eq:weak-conver}
\Bphi_T(\cdot) \overset{w}{\to}  \VQ^{1/2} \B_D(\cdot),
\end{equation}
where $\VQ \in \RB^{D \times D}$ is the asymptotic variance
	\begin{equation}
	\label{eq:opt-variance}
	\VQ  = 
	(\I-\gamma \PP^{\pi^*})^{-1}\Var(\Z)(\I-\gamma \PP^{\pi^*})^{-\top} 
\end{equation}
and $\B_D(\cdot) \in \RB^D$ is a standard Brownian motion on $[0, 1]$. 
\end{thm}

The conventional CLT asserts that $\Bphi_T(1) = \sqrt{T}(\bar{\Q}_T-\Q^*)$ converges in distribution to a rescaled Gaussian random variable $ \VQ^{1/2} \B_D(1) $ as $T \to \infty$ (see Appendix~\ref{sec:clt} for more details).  The functional CLT in Theorem~\ref{thm:fclt} extends this convergence to the whole function $\Bphi_T=\{ \Bphi_T(r) \}_{r \in [0, 1]}$ in the sense that any finite-dimensional projections of $\Bphi_T$ converge in distribution.
That is, for any given integer $n \ge 1$ and any $0 \le t_1 < \cdots < t_n \le 1$, as $T \to \infty$, $(\Bphi_T(t_1), \cdots, \Bphi_T(t_n))
\overset{d}{\to} \VQ^{1/2} (\B_D(t_1),\cdots, \B_D(t_n))$.
The convergence $\overset{w}{\to}$ in~\eqref{eq:weak-conver} also corresponds to the weak convergence of measures in the $D$-dimensional Skorokhod spaces $\BDD$ (see Appendix~\ref{sec:weak-con} for a short introduction).
Here $\BDD = \{ \text{right continuous with left limits} \ \omega(r) \in \RB^D, r \in [0, 1] \}$.
Eq.~\eqref{eq:weak-conver} is equivalent to the convergence of finite-dimensional projections.

Theorem~\ref{thm:fclt} can be viewed as a generalization of Donsker's theorem~\citep{donsker1951invariance} to Q-learning iterates.
Donsker's theorem shows the partial-sum process of a sequence of independent and identically distributed (i.i.d.) random variables weakly converges to a standard Brownian motion, while subsequent works extend this functional result to weakly dependent stationary sequences~\citep{dudley2014uniform}.
Since in our case $\pi_t$ and $\V_t$ might depend on history data arbitrarily, $\{\Q_t\}_{t \ge 0}$ is neither i.i.d. nor stationary.
To prove the functional CLT, we use a particular error decomposition and partial-sum decomposition. We give a proof sketch in Section~\ref{sec:proof}.

\begin{figure}[t!]
	\centering
	\includegraphics[width=\columnwidth]{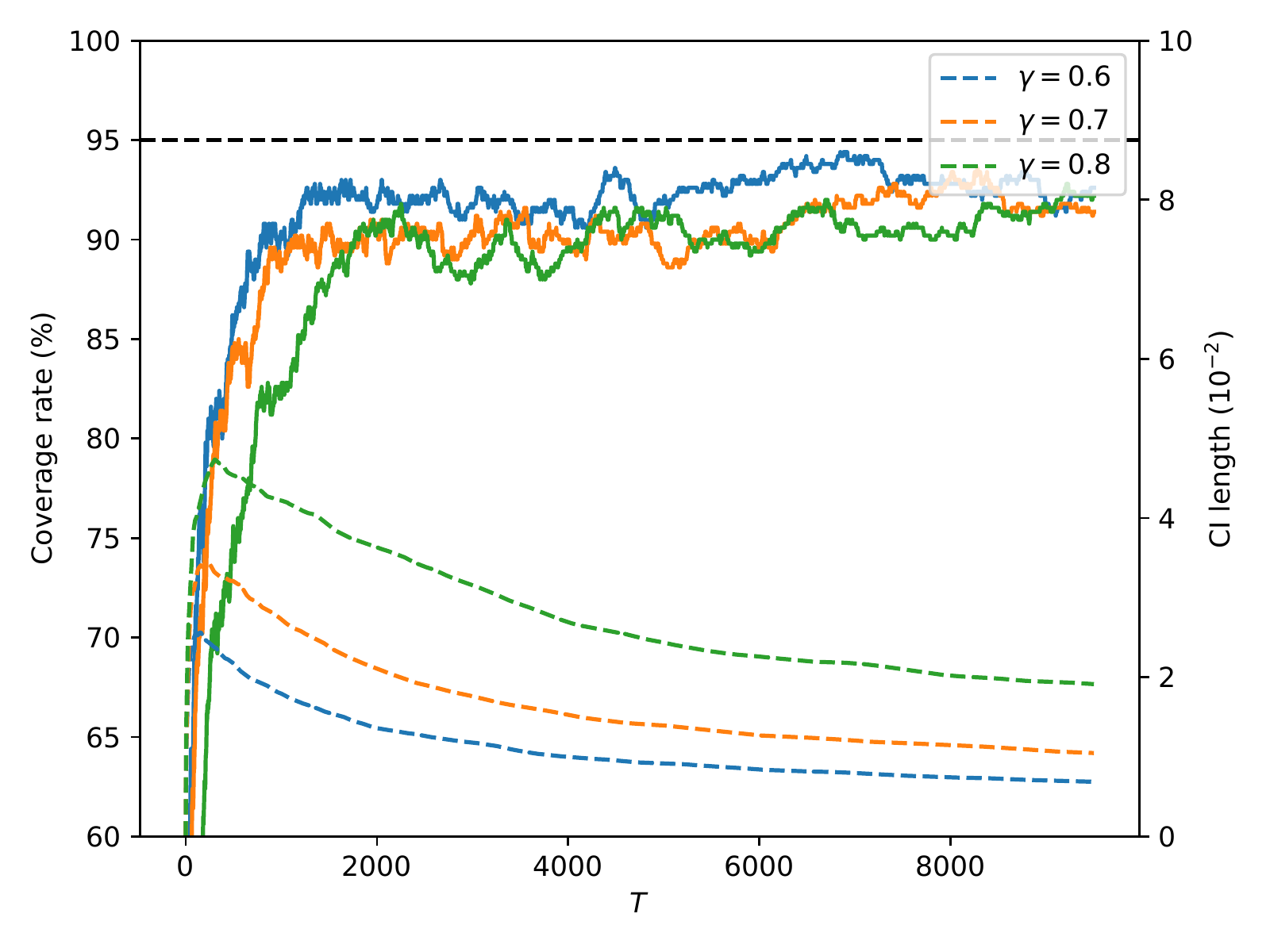}
		\vspace{-0.3in}
	\caption{Empirical coverage rates (left) and CI lengths (right) of $\bar{Q}_T(s_0, a_0)$ against the number of iterations $T$ on a specific $(s_0, a_0)$.
	Both are obtained by averaging over 500 independent Q-learning trajectories.
	Black dashed line denotes the nominal coverage rate of 95\%.
	  }
	\label{fig:coverage}
		\vspace{-0.2in}
\end{figure}

\vspace{-0.1in}
\paragraph{Comparison with previous (functional) CLTs.}
% Some works also establish CLT or FCLT results for (general) nonlinear SA.
% We specify the differences in the following.

Most CLT results consider linear SA which is non-applicable here (see~\cite{mou2020linear,mou2020optimal} and references therein).
The original result for Polyak-Ruppert averaging  \citep{polyak1992acceleration,moulines2011non,durmus2022finite} also doesn't apply in our case because it assumes a locally strongly convex Lyapunov function---which is not known to exist for Q-learning.
\cite{konda2002actor} shows $\frac{\Q_T-\Q^*}{\sqrt{\eta_T}} \overset{d}{\to}\NM(\0, \Var)$ with $\Var = \frac{\lim_{T}}{\eta_T} \EB(\Q_T-\Q^*)(\Q_T-\Q^*)^\top$ when we assume the limit involved exists.
\cite{mokkadem2006convergence} shows $\Bphi_T(1) \overset{d}{\to}  \NM(\0, \VQ)$ under a similar Lipschitz condition Assumption~\ref{asmp:gap}.

To date, formal functional CLT results for SA are mainly based on the ODE approach~\citep{abounadi2002stochastic,borkar2009stochastic,gadat2018stochastic,borkar2021ode}.
These works focus on the asymptotic behavior of the interpolated process connecting properly rescaled last iterates.
An example interpolated process $\widetilde{\Bphi}_T(\cdot)$ satisfies $\widetilde{\Bphi}_T(0) = \frac{\Q_T-\Q^*}{\sqrt{\eta_T}}$ and $\widetilde{\Bphi}_T(t_k^T) = \frac{\Q_{T+k}-\Q^*}{\sqrt{\eta_{T+k}}}$ for a specific sequence $\{ t_k^T \}_{k \ge 0}$ depending on the step size and satisfying $t_0^T=0$ and $\lim_{k} t_k^T = \infty$.
This functional CLT result implies $\widetilde{\Bphi}_T(\cdot)$ converges weakly to the solution of a specific SDE.
Theorem~\ref{thm:fclt} is different because it is concerned with the partial-sum process $\Bphi_T(\cdot)$ and explicitly formulates the asymptotic variance $\VQ$.
Recent work studying statistical inference via SGD variants also provides functional CLTs for a similar partial-sum process~\citep{lee2021fast,li2021statistical}, given the loss function is smooth and strongly convex.
However, those results don't apply here since Q-learning doesn't meet the underlying assumptions.
Our functional CLT for the partial-sum process of Q-learning is novel.

%\paragraph{More discussion on Assumption~\ref{asmp:gap}.}

\subsection{Online Statistical Inference}

The functional CLT opens a path towards statistical inference in RL.
While traditional approaches estimate asymptotic variances in RL by batch-mean estimators~\citep{chen2020statistical,zhu2021online} or bootstrapping~\citep{hao2021bootstrapping}, by contrast, the functional CLT allows us to construct an asymptotically pivotal statistic using the whole function $\Bphi_T$.
The inference method, known as random scaling, was originally designed for strongly convex optimization~\citep{lee2021fast,li2021statistical}.

\begin{prop}
	\label{prop:pivotal}
	The continuous mapping theorem together with Theorem~\ref{thm:fclt} yields that with probability approaching one, $\int_0^1\Bphi_T(r)\Bphi_T(r)^\top dr$ is invertible and 
	\begin{align}
			\label{eq:pivotal}
			\begin{split}
					&\Bphi_T(1)^\top \left(   \int_0^1 \bar{\Bphi}_T(r) \bar{\Bphi}_T(r)^\top \mathrm{d} r \right)^{-1} \Bphi_T(1) \\
				&\overset{d}{\to} \B_D(1)^\top\left(   \int_0^1 \bar{\B}_D(r) \bar{\B}_D(r)^\top \mathrm{d} r \right)^{-1}\B_D(1),
			\end{split}
	\end{align}
	where $\bar{\Bphi}_T(r) := \Bphi_T(r) - r \cdot \Bphi_T(1)$ and $\bar{\B}_D(r) := \B_D(r) - r \cdot \B_D(1)$ for simplicity.
\end{prop}

The left-hand side of~\eqref{eq:pivotal} is a pivotal quantity involving samples and the unobservable parameter of interest $\Q^*$.
The pivotal quantity can be constructed in a fully online fashion and thus is computationally efficient.\footnote{See Algorithm 1 in~\citep{lee2021fast} or Algorithm 2 in~\citep{li2021statistical} for the online procedure.}
The right-hand side of~\eqref{eq:pivotal} is a known distribution whose quantiles can be computed via simulation~\citep{kiefer2000simple,abadir2002simple}.
In this way, we don't need a consistent estimator for the asymptotic variance in order to provide asymptotically valid confidence intervals for $\Q^*$, as are required by previous work~\citep{hao2021bootstrapping,shi2020statistical,khamaru2022instance}.
As an illustration, Figure~\ref{fig:coverage} shows the empirical coverage rates and confidence interval (CI) lengths on a random MDP with three values of $\gamma$.
As $T$ increases, the empirical coverage rates increase rapidly, approaching $95\%$, and the CI lengths decay. 
More details are placed in Appendix~\ref{appen:exp}.

% The construction of an asymptotic pivotal quantity is not unique.
% Proposition~\ref{prop:pivotal} provides a particular example.
% It is of interest to identify other possible constructions and to compare their advantages in an appropriate sense, which we leave for future work.

\subsection{Proof Sketch}
\label{sec:proof}
In the part, we provide a proof sketch of Theorem~\ref{thm:fclt} to highlight our technical contributions.
A full proof of Theorem~\ref{thm:fclt} is provided in Appendix~\ref{proof:fclt}.

\paragraph{Step 1: Error decomposition.}
Let $\DDelta_t  = \Q_t-\Q^*$.
Recall that the Q-learning update rule is~\eqref{eq:Q-update}.
It follows that
\begin{align*}
	\DDelta_{t}
	&= (1-\eta_t) \DDelta_{t-1}+ \eta_t\left[(\rr_t-\rr)+\gamma( \PP_t \V_{t-1}-\PP \V^*)\right] \\
	&= (1-\eta_t) \DDelta_{t-1}+ \eta_t\left[ \Z_t+\gamma\PP_t(  \V_{t-1}- \V^*)\right],
\end{align*}
where $\Z_t =  (\rr_t-\rr) + \gamma( \PP_t - \PP) \V^*$ is the Bellman noise.
Notice that $\PP_t(  \V_{t-1}- \V^*)=  (\PP_t-\PP)( \V_{t-1}- \V^*) + \PP( \V_{t-1}- \V^*)$.
Using $\PP \V_{t-1} = \PP^{\pi_{t-1}} \Q_{t-1}$ and $\PP \V^* = \PP^{\pi^*} \Q^*$, we further have $\PP( \V_{t-1}- \V^*)= \PP^{\pi_{t-1}} \Q_{t-1}- \PP^{\pi^*} \Q^* = (\PP^{\pi_{t-1}}-\PP^{\pi^*})\Q_{t-1} + \PP^{\pi^*}\De_{t-1}$.
Putting the pieces together,
\[
\DDelta_{t} =  \A_t \De_{t-1} + \eta_t \left[  \Z_t + \gamma \Z_t' + \gamma \Z_t'' \right],
\]
where $\A_t = \I-\eta_t \sG, \sG = \sI - \gamma \PP^{\pi^*}, \Z_t'=(\PP_t-\PP)( \V_{t-1}- \V^*)$, and $\Z_t'' = (\PP^{\pi_{t-1}}-\PP^{\pi^*})\Q_{t-1}$.
Recursing the last equality gives
\begin{equation}
\label{eq:delta_iter}
	\DDelta_t
	= \prod_{j=1}^t \A_{j} \DDelta_{0}+  \sum_{j=1}^t \prod_{i=j+1}^t \A_{i} \eta_j \left( \Z_j + \gamma \Z_t' + \gamma \Z_t''\right).
\end{equation}
In addition, using the general step size in Assumption~\ref{asmp:lr}, we can show $\frac{1}{\sqrt{T}} \sum_{t=1}^T \EB\|\DDelta_t\|_{\infty}^2 \to 0$ (in Theorem~\ref{thm:general-Linfty-pw2}).

\paragraph{Step 2: Partial-sum decomposition.}
For simplicity,  for any $T \ge j \ge 0$ we denote 
\begin{equation}
\label{eq:AjT_main}
\A_{j}^T =\eta_j \sum_{t=j}^T \prod_{i=j+1}^t \A_i.
\end{equation}
Setting $\Bpsi_0(r) :=\frac{1}{\eta_0\sqrt{T}}  (\A_0^{\Tr}-\eta_0 \I) \DDelta_{0}$ and plugging~\eqref{eq:delta_iter} into $\Bphi_T(r)= \frac{1}{\sqrt{T}} \sum_{t=1}^{\floor{Tr}} \De_t$, yields
\begin{align}
&\Bphi_T(r)
=\Bpsi_0(r)
+\frac{1}{\sqrt{T}}  \sum_{j=1}^{\Tr}\A_j^{\Tr}\left(\Z_j+\gamma \Z_j' + \gamma \Z_j''\right) \nonumber \\
&
	=\Bpsi_0(r) 
	+\frac{1}{\sqrt{T}}  \sum_{j=1}^{\Tr}\G^{-1}    \Z_j 	+ \frac{1}{\sqrt{T}}  \sum_{j=1}^{\Tr}  (\A_j^{T} - \G^{-1})  \Z_j  \nonumber \\
&	
	+ \frac{\gamma}{\sqrt{T}}  \sum_{j=1}^{\Tr} \A_j^{T}  \Z_j'+  \frac{1}{\sqrt{T}}  \sum_{j=1}^{\Tr}(\A_j^{\Tr} - \A_j^T)\left[  \Z_j +  \gamma\Z_j' \right] \nonumber \\
&+   \frac{\gamma}{\sqrt{T}}  \sum_{j=1}^{\Tr} \A_j^{T}  \Z_j''
:= \sum_{i=0}^5  \Bpsi_i(r). \label{eq:error0}
\end{align} 
\paragraph{Step 3: Establish the functional CLT.}
To measure the distance between random functions, we define $\|\Bpsi\|_{\sup} = \sup_{r \in [0, 1]} \|\Bpsi(r)\|_{\infty}$.
The standard martingale functional CLT~\citep{hall2014martingale,jirak2017weak} implies $\Bpsi_1(\cdot) \overset{w}{\to} \VQ^{1/2} \B_D(\cdot)$.
To complete the proof, it suffice to show $\|\Bphi_T-\Bpsi_1\|_{\sup}  = o_{\PB}(1)$ which is implied by $\|\Bpsi_i\|_{\sup}  = o_{\PB}(1)$ for $i=0,2,3,4,5$.

By Lemma 1 in~\citep{polyak1992acceleration}, we know $\sup_{T\ge j\ge0}\|\A_j^T\|_{\infty} \le C_0$ and $\lim_{T \to \infty} \frac{1}{T}\sum_{j=1}^T\|\A_j^T - \G^{-1}\|_{2} = 0$.
Then it is obvious $\|\Bpsi_0\|_{\sup}  = o_{\PB}(1)$.
Noting that $\Z_j, \Z_j'$ are martingale differences, we can show $\EB\|\Bpsi_i\|_{\sup}^2  = o(1)$ for $i=2,3$ by Doob’s
inequality.

By definition of greedy policies $\pi^*$ and $\pi_{t-1}$, we know $\PP^{\pi^*} \Q_{t-1} \le \PP^{\pi_{t-1}} \Q_{t-1}$ and $\PP^{\pi_{t-1}} \Q^* \le \PP^{\pi^*} \Q^*$, which implies $\|\Z_t''\|_{\infty} = \|(\PP^{\pi_{t-1}}-\PP^{\pi^*})\Q_{t-1} \|_{\infty} \le \|(\PP^{\pi_{t-1}}-\PP^{\pi^*})  \De_{t-1}\|_{\infty}  \le L \|\De_{t-1}\|_{\infty}^2$ from Assumption~\ref{asmp:gap}.
Then $\EB  \|\Bpsi_5\|_{\sup} \le   \frac{L C_0}{\sqrt{T}}\sum_{t=1}^T \EB\|\De_t\|_{\infty}^2 \to 0$.

The most challenging step is to show $\|\Bpsi_4\|_{\sup}  = o_{\PB}(1)$.
Notice that $\Bpsi_4$ is a weighted sum of martingale differences, $\Z_j+\gamma \Z_j'$, with the coefficients varying in $r$ such that we can't apply Doob’s inequality.
To deal with this issue, we relate $\Bpsi_4$ to an autoregressive sequence indexed by $k \in [T]$ and analyze the maximum over $k$ directly. 
More specifically, we can show 
\[
\|\Bpsi_4\|_{\sup} \precsim \sup_{k \in [T]}\left\| \frac{1}{\sqrt{T}\eta_{ k+1}} \sum_{j=1}^k \prod_{i=j+1}^{k}\A_i \eta_j (\Z_j+\gamma \Z_j') \right\|.
\]
Previous results~\cite{lee2021fast,li2021statistical} do not apply here, since they require $\sG= \sI -\gamma \PP^{\pi^*}$  to be positive semidefinite, which isn't our case.
 Noticing that all eigenvalues of $\sG$ have nonnegative real parts, we provide a novel analysis of the right-hand side in Lemma~\ref{lem:ignore-error}, showing it is indeed $o_{\PB}(1)$ under Assumption~\ref{asmp:reward}.
 This is one of our technical contributions.

\begin{rem}
If we consider policy evaluation (so that $\pi_t$ remains unchanged and $\Bpsi_{5}$ disappears), $\Bpsi_4$ is still present.
Showing $\|\Bpsi_4\|_{\sup} = o_{\PB}(1)$ is required even for linear SA. 
\end{rem}

\section{INFORMATION-THEORETIC LOWER BOUND}
%Information-Theoretic Lower Bound}
\label{sec:infolb}
The standard CLT implies $\bar{\Q}_T$ is a $\sqrt{T}$-consistent estimate for $\Q^*$.
It is of theoretical interest to investigate whether or not $\bar{\Q}_T$ is asymptotically efficient.
%, i.e., whether $\bar{\Q}_T$ achieves the smallest asymptotic variance among all asymptotically unbiased estimators.
In parametric statistics~\citep{lehmann2006theory}, the Cramer-Rao lower bound assesses the hardness of estimating a target parameter $\beta(\theta)$ in a parametric model $\PM_\theta$ indexed by parameter $\theta$.
Any unbiased estimator whose variance achieves the Cramer-Rao lower bound is viewed as optimal and efficient.
The concept of Cramer-Rao lower bounds can be extended to possibly biased but asymptotically unbiased estimators and also to nonparametric statistical models where the dimension of the parameter $\theta$ is infinity~\citep{van2000asymptotic,tsiatis2006semiparametric}.

\paragraph{The semiparametric model.}
In our case, the transition kernel $\{P(\cdot|s, a)\}_{s, a}$ is specified by $D$ parametric distributions on $\DM$, while the random reward $\{ R(s, a) \}_{s, a}$ is fully nonparametric because the $R(s, a)$ are not assumed to come from finite-dimensional models.
%\footnote{Equivalently, our target parameter becomes those density functions whose expectation equals to $r(\cdot)$'s. In the way, the parameter space has infinite dimension.}
Hence, to derive an extended Cramer-Rao lower bound for $Q^*$ estimation, we need to enter the world of semiparametric statistics.
In particular, our MDP model $\mathcal{M}=(\SM, \AM, \gamma, P, R, r)$ has parameter $\theta = (P, R)$.
%Due to the generator setting, $P$ and $R$ are variationally independent.
Our parameter of interest is $\beta(\theta) = \Q^*$.
At iteration $t$, we observe the random rewards and empirical transitions for each $(s, a)$ and concatenate them into $\rr_t \in \RB^D$ and $\PP_t \in \RB^{D \times S}$.
The distribution of $\PP_t$ is determined by its expectation $\PP = \EB \PP_t$, which belongs to 
\begin{align}
	\label{eq:PM}
\begin{split}
	\PM_P:&=\left\{\PP\in\RB^{D\times S}:P(s'|s,a)\ge0, \forall  (s,a,s') \right. \\
	&\text{and} \left. \sum_{s' \in \SM} P(s'|s, a) = 1, \forall  (s,a) \right\},
\end{split}
\end{align}
%\begin{equation}
%\label{eq:PM}
%\PM_P:=\left\{\PP\in\RB^{D\times S}:P(s'|s,a)\ge0, \forall  (s,a,s') \in \SM \times \AM \times \SM \  \text{and} \sum_{s' \in \SM} P(s'|s, a) = 1  \right\},
%\end{equation}
while $R$ is nonparametric and belongs to 
\[
\PM_R = \left\{ \{ R(s, a)\}_{s, a}: \EB R(s, a) = r(s, a), \forall (s, a) \right\}.
\]
%Since only its expectation matters in our interest parameter $Q^*$, it is still identifiable.
According to the generative model, the $\rr_t$ and $\PP_t$ are mutually independent and also independent of the historical data.
Let $\DM = \{(\rr_t, \PP_t)\}_{t \in [T]}$ contain the $T$ samples generated as described above.

\paragraph{Semiparametric efficiency lower bound.}
\citet{tsiatis2006semiparametric} has argued that regular asymptotically linear (RAL) estimators provide a good tradeoff between expressivity and tractability.
In RL, RAL estimators are widely considered in off-policy evaluation problems~\citep{kallus2020double}.

% \begin{defn}[Regular estimator]
%     \label{def:regular}
%     Denote a perturbed dataset $\D_{1/\sqrt{T}}=\{(\rr_t', \PP_t')\}_{t\in[T]}$, where $\LM(\rr_t')\in\PM_{R, 1/\sqrt{T}}$\footnote{Given a probability space $(\Omega, P, \FM)$, $\LM(X)$ is the law of the random variable $X$ in this probability space.} and $\LM(\PP_t')\in\PM_{P,1/\sqrt{T}}$. Here $\PM_{R, 1/\sqrt{T}}$ and $\PM_{P, 1/\sqrt{T}}$ are the perturbed models which converge to $\PM_R$ and $\PM_P$ with $T$ going infinity. In this perturbed model, we denote the true optimal Q-value function as $\Q^*_{1/\sqrt{T}}$. We say a measurable random function $\widehat{\Q}_T$ of $\DM$ is regular if it satisfies $\sqrt{T}(\widehat{\Q}_T-\Q^*_{1/\sqrt{T}}))=O_{\PB}(1)$.
% \end{defn}

\begin{defn}[Regular estimator]
    \label{def:regular}
    Denote the distribution of $\rr_t$ and $\PP_t$ by $\LM(\rr)$ and $\LM(\PP)$.\footnote{Given a probability space $(\Omega, P, \FM)$, $\LM(X)$ is the law of the random variable $X$ in this probability space. Since $\rr_t$ are i.i.d., they share the same distribution $\LM(\rr)$ and similarly for $\LM(\PP)$.}
    For any given $T$, let $\LM_{T}(\rr)$ and $\LM_{T}(\PP)$ be the perturbed distributions of $\LM(\rr)$ and $\LM(\PP)$ which are consistent in the sense that they converge\footnote{$\LM_T(\rr)$  and $\LM_T(\PP)$ are differentiable in quadratic mean at $\LM(\rr)$ and $\LM(\PP)$. See Chapter 25.3 in \cite{van2000asymptotic}.} to $\LM(\rr)$ and $\LM(\PP)$ when $T$ goes infinity. 
    Let $\widehat{\Q}_T$ be any estimator of $\Q^*$ computed from $\DM$.
    Let $\Q^*_{T}$ be the true optimal Q-value function when rewards and transition probabilities are generated i.i.d.\ from $\LM_{T}(\rr)$ and $\LM_{T}(\PP)$. 
    We say $\widehat{\Q}_T$ is a \emph{regular estimator} of $\Q^*$ if $\sqrt{T}(\widehat{\Q}_T-\Q^*_T)$ weakly converges to a random variable that depends only on $\LM(\rr)$ and $\LM(\PP)$, when samples are distributed according to the probability measure $(\LM_T(\rr), \LM_T(\PP))$.
\end{defn}

\begin{rem}
	% In Theorem~\ref{thm:infolb}, an estimator $\widehat{\Q}_T$ is called regular in a parametric model with $\theta:=(\theta_{\rr}, \theta_{\PP})$ if it satifies: 
	% (1) $\sqrt{T}(\widehat{\Q}_T(\DM)-\Q^*(\theta^*))$ weakly converges to a distribution;
	% (2) For a new data generating process such that $\DM_T:=\{(\rr_t^{(T)}, \PP^{(T)}_t)\}_{t=1}^T\overset{i.i.d.}{\sim}f_{\theta_T}(\rr, \PP)$, where $f_{\theta_T}$ is the density function and $\sqrt{T}(\theta_T-\theta^*)$ converges to a constant, the limiting distribution of $\sqrt{T}(\widehat{\Q}_T(\DM_T)-\Q^*(\theta_T))$ does not depend on $T$. Hence, when we say a semiparametric estimator is regular, it stands for it is regular under every sub-parametric models. 
	%	In Theorem~\ref{thm:infolb}
	Informally speaking, an estimator is regular if its limiting distribution is unaffected by local changes in the data-generating process.	
	%	In other words, it is regular in every regular parametric submodel and the limiting distribution does not depend on the parametric submodel.
	The assumption of regularity excludes super-efficient estimators, whose asymptotic variance can be smaller than the Cramer-Rao lower bound for some parameter values, but which perform poorly in the neighborhood of points of super-efficiency.
	We refer interested readers to Section 3.1 in \citep{tsiatis2006semiparametric} for a detailed exposition.
%	There exists a straightforward criterion on influence functions to check regularity of an asymptotically linear estimator (see Theorem 2.2 in~\citep{newey1990semiparametric}; we employ this criterion in our proof of Theorem~\ref{thm:RAL}).
\end{rem}

\begin{defn}[Regular asymptotically linear]
	\label{def:RAL}
	Let $\widehat{\Q}_T \in \RB^D$ be a measurable random function of $\DM = \{(\rr_t, \PP_t)\}_{t \in [T]}$.
	We say that $\widehat{\Q}_T$ is \emph{regular asymptotically linear} (RAL) for $\Q^*$ if it is regular and asymptotically linear with a measurable random function $\Bphi(\rr_t, \PP_t) \in \RB^D$ such that
	\[
	\sqrt{T}(\widehat{\Q}_T - \Q^*) = \frac{1}{\sqrt{T}} \sum_{t=1}^T \Bphi(\rr_t, \PP_t) + o_{\PB}(1).
	\]
	Here $\Bphi(\cdot, \cdot) $ is referred to as an \emph{influence function}, and it satisfies $\EB \Bphi(\rr_t, \PP_t) = \0$ and $\EB \Bphi(\rr_t, \PP_t)\Bphi(\rr_t, \PP_t)^\top$.
\end{defn}

\begin{thm}
	\label{thm:infolb}
	Given the dataset $\DM = \{(\rr_t, \PP_t)\}_{t \in [T]}$, for any RAL estimator $\widehat{\Q}_T$ of $\Q^*$ computed from $\DM = \{(\rr_t, \PP_t)\}_{t \in [T]}$, its variance satisfies
    \begin{align*}
    \lim_{T \to \infty} T
    \EB(\widehat{\Q}_T-\Q^*) (\widehat{\Q}_T-\Q^*)^\top
    \succeq  \VQ, 
    \end{align*}
	where $\A \succeq \B$ means $\A-\B$ is positive semidefinite and $\VQ$ is given in~\eqref{eq:opt-variance}.
\end{thm}

By Definition~\ref{def:RAL}, any influence function determines an asymptotic linear estimator for $\Q^*$.
The semiparametric efficiency bound in Theorem~\ref{thm:infolb} gives us a concrete target in the construction of the influence function.
If we can find an influence function that achieves the bound,  we know that it is the most efficient among all RAL estimators.
Fortunately, Theorem~\ref{thm:RAL} implies that $\bar{\Q}_T$ is the most efficient estimator among all RAL estimators with the efficient influence function $(\I - \gamma \PP^{\pi^*})^{-1} \Z_t$.
It also implies that for any fixed $r \in [0, 1]$, $\Bphi_T(r) = \sqrt{r} \cdot \sqrt{\floor{Tr}}(\bar{\Q}_{\floor{Tr}} -\Q^*)$ has the optimal asymptotic variance (scaled by a factor $\sqrt{r}$).
% Theorem~\ref{thm:RAL} is stronger than the conventional CLT (e.g., Theorem~\ref{thm:clt}) because it not only implies asymptotic normality, but also shows the regularity of $\bar{\Q}_T$.
Proofs are provided in Appendix~\ref{proof:infolb}.

\begin{thm}
	\label{thm:RAL}
	Under Assumptions~\ref{asmp:reward},~\ref{asmp:gap} and~\ref{asmp:lr}, the averaged Q-learning iterate $\bar{\Q}_T$ is a RAL estimator for $\Q^*$.
	In particular, we have the following decomposition
	\[
	\sqrt{T}\left( \bar{\Q}_T  - \Q^* \right) = \frac{1}{\sqrt{T}} \sum_{t=1}^T (\I - \gamma \PP^{\pi^*})^{-1} \Z_t + o_{\PB}(1),
	\]
	where $\Z_t = (\rr_t-\rr) + \gamma (\PP_t - \PP) \V^*$ is the Bellman noise at iteration $t$. 
\end{thm}

\section{INSTANCE-DEPENDENT NONASYMPTOTIC CONVERGENCE}
%Instance-Dependent Nonasymptotic Convergence}
\label{sec:nonasymptotic}

In the section, we explore the nonasymptotic behavior of averaged Q-learning, i.e., we study the dependence of $\EB \|\bar{\Q}_T-\Q^*\|_{\infty}$ on finite $T$ and $(1-\gamma)^{-1}$.
%We first provide a nonasymptotic convergence result to validate~\eqref{eq:asym}

\begin{thm}
	\label{thm:con-linear}
	Let Assumptions~\ref{asmp:gap} hold and $0 \le R(s, a) \le 1$ for all $(s, a) \in \SM \times \AM$.\footnote{To simplify the  parameter dependence, we assume rewards are uniformly bounded as in previous work~\citep{wainwright2019variance,khamaru2021instance,li2021q} .
Note that, thanks to the error decomposition in~\eqref{eq:error0}, it is possible to provide a nonasymptotic analysis assuming rewards have finite second moments.
The consequence is that the dependence on $d$ and $\delta$ would change from $\log D, \log \frac{1}{\delta}$ to $D$ and $\frac{1}{\delta}$.
} When $D$ is larger than a universal constant, 
	\begin{itemize}
		\item If $\eta_t = t^{-\alpha}$ with $\alpha \in (0.5, 1)$ for $t\ge 1$ and $\eta_0 = 1$, it follows that for all $T \ge 1$, $\EB\|\bar{\Q}_T - \Q^* \|_{\infty}=$
		\begin{align*}
		\label{eq:bound-poly}
		\begin{split}
		   	&\OM \left(
		\sqrt{\|\diag(\VQ)\|_{\infty}}\sqrt{\frac{\ln D}{T}} + 
		\frac{\sqrt{\ln D}}{(1-\gamma)^{3}}\frac{1}{T^{1-\frac{\alpha}{2}}} \right)\\
		&\qquad + \widetilde{\OM}  \left(
		\frac{1}{(1-\gamma)^{3+\frac{2}{1-\alpha}}} \frac{1}{T}
		+ \frac{\gamma L}{(1-\gamma)^{4+\frac{1}{1-\alpha
		}}} \frac{1}{T^\alpha}
		\right).
		\end{split}
		\end{align*}
		\item If $\eta_t = \frac{1}{1+(1-\gamma)t}$, it follows that for all $T \ge 1$,  $\EB\|\bar{\Q}_T - \Q^* \|_{\infty}=$
        \begin{equation*}
        \label{eq:bound-lin}
        \OM \left(
        \sqrt{\frac{\|\Var(\Z)\|_{\infty}}{(1-\gamma)^2}}\sqrt{\frac{\ln D}{T}} \right)    + \TOM\left( \frac{L}{(1-\gamma)^6} \frac{1}{T} 
        \right).
        \end{equation*}
	\end{itemize}
	Here $\TOM(\cdot)$ hides polynomial dependence on $\alpha, L$ and logarithmic factors (i.e., $\ln D$ and $\ln T$).
\end{thm}

%\begin{rem}
%The unique optimal policy assumption can be replaced by the following one.
%For any $Q$-function estimator $\Q \in \RB^D$, $	\|(\PP^{\pi_{\Q}}-\PP^{\pi^*})(\Q - \Q^*)\|_{\infty} \le L \|\Q - \Q^*\|_{\infty}^2$ where $\pi_{Q}$ is the greedy policy with respective to $Q$ defined by $\pi_{Q}(s) := \arg\max_{a \in \AM} Q(s, a)$.
%The assumption has been used for nonasymptotic analysis for variance reduced Q-learning~\citep{khamaru2021instance} or for policy iteration~\citep{puterman1979convergence}.
%When Assumption~\ref{asmp:gap} holds, the above inequality holds with $L = \frac{4}{\gap}$ (see Lemma~\ref{lem:gap}).
%\end{rem}

\paragraph{Instance-dependent behavior.}
For the polynomial step size, Theorem~\ref{thm:con-linear} shows that the instance-dependent term $\OM(
\sqrt{\|\diag(\VQ)\|_{\infty}}\sqrt{\frac{\ln D}{T}} )$ dominates the $\ell_{\infty}$ error, which matches the instance-dependent lower bound established by~\citet{khamaru2021instance} given a sufficiently large $T$.
To the best of our knowledge, this is the first finite-sample analysis of averaged Q-learning in the $\ell_{\infty}$-norm showing  instance-dependent optimality.
However, for the linearly rescaled step size, we see that $\OM \left(
\sqrt{\frac{\|\Var(\Z)\|_{\infty}}{(1-\gamma)^2}}\sqrt{\frac{\ln D}{T}} \right)$ is the dominant factor, which is larger because we have
\begin{align*}
	\label{eq:VQ-VZ}
	\|\diag(\VQ)\|_{\infty} 
	&\overset{(a)}{\le}	\|(\I-\gamma \PP^{\pi^*})^{-1}\|_{\infty}^2 \|\Var(\Z)\|_{\infty} \\
	&\overset{(b)}{\le} \frac{1}{(1-\gamma)^2} \|\Var(\Z)\|_{\infty},
\end{align*}
where $(a)$ uses $\|\diag(\A \V \A^\top)\|_{\infty} \le  \|\V\|_{\infty} \|\A\|_{\infty}^2$ for any diagonal matrix $\V$ (see Lemma~\ref{lem:V}) and $(b)$ uses $\|(\I-\gamma \PP^{\pi^*})^{-1}\|_{\infty} \le (1-\gamma)^{-1}$.
Hence, the linearly rescaled step size doesn't match the instance-dependent lower bound.
It might be true because the linearly rescaled step size doesn't satisfy Assumption~\ref{asmp:lr}, implying that~\eqref{eq:asym} does not necessarily hold for it.
%A further exploration is required.

\paragraph{Comparison with variance-reduced Q-learning.}
Under the same assumptions, \citet{khamaru2021instance} analyzed a variance-reduced variant of Q-learning that also achieves instance-dependent optimality with the following guarantee: 
\begin{align*}
\lefteqn{\EB\|\widehat{\Q}_T - \Q^* \|_{\infty}=} \\
& \OM \left(
\sqrt{\|\diag(\VQ)\|_{\infty}}\sqrt{\frac{\ln D}{T}} \right)    + \TOM\left( \frac{1}{(1-\gamma)^2} \frac{1}{T} 
\right),
\end{align*}
which has a better nonleading term than averaged Q-learning.
This might somewhat explain the finding of~\citet{khamaru2021temporal} that averaging can be sub-optimal in the nonasymptotic regime with limited samples.  
However, the dominant terms are equal, implying that averaging is still powerful and efficient in the asymptotic regime.
Instance-dependent convergence with a variance structure in the dominant term has also been found for other settings; please see Appendix~\ref{sec:related}.

%optimal value estimation~\citep{zanette2019tighter,yin2021towards}, and for policy evaluation in the tabular setting~\citep{pananjady2020instance} or using linear function approximation ~\citep{li2021accelerated}, 

\paragraph{Worst-case behavior.}
The instance-dependent bound provides more information about the convergence rate.
Previous works~\citep{azar2013minimax,li2020breaking} imply the worst-case bound $\|\diag(\VQ)\|_{\infty} = \OM((1-\gamma)^{-3})$.
Such a dependence on $(1-\gamma)^{-1}$ is tight, because~\citet{khamaru2021instance} constructs 
a family of MDPs parameterized by $\lambda \ge 0$ where $\|\diag(\VQ)\|_{\infty} = \Theta((1-\gamma)^{-3+\lambda})$.
When plugging in the worst-case bound, we find that for polynomial step sizes and for sufficiently small $\varepsilon$, averaged Q-learning already achieves the optimal minimax sample complexity $\TOM\left( \frac{D}{(1-\gamma)^3\varepsilon^2} \right)$ established by~\citet{azar2013minimax}.
\citet{wainwright2019variance} uses a variance-reduced variant of Q-learning to achieve the optimality, but the algorithm requires an additional collection of i.i.d.\ samples at each outer loop to obtain an Monte Carlo approximation of the population Bellman operator~\eqref{eq:T}.
Our results show that a simple average is sufficient to guarantee optimality.
Moreover, the computation of $\bar{\Q}_T$ is fully online with no additional samples needed.

%On the other hand, 
%To obtain the optimal complexity, prior works utilize variance reduction technique~\citep{wainwright2019variance,khamaru2021instance,mou2020linear}, a procedure used to increase the precision of estimates

%As discussed in Section~\ref{sec:asym-ave-Q}, when $T$ is sufficiently large, the dominant term of $\EB \|\bar{\Q}_T-\Q^*\|_{\infty}$ is approximately $\sqrt{\frac{\ln D}{T}}\sqrt{\|\diag(\VQ)\|_{\infty}}$.
%Hence, the remaining issue is to determine the sharpest dependence of $\|\diag(\VQ)\|_{\infty}$ on $(1-\gamma)^{-1}$.

%However, to provide a finite analysis on the $\ell_{\infty}$-norm, we should also capture the tail distribution of each coordinate of $\Z$.
%Additionally, sometimes $\|\Var(\Z)\|_{\infty}$ matters since it affects the sample complexity as $\|\diag(\VQ)\|_{\infty}$ does.

\begin{figure}[t!]
	\centering
	\includegraphics[width=\columnwidth]{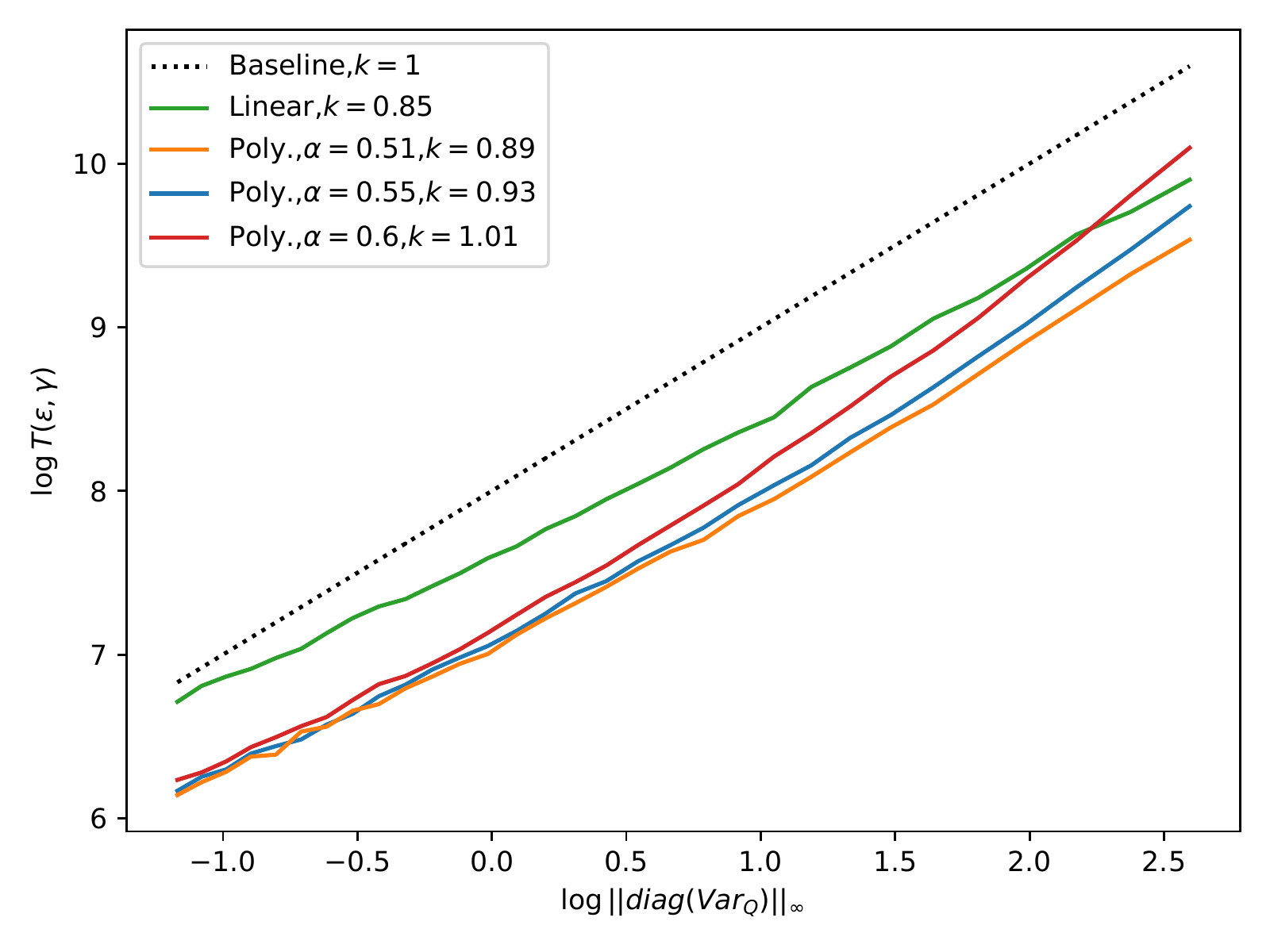}
	\vspace{-0.3in}
	\caption{Log-log plots of the sample complexity $T(\varepsilon, \gamma)$ versus the asymptotic variance $\|\diag(\VQ)\|_\infty$.
%		 (left) and versus the discount complexity parameter $(1-\gamma)^{-1}$ (right).
	  }
	\label{fig:exp}
\end{figure}

\paragraph{Confirming the theoretical predictions.}
We provide numerical experiments to illustrate instance-adaptivity as well as the worst-case behavior delineated in Theorem~\ref{thm:con-linear}.
We focus on  the sample complexity $T(\varepsilon, \gamma)= \inf\{ T: \EB\|\bar{\Q}_T -\Q^*\|_{\infty} \le \varepsilon \}$ for $\varepsilon = 10^{-4}$.
We conduct $10^3$ independent trials in a random MDP to compute $T(\varepsilon, \gamma)$ for different values of $\gamma \in \Gamma$ and two step sizes.
We plot the least-squares fits, $\{ (\log\|\diag(\VQ)\|_\infty, \log T(\varepsilon, \gamma)) \}_{\gamma \in \Gamma}$, and provide the slopes $k$ of these lines in the legend.
Further details are provided in Appendix~\ref{appen:exp}.
At a high level, we see that averaged Q-learning produces sample complexity that is well predicted by our theory---all the slopes are no larger than the theoretical limit $k$ predicted by our theory.

\paragraph{Proof Sketch.}
The proof idea of Theorem~\ref{thm:con-linear} is based on that of Theorem~\ref{thm:fclt}.
Notice that $ \bar{\Q}_T - \Q^* = \frac{1}{T}\sum_{t=1}^T \De_t = \frac{1}{\sqrt{T}} \Bphi_T(1)$.
From~\eqref{eq:error0}, we know that\footnote{Since $r=1$, $\Bpsi_{4}$ doesn't appear in the decomposition.}
\[
\EB\|\bar{\Q}_T - \Q^*\|_{\infty} = \EB\left\|\frac{\Bphi_T(1)}{\sqrt{T}} \right\|_{\infty}
\le \frac{\sum_{i \neq 4} \EB\left\| \Bpsi_i(1)\right\|_{\infty}}{\sqrt{T}}  .
\]
Bounding the term of $i=0$ is easy since it's deterministic.
Because $\Bpsi_i(1) (i=1,2,3)$ is a weighted sum of martingale differences, we use the variance-aware multi-dimensional Freedman’s inequality (in Lemma~\ref{lem:freedman}) to analyze its expectation under $\ell_{\infty}$-norm.
The instance-dependent dominant term comes from the variance term for $\EB \| \Bpsi_1(1)\|_{\infty}$.
Analyzing the variance of $\Bpsi_2(1)$ is quite challenging since it relies on $\frac{1}{T}\sum_{j=1}^T \|\A_j^T - \sG^{-1}\|_{\infty}^2$ with $\A_j^T$ defined in~\eqref{eq:AjT_main}.
We then bound that quantity in terms of $\alpha, 1-\gamma$ and $T$ in Lemma~\ref{lem:G-poly}.
Finally, due to $\|\Bpsi_5(1)\|_{\infty} \le L \|\De_t\|_{\infty}^2$, bounding $\EB\|\Bpsi_5(1)\|_{\infty}$ is reduced to bound $\EB\|\De_t\|_{\infty}^2$ for all $t \ge 0$, which can be given by a similar argument from~\citet{wainwright2019stochastic}.
Putting all pieces together completes the proof; the detailed proof is in Appendix~\ref{proof:con-linear}.

\section{RELAXATION OF THE LIPSCHITZ CONDITION}
%Relaxation of the Lipschitz Condition}
Both our asymptotic and nonasymptotic analysis rely on the Lipschitz condition in Assumption~\ref{asmp:gap}.
That condition is essentially equivalent to assuming a unique optimal policy.  It turns out that, once regularized by entropy, the (regularized) optimal policy is naturally unique.
In the following, we show that entropy-regularized Q-learning enjoys a similar functional CLT and instance-dependent bounds without Assumption~\ref{asmp:gap}.

Entropy-regularized Q-learning uses the following matrix-form update rule,
\begin{equation}
	\label{eq:r-Q-update}
\TQ_t = (1-\eta_t) \TQ_{t-1} + \eta_t \widetilde{\TM}_t \TQ_{t-1},
\end{equation}
where
\begin{equation}
	\label{eq:em-rT}
	\widetilde{\TM}_t(\Q)(s, a) = \rr_t(s, a) + \gamma (\PP_t\LM_{\lambda} \Q)(s_t)
\end{equation} 
is a soft version of the empirical Bellman operator $\widehat{\TM}$.
The nonlinear operator $\LM_{\lambda}(\cdot): \RB^D \to \RB^S$ is a soft version of a hard max, with regularization coefficient $\lambda$.  It is defined by
\begin{equation*}
	\label{eq:L}
	(\LM_{\lambda}\Q)(s) :=  \max_{\pi \in \Pi} \EB_{a \sim \pi(\cdot|s)}\left[Q(s, a)  - \lambda \log \pi(a|s) \right].
\end{equation*}
Let $\Q_{\lambda}^*$ denote the unique fixed point of the regularized Bellman equation $\Q_\lambda^* = \rr + \gamma \PP  \LM_\lambda \Q_\lambda^*$ and let $\pi_\lambda^*$ be the unique optimal policy.

\begin{thm}
	\label{thm:fclt-entropy}
	Define $\{\widetilde{\Q}_t\}_{ t \ge0}$ in~\eqref{eq:r-Q-update}.
	The corresponding partial-sum process is $\Tphi_T(r) := \frac{1}{\sqrt{T}}  \sum_{t=1}^{\Tr} (\TQ_t - \Q_{\lambda}^*)$.
		Under Assumptions~\ref{asmp:reward} and~\ref{asmp:lr}, 
%		if $\frac{1}{\sqrt{T}}\sum_{t=1}^T \lambda_{t} = o(1)$, 
	\begin{equation*}
	\Tphi_T(\cdot) \overset{w}{\to} \TVQ^{1/2} \B_D(\cdot),
	\end{equation*}
%where $\TVQ$ is defined in~\eqref{eq:opt-variance}.
where $\TVQ$ is the asymptotic matrix defined by
	\begin{equation*}
\TVQ  := 
	(\I-\gamma \PP^{\pi_{\lambda}^*})^{-1}\Var(\TZ)(\I-\gamma \PP^{\pi_{\lambda}^*})^{-\top}.
\end{equation*}
with $\TZ \overset{d.}{=} \TZ_t =  (\rr_t-\rr) + \gamma( \PP_t - \PP) \LM_\lambda \Q_\lambda^*$ the regularized Bellman noise.
\end{thm}

\begin{thm}
\label{thm:nonlinear-entropy}
Under Assumptions~\ref{asmp:reward} and~\ref{asmp:lr}, when the two step sizes are considered, $\EB\|\frac{1}{T}\sum_{t=1}^T \TQ_t - \Q_{\lambda}^*\|_{\infty}$ has similar bounds as in Theorem~\ref{thm:con-linear} except that we replace $\VQ, L$ with $\TVQ$ and $\frac{1}{\lambda}$.
\end{thm}

We note that the two theorems in this section can be proved via an almost identical argument as Theorem~\ref{thm:fclt} and~\ref{thm:con-linear}, since Assumption~\ref{asmp:gap} is naturally satisfied with $L = \frac{1}{\lambda}$ for entropy-regularized Q-learning (see Appendix~\ref{proof:entropy}).
Actually, our proof is applicable to a class of nonlinear SAs.\footnote{More specifically, our method can analyze $\Q_{t} = (1-\eta_t)\Q_{t-1} + \eta_t (\rr_t + \gamma \PP_t \LM \Q_{t-1})$ where $\LM$ is a smooth nonlinear non-expansive operator.}
Second, due to the bias introduced by entropy, the instance-dependent factor changes from $\VQ$ to $\TVQ$ and $\frac{1}{T}\sum_{t=1}^T \TQ_t$ converges to $\Q_{\lambda}^*$ instead of $\Q^*$ in expectation.
Finally, note that these results provide a new argument for the benefits of entropy regularization; it smooths the Bellman operator and weakens the assumptions required for asymptotic analysis.
It is supplementary to previous efforts that shows entropy regularization aids exploration~\citep{fox2016taming}, encourages robust optimal policies~\citep{eysenbach2021maximum}, induces a smoother landscape~\citep{ahmed2019understanding}, and hastens the convergence of RL algorithms~\citep{cen2022fast}.

% \vspace{-0.1in}
\section{DISCUSSION}
\label{sec:conclusion}
We have studied the asymptotic and nonasymptotic convergence of averaged Q-learning, establishing its statistical efficiency.
We first established a functional central limit theorem, showing that the standardized partial-sum process converges weakly to a rescaled Brownian motion, a result which can serve as an underpinning for the development of statistical inference methods for RL.
We then established a semiparametric efficiency lower bound for $\Q^*$ estimation,  showing that the averaged iterate $\bar{\Q}_T$ is the most efficient RAL estimator in the sense of having the smallest asymptotic variance.
Finally, we presented the first finite-sample error analysis of $\EB\|\bar{\Q}_T-\Q^*\|_{\infty}$ in the $\ell_{\infty}$-norm for both linearly rescaled and polynomial step sizes.
We showed that averaged Q-learning achieves the same instance-dependent optimality and worst-case optimality as previous variance-reduced algorithms~\citep{khamaru2021instance,wainwright2019variance} under a Lipschitz condition.

% There are many directions to extend our work.
% It is possible to extend our both asymptotic and nonasymptotic analysis methods to a more general problem, nonlinear stochastic approximation.
% What's more, it is promising to relax the generator model assumption and consider the Markov sampling setting where only one sample from current state-action pair is available at a time due to prior success in~\citep{konda2002actor}.

Some open problems remain.
On the one hand, with the Lipschitz condition, it's unclear whether averaged Q-learning with linearly rescaled step sizes can match the instance-dependent lower bound.
Additionally, we suspect that the dependence on $(1-\gamma)^{-1}$ of the nonleading terms in Theorem~\ref{thm:con-linear} is loose and speculate it can be improved by finer analysis.
On the other hand, without the Lipschitz condition, it is not clear whether averaged Q-learning still achieves the optimal instance-dependent bound.
Finally, previous analysis~\citep{kozuno2022kl} shows the last-iterate entropy-regularized Q-learning is minimax optimal.
It is also unknown whether the averaged iterates of entropy-regularized Q-learning achieve the optimal instance-dependent bound.

\section*{Acknowledgement}
The authors would like to express their gratitude to Prof. Csaba Szepesvári for his valuable suggestion regarding the relaxation of the Lipschitz condition through entropy regularization.
Xiang Li and Zhihua Zhang have been supported  by the National Key Research and Development Project of China (No. 2022YFA1004002) and the National Natural Science Foundation of China (No. 12271011).

\bibliography{bib/Qlearning,bib/refer,bib/stat}

\appendix
\begin{appendix}
	\onecolumn

\section{RELATED WORK}
\label{sec:related}
Due to the rapidly growing literature on Q-learning, we review only the theoretical results that are most relevant to our work.
Interested readers can check references therein for more information.

\paragraph{Asymptotic normality in RL.}
Establishing asymptotic normality of an estimator permits statistical inference and the quantification of uncertainty. Existing work on statistical inference for Q-learning has focused mainly on the off-policy evaluation (OPE) problem, where one aims to estimate the value function of a given policy using pre-collected data.  In this setting, a parametric Cramer–Rao lower bound has been established by~\citet{jiang2016doubly}, and asymptotic efficiency has been established for certain estimators using linear approximation~\citep{uehara2020minimax,hao2021bootstrapping,yin2020asymptotically,mou2020linear} or bootstrapping~\citep{hao2021bootstrapping}.  Further inferential work includesthe asymptotic analysis of multi-stage algorithms~\citep{luckett2019estimating, shi2020statistical},
asymptotic behavior of robust estimators~\citep{yang2021towards}, and work by~\citet{kallus2020double} on a semiparametric doubly robust estimator.

In contradistinction to existing work, we establish a functional central limit theorem that captures the weak convergence of the whole trajectory rather than its endpoint.
 Such functional results have not been presented previously in the RL literature.
Furthermore, we supplement these upper bounds with a semiparametric efficiency lower bound which additionally considers the randomness of rewards.
We also show that averaged Q-learning is the most efficient RAL estimator vis-a-vis this lower bound.

% For example, ~\cite{uehara2020minimax,hao2021bootstrapping,yin2020asymptotically} discussed the asymptotic normality and efficiency of porposed OPE estimators in the tabular MDPs with linear approximation.
% ~\cite{kallus2020double} studied a semiparametric efficient doubly robust estimator. 
% Prior to that, \citep{jiang2016doubly} provided the first Cramer–Rao lower bound for the OPE problem. 
% In our work, we provide a counterpart result for the $\gamma$-discounted MDP. Our results show that the averaged Q-learning is asymptotic normal with typical rate $\sqrt{T}$. In addition, we also show that the averaged Q-learning is asymptotic efficient because its asymptotic variance achieves the semiparametric efficiency bound.

\paragraph{Sample complexity for Q-learning.}
For the goal of obtaining an $\varepsilon$-accurate estimate of the optimal Q-function in a $\gamma$-discounted MDP in the presence of a generative model, model-based Q-value-iteration has been shown to achieve optimal minimax sample complexity $\TOM\left( \frac{D}{\varepsilon^2(1-\gamma)^3}\right)$~\citep{azar2013minimax,agarwal2020model,li2020breaking}.
In the model-free context,~\citet{wainwright2019stochastic} showed empirically that classical Q-learning suffers from at least worst-case fourth-order scaling in $(1-\gamma)^{-1}$ in sample complexity.
A complexity bound of $\TOM\left(\frac{D}{\varepsilon^2(1-\gamma)^5}\right)$ has been provided~\citep{wainwright2019stochastic,chen2020finite}; this is far from the optimal though better than previous efforts~\citep{even2003learning,beck2012error}.
~\citet{li2021q} gave a sophisticated analysis showing the complexity of Q-learning is $\TOM\left( \frac{D}{\varepsilon^2(1-\gamma)^4}\right)$ and provided a matching lower bound to confirm its sharpness.
%By contrast, our sub-gaussian result produces a better variance bound on Bellman noises and thus directly improves the analysis of~\citep{wainwright2019stochastic} to the optimal without additional efforts.
~\citet{wainwright2019variance,khamaru2021instance} introduced a variance-reduced variant of Q-learning~\citep{gower2020variance} that achieves the optimal sample complexity and instance complexity.
%~\citep{khamaru2021instance} argues there exists an instance-dependent functional that controls the difficulty of estimation and thus effects the sample complexity.
Our results show that a simple average over all history $Q_t$ is sufficient to guarantee the same optimality.
The averaged method is fully online without requiring additional samples and storage space.

\paragraph{Instance-dependent  convergence in RL.}
Recent years have witnessed new instance-specific bounds, where an instance-dependent functional of a variance structure appears as the dominant term on stochastic errors.
Unlike global minimax bounds which are worst-case in nature, instance-specific bounds help identify the difficulty of estimation case by case.
Such bounds have been established for policy evaluation in the tabular setting~\citep{pananjady2020instance,khamaru2021temporal,li2020breaking} or with linear function approximation~\citep{li2021accelerated} and for optimal value function estimation~\citep{yin2021towards}.
The most related work to ours is by~\citet{khamaru2021instance}, who show that a variance-reduced variant of Q-learning achieves the instance-dependent optimality after identifying an instance-dependent lower bound for $Q^*$ estimation.
By contrast, our result shows that a simple average is sufficient to yield optimality.

\paragraph{Nonlinear stochastic approximation.}
Q-learning has also been studied through the lens of nonlinear stochastic approximation.
From this general point of view, asymptotic convergence has been provided~\citep{tsitsiklis1994asynchronous,borkar2000ode}.
On the nonasymptotic side, Q-learning is studied either in the synchronous setting~\citep{shah2018q,wainwright2019stochastic,chen2020finite} or the asynchronous setting where only one sample from current state-action pair is available at a time~\citep{qu2020finite,li2020sample,chen2021lyapunov}.
The sample complexities obtained therein are far from optimal.
Others consider Q-learning with linear function approximation in the $\ell_2$-norm~\citep{melo2008analysis,chen2019finite}.
Asymptotic convergence of averaged Q-learning has been studied by~\citet{lee2019target,lee2019unified} via the ODE (ordinary differential equation) approach.
Our results are complementary to these results, including asymptotic statistical properties and finite-sample analysis in the $\ell_{\infty}$-norm.
Though peculiar to averaged Q-learning, we believe our analysis can be extended to nonlinear SA problems.

\section{CENTRAL LIMIT THEOREM FOR AVERAGED Q-LEARNING}
%Central Limit Theorem (CLT) for Averaged Q-Learning} 
\label{sec:clt}
For completeness, we present a CLT for the averaged Q-learning sequence $\bar{\Q}_T := \frac{1}{T}\sum_{t=1}^T \Q_t$ in this part.
This result can be derived not only from our Theorem~\ref{thm:fclt} but also from CLT for non-linear SA, e.g.,~\citep{mokkadem2006convergence}.

\begin{thm}[Asymptotic normality for $Q^*$]
	\label{thm:clt}
	Under Assumptions~\ref{asmp:reward},~\ref{asmp:gap} and~\ref{asmp:lr}, we have
	\[
	\sqrt{T} (\bar{\Q}_T - \Q^*)  \overset{d}{\to} \NM(\0, \VQ),
	\]
	where the asymptotic variance is given by
	\begin{equation}
%		\label{eq:opt-variance}
		\VQ  = 
		(\I-\gamma \PP^{\pi^*})^{-1}\Var(\Z)(\I-\gamma \PP^{\pi^*})^{-\top} \in \RB^{D \times D}.
	\end{equation}
	Here $\Var(\Z)$ is the covariance matrix of the Bellman noise $\Z$ defined in~\eqref{eq:V-bellman}.
\end{thm}

\paragraph{Asymptotic variance.}
Theorem~\ref{thm:clt} implies that the average of the sequence $(\Q_t)$ has an asymptotic normal distribution with $\VQ$ the asymptotic variance.
$\VQ$ includes $\Var(\Z)$, the covariance matrix of Bellman noise $\Z$, multiplied with a pre-factor $(\I - \gamma \PP^{\pi^*})^{-1}$.
By a von Neumann expansion, $(\I - \gamma \PP^{\pi^*})^{-1}$ is equivalent to $\sum_{t=0}^{\infty} (\gamma \PP^{\pi^*})^t$.
As argued by~\citet{khamaru2021instance}, the sum of the powers of $\gamma \PP^{\pi^*}$ accounts for the compounded effect of an initial perturbation when following the MDP induced by $\pi^*$.
The Bellman noise $\Z$ reflects the noise present in the empirical Bellman operator~\eqref{eq:em-T} as an estimate of the population Bellman operator~\eqref{eq:T}.
Note that this implies $\|(\I - \gamma \PP^{\pi^*})^{-1}\| \le \sum_{t=0}^{\infty} \gamma^t \|( \PP^{\pi^*})^t\|_{\infty} = (1-\gamma)^{-1}$.
$\|\diag(\VQ)\|_{\infty}$ coincides with the instance-dependent functional proposed by~\citet{khamaru2021instance} that controls the difficulty of estimating $Q^*$ in the $\ell_{\infty}$-norm.
%They imply $\|\diag(\VQ)\|_{\infty}$ quantifies the difficulty of $Q^*$ estimation.

\paragraph{Asymptotic normality for $V^*$ estimation.}
If the optimal policy is unique, we can obtain a similar result for the optimal value function $V^*$, making use of the asymptotic normality of $\bar{\Q}_T$. We define an estimator $\bar{\V}_T \in \RB^S$ greedily from $\bar{\Q}_T \in \RB^{D}$: the $s$-th entry of $\bar{\V}_T$ is $\bar{V}_T(s) \in \arg\max_{a \in \AM} \bar{Q}_T(s, a)$.
As a corollary of Theorem~\ref{thm:clt}, $\bar{\V}_T$ enjoys a similar asymptotic normality with the asymptotic variance defined by $\VV$.
One can check that 
\begin{equation}
%	\label{eq:VV-VQ}
	\VV  =\boldsymbol{\Pi}^{\pi^*} \VQ (\boldsymbol{\Pi}^{\pi^*})^\top,
\end{equation}
where $\boldsymbol{\Pi}^{\pi^*} \in \{0, 1\}^{S \times D}$ is the projection matrix associated with the deterministic optimal policy $\pi^*$ (see~\eqref{eq:project-pi}).
Hence, $\VV$ is formed by selecting entries from $\VQ$.
In particular, $\VV(s, s') = \VQ((s, \pi^*(s)), (s', \pi^*(s')))$ for any $s, s' \in \SM$.
The proof is deferred to Appendix~\ref{proof:V-asym}.

\begin{lem}
	\label{lem:gap}
	If $\pi^*$ is unique, then we have a positive optimality gap $\gap := \min_{s} \min_{a \neq \pi^*(s)}| V^*(s)-Q^{*}(s, a)| > 0$ where $\pi^*(s)$ is the unique action satisfying $V^*(s) = Q^*(s, a^*(s))$.
	% If $\gap = \min_{s} \min_{a':Q^{*}(s, a') \neq V^*(s) }| V^*(s)-Q^{*}(s, a')| > 0$, then the optimal policy $\pi^*$ is unique.
	For any $Q$-function estimator $\Q \in \RB^D$, it follows that $\{ \pi_{\Q} \neq \pi^* \} \subseteq \{   \|\Q-\Q^*\|_{\infty} \ge \frac{\gap}{2} \}$ and 
	\begin{equation}
		\label{eq:Lip}
		\|(\PP^{\pi_{\Q}}-\PP^{\pi^*})(\Q - \Q^*)\|_{\infty} \le L \|\Q - \Q^*\|_{\infty}^2
		\ \text{with} \
		L = \frac{4}{\gap},
	\end{equation}
	where $\pi_{Q}$ is the greedy policy with respective to $Q$ defined by $\pi_{Q}(s) := \arg\max_{a \in \AM} Q(s, a)$.
	If $\arg\max_{a \in \AM} Q(s, a)$ has more than one element, we break the tie by randomness. 
\end{lem}

\begin{cor}[Asymptotic normality for $V^*$]
	\label{cor:V}
	Let $\bar{\V}_T \in \RB^{S}$ be the greedy value function computed from $\bar{\Q}_T \in \RB^{D}$, i.e., $\bar{V}_T(s) \in \arg\max_{a \in \AM} \bar{Q}_T(s, a)$.
	Under Assumptions~\ref{asmp:reward} and~\ref{asmp:lr}, if we assume the optimal policy $\pi^*$ is unique, then
	\[
	\sqrt{T} (\bar{\V}_T - \V^*)  \overset{d}{\to} \NM(\0, \VV),
	\]
	where the asymptotic variance is
	\begin{equation}
		\label{eq:opt-variance-V}
		\VV  = 
		(\I-\gamma \PP_{\pi^*})^{-1}\Var(\boldsymbol{\Pi}^{\pi^*}\Z)(\I-\gamma \PP_{\pi^*})^{-\top}  \in \RB^{S \times S},
	\end{equation}
	and $\Var(\boldsymbol{\Pi}^{\pi^*}\Z)$ is the covariance matrix of the projected Bellman noise $\boldsymbol{\Pi}^{\pi^*}\Z$.
\end{cor}

\paragraph{Insights on sample efficiency.}
The asymptotic results shed light on the sample efficiency of averaged Q-learning.
Under ideal conditions, we have
\begin{equation}
	\label{eq:asym}
	\sqrt{T} \EB \| \bar{\Q}_T - \Q^*\|_{\infty} \to \EB\|\ZM\|_{\infty} \approx \sqrt{\ln D} \sqrt{\|\diag(\VQ)\|_{\infty}}
	\ \text{where} \  \ZM \sim \NM(\0, \VQ).
\end{equation}
In this case, roughly speaking, to obtain an $\varepsilon$-accurate estimator of the optimal Q-value function $\Q^*$ (i.e., $ \EB \| \bar{\Q}_T - \Q^*\|_{\infty} \le \varepsilon$), we require approximately $T = \OM \left(\frac{\ln D }{\varepsilon^2}\|\diag(\VQ)\|_{\infty}\right)$ iterations or equivalently $DT = \OM \left(\frac{D\ln D }{\varepsilon^2}\|\diag(\VQ)\|_{\infty}\right)$ samples.
This explains why~\citet{khamaru2021instance} regarded $\|\diag(\VQ)\|_{\infty}$ as the difficulty indicator because it affects the sample complexity directly.
%The remaining question is to figure out how $\|\diag(\VQ)\|_{\infty}$ depends on $(1-\gamma)^{-1}$.

\subsection{Proof of Theorem~\ref{thm:clt}}
\label{proof:clt}
\begin{proof}[Proof of Theorem~\ref{thm:clt}]
	One can prove Theorem~\ref{thm:clt} by applying continuous mapping theorem to Theorem~\ref{thm:fclt} with the functional $f: \BDD \to \RB^D, f(w) = w(1)$.
	Once we can prove $f$ is a continuous functional in $(\BDD, d_0)$, an application of~\eqref{eq:BC-con} would conclude the proof.
	Recalling the metric~\eqref{eq:d0} defined on $\BDD$, we have for any $w_1, w_2 \in \BDD$,
	\begin{align*}
		\|f(w_1)-f(w_2)\|_{\infty}
		&=\| w_1(1) - w_2(1)\|_{\infty} \le\inf_{\lambda \in \Lambda} \sup_{t \in [ 0,1]}\|w_1(\lambda(t)) - w_2(t)\|_{\infty}\\
		&\le \inf_{\lambda \in \Lambda} \left\{
		\sup_{0\le s < t \le 1}\left|\ln\frac{\lambda(t)-\lambda(s)}{t-s}\right|
		+ \sup_{t \in [ 0,1]}\|w_1(\lambda(t)) - w_2(t)\|_{\infty}
		\right\}
		= d_0(w_1, w_2).
	\end{align*}
	We even show that $f$ is $1$-Lipschitz continuous in $(\BDD, d_0)$ and thus complete the proof.
\end{proof}

\subsection{Proof of Corollary~\ref{cor:V}}
\begin{proof}[Proof of Corollary~\ref{cor:V}]
	\label{proof:V-asym}
	We first prove 
	\begin{equation}
		\label{eq:var-equal}
		\VV  =\boldsymbol{\Pi}^{\pi^*} \VQ (\boldsymbol{\Pi}^{\pi^*})^\top.
	\end{equation}
	Recall the definition
	\begin{gather*}
		\VQ  = 
		(\I-\gamma \PP^{\pi^*})^{-1}\Var(\Z)(\I-\gamma \PP^{\pi^*})^{-\top} \in \RB^{D \times D}\\
		\VV  = 
		(\I-\gamma \PP_{\pi^*})^{-1}\Var(\boldsymbol{\Pi}^{\pi^*}\Z)(\I-\gamma \PP_{\pi^*})^{-\top}  \in \RB^{S \times S}.
	\end{gather*}
	For one thing, we have $\Var(\boldsymbol{\Pi}^{\pi^*}\Z)= \boldsymbol{\Pi}^{\pi^*}\Var(\Z)(\boldsymbol{\Pi}^{\pi^*})^\top$.
	For another thing, we have $\boldsymbol{\Pi}^{\pi^*}(\I-\gamma \PP^{\pi^*})^{-1} = (\I-\gamma \PP_{\pi^*})^{-1}\boldsymbol{\Pi}^{\pi^*}$.
	This is because 
	\[
	(\I-\gamma \PP_{\pi^*})\boldsymbol{\Pi}^{\pi^*}  = \boldsymbol{\Pi}^{\pi^*}- \gamma \boldsymbol{\Pi}^{\pi^*} \PP \boldsymbol{\Pi}^{\pi^*}= \boldsymbol{\Pi}^{\pi^*}(\I-\gamma \PP^{\pi^*}).
	\]
	Putting these together,~\eqref{eq:var-equal} follows from direct verification.
	
	We then prove the asymptotic normality of $\bar{\V}_T$.
	Let $\bar{\pi}_t$ is the greedy policy with respect to $\bar{\Q}_t$, i.e., $\bar{\pi}_t(s) \in \argmax_{a \in \AM} \bar{Q}_t(s, a)$.
	From the definition of our estimator,
	\[
	\bar{\V}_T = \boldsymbol{\Pi}^{\bar{\pi}_T}\bar{\Q}_T
	\quad \text{and} \quad  
	\V^*  =  \boldsymbol{\Pi}^{\pi^*} \Q^*
	\]
	which implies
	\[
	\bar{\V}_T - \V^* 
	=\left(  \boldsymbol{\Pi}^{\bar{\pi}_T}\bar{\Q}_T- \boldsymbol{\Pi}^{\pi^*}\bar{\Q}_T\right)
	+\left(  \boldsymbol{\Pi}^{\pi^*}\bar{\Q}_T- \boldsymbol{\Pi}^{\pi^*} \Q^*\right).
	\]
	On the other hand, it is easy to see that
	\[
	\sqrt{T}\left(  \boldsymbol{\Pi}^{\pi^*}\bar{\Q}_T- \boldsymbol{\Pi}^{\pi^*} \Q^*\right)
	\overset{d.}{\to} \NM(\0, \boldsymbol{\Pi}^{\pi^*} \VQ (\boldsymbol{\Pi}^{\pi^*})^\top)
	= \NM(\0, \VV).
	\]
	If we can prove
	\begin{equation}
		\label{eq:V-asym-op1}
		\sqrt{T}\left(  \boldsymbol{\Pi}^{\bar{\pi}_T}\bar{\Q}_T- \boldsymbol{\Pi}^{\pi^*}\bar{\Q}_T\right) = o_\PB(1),
	\end{equation}
	then the conclusion follows from Slutsky's theorem.
	We have that
	\begin{align*}
		\sqrt{T}\EB \|  \boldsymbol{\Pi}^{\bar{\pi}_T}\bar{\Q}_T- \boldsymbol{\Pi}^{\pi^*}\bar{\Q}_T\|_{\infty}
		& \le \sqrt{T}\EB \|  \boldsymbol{\Pi}^{\bar{\pi}_T}- \boldsymbol{\Pi}^{\pi^*}\|_{\infty}\|\bar{\Q}_T\|_{\infty} \\
		&\overset{(a)}{\le} \frac{\sqrt{T}}{1-\gamma }\EB \|  \boldsymbol{\Pi}^{\bar{\pi}_T}- \boldsymbol{\Pi}^{\pi^*}\|_{\infty}\\
		&\overset{(b)}{=} \frac{2\sqrt{T}}{1-\gamma }\PB \left( \bar{\pi}_{T} \neq \pi^* \right) \\
		&\overset{(c)}{\le} \frac{2\sqrt{T}}{1-\gamma }\PB \left( \|\bar{\Q}_T-\Q^*\|_{\infty} \ge \frac{\gap}{2} \right)\\
		&\le\frac{2\sqrt{T}}{1-\gamma } \frac{4}{\gap^2} \EB\|\bar{\Q}_T-\Q^*\|_{\infty}^2 \\
		&\overset{(d)}{\le}\frac{1}{1-\gamma } \frac{8}{\gap^2} \frac{1}{\sqrt{T}}\sum_{t=1}^T\EB\|{\Q}_t-\Q^*\|_{\infty}^2,
	\end{align*}
	where $(a)$ uses $\|\bar{\Q}_T\|_{\infty} \le (1-\gamma)^{-1}$, $(b)$ uses the fact that both $\bar{\pi}_T$ and $\pi^*$ are deterministic policies and thus $\|  \boldsymbol{\Pi}^{\bar{\pi}_T}- \boldsymbol{\Pi}^{\pi^*}\|_{\infty} = 2 \cdot1_{ \{ \bar{\pi}_T \neq \pi^* \} }$, $(c)$ uses the fact $\{ \bar{\pi}_{t} \neq \pi^* \} \subseteq \{   \|\bar{\Q}_t-\Q^*\|_{\infty} \ge \frac{\gap}{2} \}$ which we derived in Lemma~\ref{lem:gap}, and finally $(d)$ follows from Jensen's inequality.
	
	From Theorem~\ref{thm:general-Linfty-pw2}, we know $\frac{1}{\sqrt{T}}\sum_{t=1}^T\EB\|{\Q}_t-\Q^*\|_{\infty}^2 \to 0$ as $T \to \infty$.
	Therefore, we have that $\sqrt{T}\EB \|  \boldsymbol{\Pi}^{\bar{\pi}_T}\bar{\Q}_T- \boldsymbol{\Pi}^{\pi^*}\bar{\Q}_T\|_{\infty} = o(1)$ which implies~\eqref{eq:V-asym-op1} is true.
\end{proof}

\subsection{Proof of Lemma~\ref{lem:gap}}
\label{proof:gap}
\begin{proof}[Proof of Lemma~\ref{lem:gap}]
	Recall that $\gap = \min_{s} \min_{a \neq \pi^*(s)}| Q^*(s, \pi^*(s))-Q^{*}(s, a)|$.
	If $\gap = 0$, by definition, there must exist some $s_0 \in \SM$ and $a_0 \in \AM$ such that $V^*(s_0) = Q^*(s_0, a_0)$ and $a_0 \neq \pi^*(s_0)$, which is contradictory with the uniqueness of $\pi^*$.
	Hence, a unique $\pi^*$ implies a positive $\gap$.
	
	For any $\Q$ satisfying $\|\Q-\Q^*\|_{\infty} < \frac{\gap}{2}$, we must have $\|\Q(s, \cdot)-\Q^*(s, \cdot)\|_{\infty} < \frac{\mathrm{gap}}{2}$ for any $s \in \SM$.
	In this case, it must be true that $\pi_{Q}(s)= \pi^*(s)$ for all $s \in \SM$.
	Otherwise, there exists some $s \in \SM$ such that $\pi_{Q}(s) \neq \pi^*(s)$.
	We then have
	\begin{gather*}
		Q(s, \pi_{Q}(s)) <Q^*(s, \pi_{Q}(s)) + \frac{\gap}{2} 
		\overset{(a)}{\le}  Q^*(s, \pi^*(s)) - \frac{\gap}{2} 
		< Q(s, \pi^*(s)),
	\end{gather*}
	where $(a)$ follows from the definition of the optimality gap.
	The result $ Q(s, \pi_{Q}(s))  < Q(s, \pi^*(s))$ contradicts with the fact that $\pi_{Q}(s)$ is the greedy policy with respect to $Q$ at state $s$, which implies $Q(s, \pi^*(s)) \le  Q(s, \pi_{Q}(s))$.
	This implies that the event $\{ \pi_{\Q} \neq \pi^* \} \subseteq \{   \|\Q-\Q^*\|_{\infty} \ge \frac{\gap}{2} \}$ and thus $1_{  \{ \pi_{Q} \neq \pi^* \}} \le 1_{ \{   \|\Q-\Q^*\|_{\infty} \ge \frac{\gap}{2} \} }$.
	Hence,
	\begin{align*}
		\|(\PP^{\pi_{\Q}}-\PP^{\pi^*})(\Q - \Q^*)\|_{\infty}
		&\le \|\PP^{\pi_{\Q}}-\PP^{\pi^*}\|_{\infty} \|\Q - \Q^*\|_{\infty}\\
		&\le \|\PP\|_{\infty} \| \BPi^{\pi_\Q} -\BPi^{\pi^*}\|_{\infty} \|\Q - \Q^*\|_{\infty}\\
		&= 1 \cdot 2\cdot 1_{  \{ \pi_{\Q} \neq \pi^* \}} \cdot  \|\Q - \Q^*\|_{\infty}\\
		&\le 2\cdot1_{ \{   \|\Q-\Q^*\|_{\infty} \ge \frac{\gap}{2} \} } \|\Q - \Q^*\|_{\infty} \\
		&\le \frac{4}{\gap}\|\Q - \Q^*\|_{\infty}^2,
	\end{align*}
	where the last line uses $ 1_{ \{   \|\Q-\Q^*\|_{\infty} \ge \frac{\gap}{2} \} } \le \frac{2}{\gap}\|\Q - \Q^*\|_{\infty}$.
\end{proof}

\subsection{Proof of Proposition~\ref{prop:pivotal}}
\begin{proof}[Proof of Proposition~\ref{prop:pivotal}]
	Let $g : \BDD \to \RB$ be a functional defined as
	\[
	g(w) = 
	w(1)^\top \left(   \int_0^1 w(r) w(r)^\top dr \right)^{-1} w(1)
	\ \text{for any} \ w \in \BDD. 
	\]
	Here the domain of $g$ is
	\[
	\mathrm{dom}(g) = \left\{ w \in \BDD, \int_0^1 w(r) w(r)^\top dr  \ \text{is invertible}  \right\}.
	\]
	Once we prove $g$ is continuous in $(\mathrm{dom}(g), d_0)$, the continuous mapping theorem together with Theorem~\ref{thm:fclt} would complete the proof for Proposition~\ref{prop:pivotal}.

	In Appendix~\ref{proof:clt}, we have shown $f: \BDD \to \RB^D, f(w) = w(1)$ is $1$-Lipschitz continuous in $(\BDD, d_0)$.
	Let $h: \BDD \to \RB^{D \times D}$ be defined by $h(w) = \int_0^1 w(r) w(r)^\top dr$.
	Hence, once we prove $h$ is continuous in $(\BDD, d_0)$, it follows that $g = f^\top h^{-1} f$ is also continuous in $(\mathrm{dom}(g), d_0)$.
	To that end, we only show each entry of $h$ is continuous in $w$.
	This is true because of each entry of $h$ is in form of integration which is a continuous functional on the Skorohod space $\BDone$.
	
	Finally, by Theorem~\ref{thm:fclt} and definition of weak convergence, we know that as $T$ goes to infinity,
	\[
	\PB\left(  \Bphi_T \notin \mathrm{dom}(g) \right)
	\to \PB\left(  \B_D \notin \mathrm{dom}(g) \right) = 0.
	\]
	Hence, with probability approaching to one, $\int_0^1\Bphi_T(r)\Bphi_T(r)^\top dr$ is invertible and thus $g(\Bphi_T)$ is well defined.
\end{proof}

\section{PROOF OF THEOREM~\ref{thm:fclt}}
\label{proof:fclt}
\subsection{Preliminaries and High-level Idea}
In this section, we provide a self-contained proof of our functional central limit theorem (FCLT).
Let $\DDelta_t  = \Q_t-\Q^*$ be the error vector at iteration $t$.
The application of Polyak-Ruppert average~\citep{polyak1992acceleration} gives an estimator for $\Q^*$: $\bar{\Q}_T =  \frac{1}{T}  \sum_{t=1}^T\Q_t$.
Then its partial sum of the first $r$-fraction  $(r \in [0, 1])$ is $\frac{1}{T}  \sum_{t=1}^{\Tr}\Q_t$.
The associated standardized partial-sum process is defined by
\[
\Bphi_T(r) = \frac{1}{\sqrt{T}}  \sum_{t=1}^{\Tr} \DDelta_t= \frac{1}{\sqrt{T}}  \sum_{t=1}^{\Tr} (\Q_t - \Q^*).
\]
Here $\Bphi_T(\cdot)$ should be viewed as a $D$-dimensional random function.
For simplicity, we also use $\Bphi_T=\{\Bphi_T(r)\}_{r \in [0, 1]}$ to denote the whole function.

\subsubsection{Weak convergence of measures in Polish spaces}
\label{sec:weak-con}
We will introduce some basic knowledge of weak convergence in metric spaces.
See Chapter VI in~\citep{jacod2003skorokhod} for a detailed introduction.

A Polish space is a topological space that is separable, complete, and metrizable.
Let $\BD = \{ \text{càdlàg function} \ \omega(r) \in \RB^d,r \in [0, 1]  \}$ collect all $d$-dimensional functions which are right continuous with left limits.
Define $\BDM$ as the $\sigma$-field generated by all maps $X \mapsto X(r)$ for $r \in [0, 1]$.
The $J_1$ Skorokhod topology equips $\BD$ with a metric $d_0$ such that $(\BD, d_0)$ is a Polish space and $\BDM$ is its Borel $\sigma$-field (the $\sigma$-field generated by all open subsets). 
In particular, for any $w_1, w_2 \in \BD$,
\begin{equation}
\label{eq:d0}
d_0(w_1, w_2) = \inf_{\lambda \in \Lambda} \left\{
\sup_{0\le s < t \le 1}\left|\ln\frac{\lambda(t)-\lambda(s)}{t-s}\right|
+ \sup_{t \in [ 0,1]}\|w_1(\lambda(t)) - w_2(t)\|_{\infty}
\right\},
\end{equation}
where $\Lambda$ denotes the class of strictly increasing continuous mappings $\lambda: [0,1] \to [0, 1]$ with $\lambda(0) =0$ and $\lambda(1) = 1$.

An important subset of $\BD$ is $\BC = \{ \text{continuous} \ \omega(r) \in \RB^d,r \in [0, 1]  \}$, which collects all $d$-dimensional continuous functions defined on $[0, 1]$.
The uniform topology  equips $\BC$ with the uniform norm
\begin{equation}
\label{eq:norm}
\|\omega\|_{\sup} := \sup_{r \in [0, 1]} \|\omega(r)\|_{\infty}.
\end{equation}
The resulting $(\BC, \|\cdot\|_{\sup})$ is a Polish space.
Additionally, we have $d_0(w_1, w_2) \le \|w_1-w_2\|_{\sup}$ for any $w_1, w_2 \in \BD$.
The $J_1$ Skorokhod topology is weaker than the uniform topology.
%then $\BC$ becomes a metric space.
%It is complete because a uniform limit of continuous functions is continuous.
%It is separable because the set of polynomials on $[0, 1]$ is dense by the Weierstrass approximation theorem and the set of polynomials with rational coefficients is countable
%and dense in the set of all polynomials.
%Therefore, $(\BC, \|\cdot\|_{\sup})$ is a Polish space.
However, if $X \in \BD$ is a continuous function, a sequence $\{X_t\}_{t \ge 0}  \subseteq \BD$ converges to $X$ for the Skorokhod topology if and only if it converges to $X$ under the uniform norm $\|\cdot\|_{\sup}$.
Hence, the Skorokhod topology relativized to $\BC$ coincides with the uniform topology there.

Any random element $X_t \in \BD$ introduces a probability measure on $\BD$ denoted by $\LM(X_t)$ such that $(\BD, \BDM, \LM(X_t))$ becomes a probability space.
We say a sequence of random elements $\{X_t\}_{t \ge 0}  \subseteq \BD$ weakly converges to $X$, if for any bounded continuous function $f: \BD \to \RB$, we have 
\begin{equation}
\label{eq:BC-con}
\EB f(X_T) \to \EB f(X)
\ \text{as} \ T \ \text{goes to infinity}.
\end{equation}
The condition is equivalent to that any finite-dimensional projections of $\Bphi_T$ converge in distribution.
We denote the weak convergence by $X_T \overset{w}{\to} X$.

\begin{thm}[Slutsky's theorem on Polish spaces]
	\label{thm:slutsky}
	Suppose $\SM$ is a Polish space with metric $d$ and $\{(X_t, Y_t)\}_{t \ge 0}$ are random elements of $\SM \times \SM$.
	Suppose $X_T \overset{w}{\to} X$ and $d(X_T, Y_T) \overset{w}{\to} 0$, then $Y_T \overset{w}{\to} X$.
\end{thm}
By Slutsky's theorem in Theorem~\ref{thm:slutsky}, if $\|Y_T\|_{\sup} \overset{w}{\to} 0$ and $X_T \overset{w}{\to} X$, then $X_T+ Y_T \overset{w}{\to} X$.
A sufficient condition to $\|Y_T\|_{\sup} \overset{w}{\to} 0$ is $\EB\|Y_T\|_{\sup} \to 0$ by Markov's inequality.

\begin{prop}
\label{prop:slutsky}
For two random sequences $\{ X_t\}_{t \ge 0}, \{ Y_t \}_{t \ge 0} \subseteq \BD$ satisfying $\EB\|Y_T\|_{\sup} \to 0$ and $X_T \overset{w}{\to} X$, we have $X_T + Y_T\overset{w}{\to} X$.
\end{prop}

%Hence, $\Bphi_T(r)$ introduces a probability measure on $\BC$.
%\footnote{For example, we define $\Bphi_T(r) =  \frac{1}{\sqrt{T}}  \sum_{t=1}^{k} \DDelta_t$ if $r = k/T$ for some $k$ and interpolate these points with linear lines.}

\subsubsection{Proof Idea}
\label{proof:main-proof-idea}
In the following, we will show under the three assumptions in the main text, we can establish
\[
\Bphi_T \overset{w}{\to}  \VQ^{1/2} \B_D,
\]
where $\B_D \in \BCD$ is the standard $D$-dimensional Brownian motion on $[0, 1]$.
That is the associated measure of $\Bphi_T$ weakly converges to the measure introduced by $\VQ^{1/2} \B_D$ on $\BDD$.
%It is equivalent to say for any bounded continuous function $f \in \RB^d$ on $[0, 1]$, $\EB f(\Bphi_T(r)) \to \EB f(\VQ^{1/2} \B_D(r))$ in the sense that $\|\EB f(\Bphi_T(\cdot)) -\EB f(\VQ^{1/2} \B_D(r))\|_{\sup} \to 0$ as $T \to \infty$.

To proceed the proof, we will use two auxiliary sequences $\{ \DDelta_t^1\}_{t \ge 0}$ and $\{ \DDelta_t^2\}_{t \ge 0}$ defined in Lemma~\ref{lem:sandwitch}.
The proof of Lemma~\ref{lem:sandwitch} can be found in Appendix~\ref{proof:sandwitch}.

\begin{lem}
	\label{lem:sandwitch}
	Denote  $\G = \I - \gamma \PP^{\pi^*}, \A_t = \I - \eta_t\G$ and $\W_t =  (\rr_t-\rr) + \gamma( \PP_t - \PP) \V_{t-1}$ for short.
	The auxiliary sequences $\{ \DDelta_t^1\}_{t \ge 0}$ and $\{ \DDelta_t^2\}_{t \ge 0}$ are defined iteratively: $\DDelta_0^1 = \DDelta_0^2 = \DDelta_0$ and for $t \ge 1$
	\begin{gather}
	\DDelta_t^1 =\A_t \DDelta_{t-1}^1+ \eta_t
	\left[\W_t + \gamma
	(\PP^{\pi_{t-1}}-\PP^{\pi^*})\DDelta_{t-1} \right]
	\label{eq:Delta1}\\
	\DDelta_t^2 = \A_t \DDelta_{t-1}^2+ \eta_t \W_t~\label{eq:Delta2}.
	\end{gather}
	As long as $\sup_t \eta_t \le 1$, it follows that all $t \ge 0$,
	\begin{equation}
	\label{eq:Delta-relation}
	\DDelta_t^2 \le  \DDelta_t \le \DDelta_t^1.
	\end{equation}
\end{lem}

The two sequences form a sandwich bound for $\DDelta_t$, producing $\DDelta_t^2 \le  \DDelta_t \le \DDelta_t^1$ coordinate-wise. 
We similarly define the error vectors of their first $r$-fraction partial sums as
\[
\Bphi_T^1(r) := \frac{1}{\sqrt{T}}  \sum_{t=1}^{\Tr} \DDelta_t^1
\ \text{and} \
\Bphi_T^2(r) := \frac{1}{\sqrt{T}}  \sum_{t=1}^{\Tr} \DDelta_t^2.
\]
Then, it is valid that $\Bphi_T^1, \Bphi_T^2 \in \BDD$ and for any $r \in [0, 1]$,
\begin{equation}
\label{eq:sandwitch-asym-nor}
\Bphi_T^2(r)  \le \Bphi_T(r) \le \Bphi_T^1(r).
\end{equation}
In the following subsections, we will show that under Assumption~\ref{asmp:reward},~\ref{asmp:gap} and~\ref{asmp:lr}, we can find a random function $\ZM \in \BDD$ which satisfies
\begin{equation}
\label{eq:Zr}
\ZM \overset{w}{\to}   \VQ^{1/2} \B_D.
\end{equation}
Furthermore, $\Bphi_T^1$ and $\Bphi_T^2$ weakly converge to $\ZM$ such that 
\begin{equation}\label{eq:sandwich-o1}
\lim\limits_{T \to \infty} \EB \| \Bphi_T^{k} - \ZM\|_{\sup} = 0
\ \text{for} \ k=1,2.
\end{equation}
By the sandwich inequality~\eqref{eq:sandwitch-asym-nor}, we have
\[
\EB  \| \Bphi_T - \ZM\|_{\sup} 
\le \EB \| \Bphi_T^{1}- \ZM\|_{\sup}  
+ \EB \| \Bphi_T^{2} - \ZM\|_{\sup} \to 0
\]
as $T$ goes to infinity.
Proposition~\ref{prop:slutsky} implies $\Bphi_T$ weakly converges to a rescaled Brownian motion $\VQ^{1/2} \B_D$, by which we complete the proof.

\subsection{Functional CLT for \texorpdfstring{$ \Bphi_T^1$}{Phi-T-1}}
\label{proof:Delta1-asym}
We first establish the FLCT of $ \Bphi_T^1(r)= \frac{1}{\sqrt{T}}  \sum_{t=1}^{\Tr} \DDelta_t^1$, i.e.,  $\lim\limits_{T \to \infty} \EB  \| \Bphi_T^1 - \ZM\|_{\sup} = 0$ for some $\LM(\ZM) = \LM(\VQ^{1/2} \B_D)$. 
The FCLT of $\Bphi_T^2(r)=\frac{1}{\sqrt{T}}  \sum_{t=1}^{\Tr} \DDelta_t^2$ can be validated in an almost identical way.
We start by rewriting~\eqref{eq:Delta1} as
\begin{equation}
\label{eq:Delta1-new}
\DDelta_t^1
= \A_t \DDelta_{t-1}^1+ \eta_t\left( \Z_t + \gamma \D_{t-1}^1\right),
\end{equation}
where $\A_t = \I - \eta_t (\I - \gamma \PP^{\pi^*})$, $\Z_t = (\rr_t-\rr) + \gamma( \PP_t - \PP) \V^*$, and 
\begin{equation}
\label{eq:D1}
\D_{t-1}^1 = ( \PP_t - \PP) (\V_{t-1} - \V^*) +( \PP^{\pi_{t-1}} -\PP^{\pi^*}) \DDelta_{t-1}.
\end{equation}
We comment that $\{\Z_t\}_{t \ge 0}$ collects the i.i.d.\ noise inherent in the empirical Bellman operator and $\{\D_{t-1}^1\}_{t \ge 1}$ captures the closeness between the current $Q$-function estimator $\Q_{t-1}$ and the optimal $\Q^*$.
Recurring~\eqref{eq:Delta1-new} gives
\begin{equation*}
\DDelta_t^1
= \prod_{j=1}^t \A_{j} \DDelta_{0}+  \sum_{j=1}^t \prod_{i=j+1}^t \A_{i} \eta_j \left( \Z_j + \gamma\D_{j-1}^1\right).
\end{equation*}
Here we use the convention that $\prod_{i=t+1}^t \A_{i} = \I$ for any $t \ge 0$.
For any $r \in [0, 1]$, summing the last equality over $t=1, \cdots, \Tr$ and scaling it properly, we have
\begin{align}
\label{eq:Delta1-begin-fclt}
\Bphi_T^1(r) &= \frac{1}{\sqrt{T}}  \sum_{t=1}^{\Tr} \DDelta_t^1 \nonumber  =\frac{1}{\sqrt{T}}  \sum_{t=1}^{\Tr} \prod_{j=1}^t \A_{j} \DDelta_{0}
+\frac{1}{\sqrt{T}}  \sum_{t=1}^{\Tr}\sum_{j=1}^t \prod_{i=j+1}^t \A_{i} \eta_j \left(\Z_j+ \gamma\D_{j-1}^1\right) \nonumber \\
&=\frac{1}{\sqrt{T}}  \sum_{t=1}^{\Tr} \prod_{j=1}^t \A_{j} \DDelta_{0}
+\frac{1}{\sqrt{T}}  \sum_{j=1}^{\Tr}\sum_{t=j}^{\Tr} \prod_{i=j+1}^t \A_{i} \eta_j \left( \Z_j+ \gamma\D_{j-1}^1\right) \nonumber \\
&=\frac{1}{\eta_0\sqrt{T}}  (\A_0^{\Tr}-\eta_0 \I) \DDelta_{0}
+\frac{1}{\sqrt{T}}  \sum_{j=1}^{\Tr}\A_j^{\Tr}\left(\Z_j+\gamma \D_{j-1}^1\right),
\end{align}
where the last line uses the following notation:
\begin{equation}
\label{eq:A_jT}
\A_{j}^T =\eta_j \sum_{t=j}^T \prod_{i=j+1}^t \A_i
\ \text{for any} \
T \ge j  \ge 0.
\end{equation}
Define $\G = \I - \gamma \PP^{\pi^*}$ with $\gamma \in [0 ,1)$, then $\A_i = \I - \eta_i \G$.
Typically speaking, $\A_j^T$ approximates $\G$ uniformly well (see Lemma~\ref{lem:G-poly}).
By the observation, we further expand~\eqref{eq:Delta1-begin-fclt} and decompose $\Bphi_T^1(r)$ into six terms $\{ \Bpsi_i \}_{i=0}^5$ which will be analyzed respectively in the following:
\begin{align}
\label{eq:Delta1-decom-fclt}
\Bphi_T^1(r)
&=\frac{1}{\eta_0\sqrt{T}}  (\A_0^{\Tr}-\eta_0 \I) \DDelta_{0}
+\frac{1}{\sqrt{T}}  \sum_{j=1}^{[Tr]}\A_j^{\Tr}\left(\Z_j+\gamma \D_{j-1}^1\right) \nonumber \\
&=\frac{1}{\eta_0\sqrt{T}}   (\A_0^{\Tr}-\eta_0 \I)  \DDelta_{0}
+\frac{1}{\sqrt{T}}  \sum_{j=1}^{\Tr}\G^{-1} \Z_j  + \frac{1}{\sqrt{T}}  \sum_{j=1}^{\Tr}  (\A_j^{\Tr} - \G^{-1})\Z_j\nonumber \\
&\quad \quad 
+  \frac{\gamma}{\sqrt{T}}  \sum_{j=1}^{\Tr}\A_j^{\Tr}( \PP_j - \PP) (\V_{j-1} - \V^*)
+ \frac{\gamma}{\sqrt{T}}  \sum_{j=1}^{\Tr}\A_j^{\Tr}( \PP^{\pi_{j-1}} -\PP^{\pi^*}) \DDelta_{j-1} \nonumber
\\
&=\frac{1}{\eta_0\sqrt{T}}   (\A_0^{\Tr}-\eta_0 \I)  \DDelta_{0}
+\frac{1}{\sqrt{T}}  \sum_{j=1}^{\Tr}\G^{-1}    \Z_j  
+ \frac{1}{\sqrt{T}}  \sum_{j=1}^{\Tr}  (\A_j^{T} - \G^{-1})  \Z_j   \nonumber \\
& \quad \quad
 + \frac{\gamma}{\sqrt{T}}  \sum_{j=1}^{\Tr} \A_j^{T}  ( \PP_j - \PP) (\V_{j-1} - \V^*)  +  \frac{1}{\sqrt{T}}  \sum_{j=1}^{\Tr}(\A_j^{\Tr} - \A_j^T)\left[  \Z_j +  \gamma ( \PP_j - \PP) (\V_{j-1} - \V^*) \right] \nonumber\\
 & \quad \quad 
+ \frac{\gamma}{\sqrt{T}}  \sum_{j=1}^{\Tr}\A_j^{\Tr}( \PP^{\pi_{j-1}} -\PP^{\pi^*}) \DDelta_{j-1} \nonumber
\\
&:= \Bpsi_0(r)  + \Bpsi_1(r) + \Bpsi_2(r) + \Bpsi_3(r) +\Bpsi_4(r) +\Bpsi_5(r).
\end{align} 
Readers should keep in mind that all $\Bpsi_{i}$'s depend on $T$, a dependence which we omit for simplicity.
In the following, we will show~\eqref{eq:Zr} is true by setting $\ZM = \Bpsi_1$.
In order to establish~\eqref{eq:sandwich-o1}, we will show that $\EB\|\Bpsi_i\|_{\sup} = o(1)$ for $i=0, 2,3,4,5$.
In this way, based on~\eqref{eq:Delta1-decom-fclt}, we have
\[
\EB\| \Bphi_T^1-  \Bpsi_1\|_{\sup}
\le \sum_{i=0,2,3,4,5}\EB \| \Bpsi_i\|_{\sup} = o(1)
\ \text{as} \ T \to \infty,
\]
and validate~\eqref{eq:sandwich-o1}.
To that end, we first study the properties of  $\{ \A_{j}^T \}_{0\le j \le T}$ since it appears in many $\Bpsi_{i}$'s.

% \frac{\gamma}{\sqrt{T}}  \sum_{j=1}^{\Tr} ( - \A_j^{T})( \PP_j - \PP) (\V_{j-1} - \V^*)\nonumber \\

\subsubsection{ Properties of \texorpdfstring{$\{\boldsymbol{\A}_j^T\}_{0\le j \le T}$}{AjT} }

First, prior work~\citep{polyak1992acceleration} considers a general step size $\{ \eta_t \}_{t \ge 0}$ satisfying Assumption~\ref{asmp:lr} and establishes the following lemma.

\begin{lem}[Lemma 1 in~\citep{polyak1992acceleration}]
	\label{lem:general-step-size}
	For  $\{\eta_t\}_{t \ge 0}$ satisfying Assumption~\ref{asmp:lr}, 
	\begin{itemize}
		\item Uniform boundedness: $\|\A_j^T\|_{\infty} \le C_0$ uniformly for all $T \ge j \ge 0$ for some constant $C_0 \ge 1$; 
		\item Uniform approximation: $\lim_{T \to \infty} \frac{1}{T}\sum_{j=1}^T\|\A_j^T - \G^{-1}\|_{2} = 0$.
	\end{itemize}
\end{lem}

Lemma~\ref{lem:general-step-size} shows that when the step size $\eta_t$ decreases at a slow rate, $\A_{j}^T$ is uniformly bouned (that is $\sup_{T \ge j \ge 1}\|\A_j^T\|_{\infty} < \infty$) and is a good surrogate of $\G^{-1} := (\I-\gamma \PP^{\pi^*})^{-1}$ in the asymptotic sense: $\lim_{T \to \infty} \frac{1}{T}\sum_{j=1}^T\|\A_j^T - \G^{-1}\|_2 = 0$.\footnote{The original Lemma 1 in~\citep{polyak1992acceleration} uses the $\ell_2$-norm and spectral norm. Due to the equivalence between these norms, we formulate our Lemma~\ref{lem:general-step-size}.}
It is sufficient to derive our asymptotic result. 
However, on purpose of non-asymptotic analysis, we should provide a non-asymptotic counterpart capturing the specific decaying rate in the $\ell_{\infty}$-norm.
Therefore, we consider two specific step sizes, namely (S1) the linear rescaled step size and (S2) polynomial step size.
Define $\eeta_t=(1-\gamma)\eta_t$ as the rescaled step size for simplicity, we have
\begin{enumerate}
	\item[(S1)] linear rescaled step size that uses $\eta_t =  \frac{1}{1+ (1-\gamma)t}$ (equivalently $\eeta_t = \frac{1-\gamma}{1+ (1-\gamma)t}$);
	\item[(S2)] polynomial step size that uses $\eta_t = t^{-\alpha}$ with $\alpha \in (0, 1)$ for $t\ge 1$ and $\eta_0 = 1$.
\end{enumerate}

The first is uniform boundedness whose proof is provided in Appendix~\ref{proof:Ajt-bound}.
\begin{lem}[Uniform boundedness]
	\label{lem:AjT-bound}
	There exists some $c>0$ such that
	\[
	\|\A_{j}^T\|_{\infty} \le C_0 := 
	\left\{ \begin{array}{ll}
		\frac{\ln(1+(1-\gamma)T)}{1-\gamma}& \mathrm{(S1)}\\ \frac{c 2^{\frac{1}{1-\alpha}} }{\sqrt{1-\alpha}} \frac{1}{(1-\gamma)^{\frac{1}{1-\alpha}}} & \mathrm{(S2)}
	\end{array} \right.
\ \text{for any} \ T \ge j \ge 1.
	\]
\end{lem}

The second is the uniform approximation.
The proof is deferred in Appendix~\ref{proof:G-poly}.
We observe that as $T$ grows, $\frac{1}{T}\sum_{j=1}^T\|\A_j^T - \G^{-1} \I\|_{\infty}^2$ vanishes under (S2), but is only guaranteed to be bounded for (S1).
This is not contradictory with Lemma~\ref{lem:general-step-size} since (S1) doesn't satisfy Assumption~\ref{asmp:lr}.

\begin{lem}[Uniform approximation]
	\label{lem:G-poly}
	There exists some constant $c>0$ such that
	\[
	\sqrt{\frac{1}{T}\sum_{j=1}^T\|\A_j^T - \G^{-1}\|_{\infty}^2} \le
	\left\{ \begin{array}{ll}
	\frac{5}{1-\gamma}& \mathrm{(S1)}\\ \frac{c\alpha2^{\frac{1}{1-\alpha}}}{\sqrt{T}} \left[ \frac{1}{(1-\alpha)^\frac{3}{2}} \frac{1}{(1-\gamma)^{1+\frac{1}{1-\alpha}}} 
	+ 
	\frac{1}{(1-\gamma)^2} \sqrt{\sum_{j=1}^T \frac{1}{j^{2(1-\alpha)}}} \right]
	+\frac{1}{(1-\gamma)} \sqrt{\frac{1}{T \eeta_{T}}}. & \mathrm{(S2)}
	\end{array} \right.
	\]
\end{lem}

\subsubsection{Establishing the Functional CLT}
\label{proof:FCLT-Delta1}
\paragraph{Uniform negligibility of $\Bpsi_0$.}
It is clear that $\Bpsi_0$ is a deterministic function.
Using the uniform boundedness of $\A_j^T (T \ge j \ge 0)$ in Lemma~\ref{lem:general-step-size}, we have
\begin{align*}
\|\Bpsi_0\|_{\sup}&=
\sup_{r \in [0, 1]}\|\Bpsi_0(r)\|_{\infty}
= \frac{1}{\eta_0\sqrt{T}} \sup_{r \in [0, 1]} \| (\A_0^{\Tr}-\eta_0 \I) \DDelta_{0}\|_{\infty} \\
&\le \frac{1}{\eta_0\sqrt{T}} \left(\sup_{0 \le t \le T}\| \A_0^{t}\|_{\infty} + \eta_0\right)\|\DDelta_{0}\|_{\infty} \\
&\le \frac{1}{\eta_0\sqrt{T}} \frac{2C_0}{1-\gamma} \to 0 
\ \text{as} \ T \to \infty,
\end{align*}
where we use $\eta_0 \le 1 \le C_0$ and $\|\DDelta_0\|_{\infty} \le \frac{1}{1-\gamma}$.

\paragraph{Partial-sum asymptotic behavior of $\Bpsi_1$.}
Recall that $\Z_j = (\rr_j - \rr)  + \gamma ( \PP_j - \PP) \V^*$ is the noise inherent in the
empirical Bellman operator at iteration $j$.
Since at each iteration the simulator generates rewards $\rr_j$ and produces the empirical transition $\PP_j$ in an i.i.d. fashion, $\TM_1(r)$ is the scaled partial sum of $\Tr$ independent copies of the random vector $\Z_j$ which has zero mean and finite variance denoted by $\Var(\Z_j) =  \Var(\rr_j + \gamma\PP_j \V^*) = \EB \Z_j \Z_j^\top$.
Additionally, it is clear that $\|\Z_j\|_{\infty} \le (1-\gamma)^{-1}$  is uniformly bounded and thus its moments of any order is uniformly bounded.
By Theorem 4.2 in~\citep{hall2014martingale} (or Theorem 2.2 in~\citep{jirak2017weak}), we establish the following FCLT for the partial sums of independent random vectors.

\begin{lem}
	\label{lem:T1}
	For any $r \in [0, 1]$,
	\[
	\Bpsi_1(\cdot)
	= \frac{1}{\sqrt{T}}  \sum_{j=1}^{\lfloor T\cdot\rfloor} \G^{-1}\Z_j
	\overset{w}{\to} \VQ^{1/2}\B_D(\cdot),
	\]
	where $\B_D$ is the $D$-dimensional standard Brownian motion and  the variance matrix $\VQ$ is
	\[
	\VQ  = 
	\G^{-1}\Var(\Z_j)\G^{-\top} = 
	(\I-\gamma \PP^{\pi^*})^{-1}\Var(\Z_j)(\I-\gamma \PP^{\pi^*})^{-\top}.
	\]
\end{lem}

%The celebrated Donsker's invariance principle helps capture the partial-sum asymptotic behavior of $\TM_1(r) = \frac{1}{\sqrt{T}}  \sum_{j=1}^{\Tr}\G^{-1} \Z_j$~\citep{hall2014martingale}.

\paragraph{Uniform negligibility of $\Bpsi_2$.}
Recall that $\Bpsi_2(r) = \frac{1}{\sqrt{T}}  \sum_{j=1}^{\Tr}  (\A_j^{T} - \G^{-1})\Z_j$.
If we define $\X_t =\frac{1}{\sqrt{T}}  \sum_{j=1}^{t}  (\A_j^{T} - \G^{-1})\Z_j$, then $\Bpsi_2(r) = \X_{\Tr}$.
Let $\FM_t = \sigma(\{ \rr_j, \PP_j \}_{0 \le j \le t})$ be the $\sigma$-field generated by all randomness before and including iteration $t$.
Then $\{ \X_t, \FM_t \}$ is a martingale since $\EB[\X_t|\FM_{t-1}] = \X_{t-1}$.
As a result $\{ \|\X_t\|_2, \FM_t \}$ is a submartingale since by conditional Jensen's inequality, we have $\EB[\|\X_t\|_2|\FM_{t-1}] \ge \|\EB[\X_t|\FM_{t-1}] \|_2 = \|\X_{t-1}\|_2$.
By Doob’s maximum inequality for submartingales (which we use to derive the following $(*)$ inequality), 
\begin{align*}
	\EB\sup_{r \in [0, 1]} \|\Bpsi_{2}(r)\|_2^2
	&=\EB \sup_{0 \le t \le T} \|\X_t\|_2^2 
	\overset{(*)}{\le} 4 \EB\|\X_T\|_2^2\\
	&= 4 \EB\|\TM_{2}(1)\|_2^2
	=4 \EB  \left\|\frac{1}{\sqrt{T}}  \sum_{j=1}^{T}  (\A_j^{T} - \G^{-1})\Z_j\right\|_2^2\\
	&= \frac{4}{T} \sum_{j=1}^T \EB \|(\A_j^{T} - \G^{-1})\Z_j\|_2^2
	\le \frac{4}{T} \sum_{j=1}^T \|\A_j^{T} - \G^{-1}\|_2^2\EB\|\Z_j\|_2^2\\
	&\le c_1 \cdot \frac{1}{T} \sum_{j=1}^T \|\A_j^{T} - \G^{-1}\|_2.
\end{align*}
Here, we change to the $\ell_2$-norm since it will facilitate the analysis. 
The last inequality follows by using a finite $c_1$ satisfying $\EB\|\Z_j\|_2^2 \sup_{T \ge j \ge 1}\|\A_j^{T} - \G^{-1}\|_2 \le c_1$.
Indeed, we can set $c_1 = (\frac{1}{1-\gamma}+\sup_{T\ge j} \|\A_j^{T}\|_2)\tr(\VQ)$ thanks to Lemma~\ref{lem:general-step-size}.
In addition, Lemma~\ref{lem:general-step-size} implies $ \frac{1}{T} \sum_{j=1}^T \|\A_j^{T} - \G^{-1}\|_2 \to 0$ as $T$ goes to infinity.
As a result, $\EB\|\Bpsi_{2}\|_{\sup}= \EB\sup_{r \in [0, 1]} \|\Bpsi_{2}(r)\|_{\infty} \le \EB\sup_{r \in [0, 1]} \|\Bpsi_{2}(r)\|_2\le \sqrt{\EB\sup_{r \in [0, 1]} \|\Bpsi_{2}(r)\|_2^2}= o(1)$.

\paragraph{Uniform negligibility of $\Bpsi_3$.}
Recall that $\Bpsi_3(r)= \frac{\gamma}{\sqrt{T}}  \sum_{j=1}^{\Tr} \A_j^{T}  ( \PP_j - \PP) (\V_{j-1} - \V^*)$.
By a similar argument in the analysis of $\Bpsi_2$, we have $ \EB\sup_{r \in [0, 1]} \|\Bpsi_{3}(r)\|_2^2 \le 4 \EB\|\Bpsi_{3}(1)\|_2^2$ by Doob’s maximum inequality.
Therefore,
\begin{align*}
	\EB\sup_{r \in [0, 1]} \|\Bpsi_{3}(r)\|_2^2
	&\le 4 \EB\|\Bpsi_{3}(1)\|_2^2 \overset{(a)}{=}\frac{4}{T} \sum_{j=1}^T \EB \|\A_j^{T}  ( \PP_j - \PP) (\V_{j-1} - \V^*)\|_2^2\\
	&\le \frac{4}{T} \sum_{j=1}^T \|\A_j^{T}\|_2^2  \EB\| \PP_j - \PP\|_2^2 \|\V_{j-1} - \V^*\|_2^2\\
	&\overset{(b)}{\le} c_2 \cdot \frac{1}{T} \sum_{j=1}^T \EB\|\V_{j-1} - \V^*\|_2^2,
\end{align*}
where $(a)$ follows since all cross terms have zero mean due to $\EB[( \PP_j - \PP) (\V_{j-1} - \V^*)|\FM_{j-1}] = 0$, and $(b)$ follows by setting $c_2=16D(\sup_{T\ge j}\|\A_j^{T}\|_2)^2$ because of the uniform boundedness of $\|\A_j^{T}\|_{\infty}$ from Lemma~\ref{lem:general-step-size} and $\| \PP_j - \PP\|_2^2 \le D \| \PP_j - \PP\|_{\infty}^2 = 4D$.
By Theorem~\ref{thm:Linfty-pw2}, we know $\frac{1}{T} \sum_{j=1}^T \EB\|\V_{j-1} - \V^*\|_2^2 \to 0$ under the general step size when $T \to \infty$.
As a result, $\EB\|\Bpsi_{3}(r)\|_{\sup}= \EB\sup_{r \in [0, 1]} \|\Bpsi_{3}(r)\|_{\infty} \le \EB\sup_{r \in [0, 1]} \|\Bpsi_{3}(r)\|_2\le \sqrt{\EB\sup_{r \in [0, 1]} \|\Bpsi_{3}(r)\|_2^2}= o(1)$.

\paragraph{Uniform negligibility of $\Bpsi_4$.}
Recall that $\Bpsi_{4}(r) =  \frac{1}{\sqrt{T}}  \sum_{j=1}^{\Tr}(\A_j^{\Tr} - \A_j^T)\varepsi_j$ where $\varepsi_j= \Z_j +  \gamma ( \PP_j - \PP) (\V_{j-1} - \V^*) $.
It is clear that we have $\sup_{j \ge 0}\EB\|\Q_t-\Q^*\|^4 < \infty$ as a result of $\sup_{j \ge 0}\EB \EB \|\Q_j\_{\infty}^4< \infty$ in Lemma~\ref{lem:moment}.
Notice that the coefficient $\A_j^{\Tr} - \A_j^T$ changes as $r$ varies. 
The analysis of $\Bpsi_{4}$ should be more careful and subtle.

\begin{lem}[Moment bounds]
	\label{lem:moment}
	Under Assumption~\ref{asmp:reward}, it follows that
	\[
	\sup_{t \ge 0} \EB \|\Q_t\|_{\infty}^4 < \infty.
	\]
\end{lem}
\begin{proof}[Proof of Lemma~\ref{lem:moment}]
By Lemma~\ref{lem:infty}, $\|\De_t\|_{\infty} \le a_t + b_t + \|\N_t\|_{\infty}$.
It implies that $\EB\|\De_t\|_{\infty}^4 \le 3^3 \EB\left(a_t^4 + b_t^4 + \|\N_t\|_{\infty}^4 \right)$.
Notice that
	\begin{gather*}
		a_t = (1-\eta_t(1-\gamma))a_{t-1}\\
		b_t = (1-\eta_t(1-\gamma))b_{t-1} + \eta_t \gamma \|\N_{t-1}\|_{\infty}\\
		\N_t = (1-\eta_t) \N_{t-1} + \eta_t \Z_t.
	\end{gather*}
First, it is easy to find that $\sup_{t \ge 0} a_t < \infty$ since it is deterministic and decays exponentially fast. 
Second, we have $\sup_{t \ge 0} \|\N_t\|_{\infty}^4 < \infty$.
This is because we have $\EB \|\N_t\|_{\infty}^4 \le (1-\eta_t) \EB \|\N_{t-1}\|_{\infty}^4 + \eta_t \EB\|\Z_t\|_{\infty}^4$ from Jensen's inequality.
It is easy to show $\sup_{t \ge 0} \|\N_t\|_{\infty}^4 < \sup_{t \ge 0} \EB\|\Z_t\|_{\infty}^4 < \infty$ by this inequality and induction.
Finally, iterating the expression of $b_t$, we have $b_T = \gamma\sum_{t=1}^T \prod_{j=t+1}^T(1-(1-\gamma)\eta_j)\eta_t \|\N_{t-1}\|_{\infty} = \frac{\gamma}{1-\gamma} \sum_{t=1}^T \eeta_{(t, T)} \|\N_{t-1}\|_{\infty}$ with $\eeta_{(t, T)}$ a probability defined on $[T]$ in~\eqref{eq:eeta-tT}.
The last equation implies $b_T$ is a probability weighted sum of $\N_t (t \in [T])$.
Hence, by Jensen's inequality, we know  $\sup_{t \ge 0} \EB b_t^4 < \sup_{t \ge 0} \EB\|\Z_t\|_{\infty}^4 < \infty$.
\end{proof}

Recall $\FM_t = \sigma(\{ \rr_j, \PP_j \}_{0 \le j \le t})$ is the $\sigma$-field generated by all randomness before and including iteration $t$.
$\{ \varepsi_t, \FM_t \}$ is a martingale difference since $\EB[\varepsi_t|\FM_{t-1}] = \0$.
Furthermore, $\varepsi_t$ has finite moments of any order since it is almost surely bounded $\|\varepsi_t\|_{\infty} = \OM((1-\gamma)^{-1})$.
On the other hand, by definition~\eqref{eq:A_jT}, it follows that for any $0 \le k \le T$,
\begin{align*}
\sum_{j=1}^{k}(\A_j^T- \A_j^{k} )\varepsi_j
&= \sum_{j=1}^{k} \sum_{t=k+1}^T \left(\prod_{i=j+1}^t \A_i\right)\eta_j \varepsi_j
= \sum_{t=k+1}^{T} \sum_{j=1}^k \left(\prod_{i=j+1}^t \A_i\right)\eta_j \varepsi_j\\
&=\sum_{t=k+1}^{T}  \left(\prod_{i=k+1}^t \A_i\right)\sum_{j=1}^k  \left(\prod_{i=j+1}^k \A_i\right)
\eta_j \varepsi_j\\
&=\frac{1}{\eta_{k+1}}\A_{k+1}^T\A_{k+1}\sum_{j=1}^k  \left(\prod_{i=j+1}^k \A_i\right)
\eta_j \varepsi_j
\end{align*}
On one hand, $\|\A_{k+1}^T\A_{k+1}\|_{2} \le c_3$ is uniformly bounded with $c_3= (\sup_{T\ge j}\|\A_j^{T}\|_2)(1+\|\G\|_2)$ for any $T \ge k+1$ from Lemma~\ref{lem:general-step-size}.
On the other hand, we define an auxiliary sequence $\{\Y_k \}_{k \ge 1}$ as following: $\Y_1 = \0$ and $\Y_{k+1} = \A_k \Y_k + \eta_k \varepsi_k$ for any $k \ge 1$.
One can check that $\Y_{k+1} = \sum_{j=1}^k  \left(\prod_{i=j+1}^k \A_i\right)
\eta_j \varepsi_j $ where we use the convention $ \prod_{i=k+1}^k\A_i = \I$ for any $k \ge 0$.
These results imply we can apply Lemma~\ref{lem:ignore-error}.
Putting these pieces together, we have that
\begin{align*}
\|\Bpsi_4\|_{\sup} 
&\le \sup_{r \in [0, 1]} \|\Bpsi_{4}(r)\|_2\le  c_3  \sup_{0 \le k \le T}  \left\| \frac{1}{\sqrt{T}\eta_{k+1}}\sum_{j=1}^k  \left(\prod_{i=j+1}^k \A_i\right)
\eta_j \varepsi_j \right\|_2\\
&= c_3 \sup_{0 \le k \le T} \frac{1}{\sqrt{T}} \frac{\|\Y_{k+1}\|_2}{\eta_{k+1}} \overset{(*)}{=} o_{\PB}(1),
\end{align*}
where $(*)$ follows from Lemma~\ref{lem:ignore-error}.

\paragraph{Uniform negligibility of $\Bpsi_5$.}
In the following, we will prove $\|\Bpsi_5\|_{\sup}= o_{\PB}(1)$ by showing $\EB \|\Bpsi_5\|_{\sup} = o(1)$.
It is worth mentioning that $\Bpsi_5$ arises purely due to the non-stationary nature of Q-learning.
If we consider a stationary update process, e.g., policy evaluation~\citep{mou2020linear,mou2020optimal,khamaru2021instance}, $\pi_t$ would remain the same all the time and $\Bpsi_5$ would disappear in the case.
Notice that $\Bpsi_5(r) = \frac{\gamma}{\sqrt{T}}  \sum_{j=1}^{\Tr}\A_j^{\Tr}( \PP^{\pi_{j-1}} -\PP^{\pi^*}) \DDelta_{j-1}$ is a sum of correlated random variables (which are even not mean-zero).
We need a high-order residual condition Assumption~\ref{asmp:gap} to bound $\EB \|\Bpsi_5\|_{\sup} $.
With such a Lipschitz condition, Lemma~\ref{lem:T4} shows $\EB \|\Bpsi_5\|_{\sup} $ is dominated by $\frac{1}{\sqrt{T}}   \sum_{j=1}^T\EB
\left\| \DDelta_{j-1}\right\|_{\infty}^2$, which is $o(1)$ for the general step size as suggested by Theorem~\ref{thm:general-Linfty-pw2}.
The proof of Lemma~\ref{lem:T4} is in Appendix~\ref{proof:T4}.

\begin{lem}
	\label{lem:T4}
	It follows that 
	\[
	\EB\|\Bpsi_5\|_{\sup} = 
	\EB \sup_{r \in [0, 1]}\|\Bpsi_5(r)\|_{\infty} \le \gamma LC_0\cdot \frac{1}{\sqrt{T}}   \sum_{j=1}^T\EB
	\left\| \DDelta_{j-1}\right\|_{\infty}^2.
	\]
\end{lem}

\paragraph{Putting the pieces together.}
From~\eqref{eq:Delta1-begin-fclt}, $\Bphi_T^1  = \sum_{i=0}^5 \Bpsi_i$.
We have shown $\Bpsi_1 \overset{w}{\to} \VQ^{1/2} \B_D$ in the sense of $(\BDD, d_0)$ and $\|\Bpsi_i\|_{\sup} = o_\PB(1)$ for $i \neq 1$.
Using $\|\Bphi_T^1-\Bpsi_1\|_{\sup} \le \sum_{i \neq 1} \|\Bpsi_i\|_{\sup}$, we know that $\|\Bphi_T^1-\Bpsi_1 \|_{\sup} = o_{\PB}(1)$.
Proposition~\ref{prop:slutsky} implies $\Bphi_T^1 \overset{w}{\to} \VQ^{1/2} \B_D$.
We then establish the FCLT for $\Bphi_T^1(r)$.

\subsection{Functional CLT for \texorpdfstring{$ \Bphi_T^2$}{Phi-T-2}}
\label{proof:fclt-Delta2}
We can repeat the above analysis for $\Bphi_T^2$.
We rewrite~\eqref{eq:Delta2} as
\begin{equation}
\label{eq:Delta2-new}
\DDelta_t^2
= \A_t \DDelta_{t-1}^2+ \eta_t\left( \Z_t + \gamma \D_{t-1}^2\right),
\end{equation}
where $\A_t = \I - \eta_t (\I - \gamma \PP^{\pi^*})$ and $\Z_t = (\rr_t-\rr) + \gamma( \PP_t - \PP) \V^*$ are the same as those defined in~\eqref{eq:Delta1-new} except that $\D_{t-1}^1$ (defined in~\eqref{eq:D1}) is replaced by 
\begin{equation}
\label{eq:D2}
\D_{t-1}^2 = ( \PP_t - \PP) (\V_{t-1} - \V^*).
\end{equation}
Since $\D_{t-1}^2$ is much simpler than $\D_{t-1}^1$, the analysis for $\Bphi_T^2(r)$ should be easier than $\Bphi_T^1(r)$.
Using the notation $\A_j^T$ (see\eqref{eq:A_jT}), we decompose $\Bphi_T^2(r)$ into five terms:
\begin{align}
\label{eq:Delta2-begin-fclt}
\Bphi_T^2(r) &= \frac{1}{\sqrt{T}}  \sum_{t=1}^{\Tr} \DDelta_t^2 
=\frac{1}{\eta_0\sqrt{T}}  (\A_0^{\Tr}-\eta_0 \I) \DDelta_{0}
+\frac{1}{\sqrt{T}}  \sum_{j=1}^{[Tr]}\A_j^{\Tr}\left(\Z_j+\gamma \D_{j-1}^2\right)  \nonumber\\
&=\frac{1}{\eta_0\sqrt{T}}   (\A_0^{\Tr}-\eta_0 \I)  \DDelta_{0}
+\frac{1}{\sqrt{T}}  \sum_{j=1}^{\Tr}\G^{-1}    \Z_j  
+ \frac{1}{\sqrt{T}}  \sum_{j=1}^{\Tr}  (\A_j^{T} - \G^{-1})  \Z_j   \nonumber \\
& \quad \quad
+ \frac{\gamma}{\sqrt{T}}  \sum_{j=1}^{\Tr} \A_j^{T}  ( \PP_j - \PP) (\V_{j-1} - \V^*)  \nonumber \\
& \quad \quad+  \frac{1}{\sqrt{T}}  \sum_{j=1}^{\Tr}(\A_j^{\Tr} - \A_j^T)\left[  \Z_j +  \gamma ( \PP_j - \PP) (\V_{j-1} - \V^*) \right] \nonumber\\
&:= \Bpsi_0(r)  + \Bpsi_1(r) + \Bpsi_2(r) + \Bpsi_3(r) +\Bpsi_4(r).
\end{align}
Here $\{\Bpsi_i\}_{i=0}^4$ are exactly the same as those in~\eqref{eq:Delta1-decom-fclt}.
Our previous analysis provides us a low-hanging fruit result:
$\Bpsi_1 \overset{w}{\to} \VQ^{1/2} \B_D$ in the sense of $(\BDD,d_0)$ and $\|\Bpsi_i\|_{\sup} = o_\PB(1)$ for $i \neq 1$.
Then we know that $\|\Bphi_T^2-\TM_1 \|_{\sup} = o(1)$ and   $\Bphi_T^2\overset{w}{\to} \VQ^{1/2} \B_D$ due to Proposition~\ref{prop:slutsky}.
We thus establish the FCLT for $\Bphi_T^2$.

\subsection{Proofs of Lemmas}
\subsubsection{Proof of Lemma~\ref{lem:sandwitch}}
\label{proof:sandwitch}
\begin{proof}[Proof of Lemma~\ref{lem:sandwitch}]
	We use mathematical induction to prove the statement.
	When $t=0$, the inequality~\eqref{eq:Delta-relation} holds by initialization.
	Assume~\eqref{eq:Delta-relation} holds at $t-1$, i.e., $\DDelta_{t-1}^2 \le  \DDelta_{t-1} \le \DDelta_{t-1}^1.$
	Let us analyze the case of $t$.
	By the Q-learning update rule, it follows that
	\begin{align}
	\label{eq:Delta-iteration}
	\DDelta_t 
	&=(1-\eta_t) \DDelta_{t-1}+ \eta_t\left[(\rr_t-\rr)+\gamma( \PP_t \V_{t-1}-\PP \V^*)\right]  \nonumber \\
	&\overset{(a)}{=}(1-\eta_t) \DDelta_{t-1}+ \eta_t\left[\W_t +\gamma( \PP \V_{t-1}-\PP \V^*)\right] \nonumber \\
	&
	\overset{(b)}{=} (1-\eta_t) \DDelta_{t-1}+ \eta_t\left[\W_t+\gamma(\PP^{\pi_{t-1}} \Q_{t-1}-\PP^{\pi^*} \Q^*)\right] \nonumber \\
	& \overset{(c)}{=} \A_t \DDelta_{t-1}+ \eta_t\left[\W_t +\gamma( \PP^{\pi_{t-1}} -\PP^{\pi^*}) \Q_{t-1}\right],
	\end{align}
	where $(a)$ uses $\W_t =  (\rr_t-\rr) + \gamma( \PP_t - \PP) \V_{t-1}$; $(b)$ uses $\PP \V_{t-1} = \PP^{\pi_{t-1}} \Q_{t-1}$ and $\PP \V^* = \PP^{\pi^*} \Q^*$, and $(c)$ follows by arrangement and the shorthand $\A_t = \I - \eta_t (\I - \gamma \PP^{\pi^*})$.
	Since all the entries of $\A_t =  \I - \eta_t (\I - \gamma \PP^{\pi^*}) $ are non-negative (which results from the assumption $\sup_t \eta_t \le 1$), then $\A_t\DDelta_{t-1}^2 \le  \A_t\DDelta_{t-1} \le \A_t\DDelta_{t-1}^1$.
	
	For one hand, based on~\eqref{eq:Delta-iteration}, we have
	\begin{align*}
	\DDelta_t^2 
	&= \A_t  \DDelta_{t-1}^2+ \eta_t \W_t \le \A_t  \DDelta_{t-1}+ \eta_t\W_t\\
	&\le \A_t \DDelta_{t-1}+ \eta_t\left[ \W_t +\gamma( \PP^{\pi_{t-1}} -\PP^{\pi^*}) \Q_{t-1}\right] = \DDelta_t,
	\end{align*}
	where the last inequality uses $ \PP^{\pi_{t-1}} \Q_{t-1} \ge \PP^{\pi^*}\Q_{t-1}$ which results from the fact $\pi_{t-1}$ is the greedy policy with respect to $\Q_{t-1}$.
	For the other hand, it follows that
	\begin{align*}
	\DDelta_t
	&= \A_t  \DDelta_{t-1}+ \eta_t\left[\W_t+\gamma( \PP^{\pi_{t-1}} -\PP^{\pi^*}) \Q_{t-1}\right]\\
	&\le \A_t  \DDelta_{t-1}^1+ \eta_t\left[\W_t +\gamma( \PP^{\pi_{t-1}} -\PP^{\pi^*}) \Q_{t-1}\right]\\
	&= \A_t  \DDelta_{t-1}^1+ \eta_t\left[ \W_t+\gamma( \PP^{\pi_{t-1}} -\PP^{\pi^*}) \DDelta_{t-1}+\gamma( \PP^{\pi_{t-1}} -\PP^{\pi^*})  \Q^* \right]\\
	&\le\A_t  \DDelta_{t-1}^1+ \eta_t\left[\W_t +\gamma( \PP^{\pi_{t-1}} -\PP^{\pi^*}) \DDelta_{t-1} \right] =\DDelta_t^1,
	\end{align*}
	where the last inequality uses $ \PP^{\pi_{t-1}} \Q^* \le \PP^{\pi^*}\Q^*$ which results from the fact $\pi^*$ is the greedy policy with respect to $\Q^*$.
	Hence, we have proved $\DDelta_{t}^2 \le  \DDelta_{t} \le \DDelta_{t}^1$ holds at iteration $t$.
\end{proof}

\subsubsection{Proof of Lemma~\ref{lem:AjT-bound}}
\label{proof:Ajt-bound}
\begin{proof}[Proof of Lemma~\ref{lem:AjT-bound}]
	By the definition of~\eqref{eq:A_jT}, we have $\|\A_j^T\|_\infty
	\le\eta_j\sum_{t=j}^T\prod_{i=j+1}^t(1-\widetilde{\eta}_i)$.
	Plugging the specific form of $\{ \eta_t \}$, we have for (S1)
	\begin{align}
	\|\A_j^T\|_\infty
	&\le\eta_j\sum_{t=j}^T\prod_{i=j+1}^t\frac{1+(1-\gamma)(i-1)}{1+(1-\gamma)i}=\eta_j\sum_{t=j}^T\frac{1+(1-\gamma)j}{1+(1-\gamma)t} \nonumber \\
	&\le\frac{1}{1-\gamma}\ln\frac{1+(1-\gamma)T}{1+(1-\gamma)(j-1)} \le \frac{\ln(1+(1-\gamma)T)}{1-\gamma}
	\label{eq:S1-bound}
	\end{align}
	and for (S2)
	\begin{align*}
	\|\A_j^T\|_\infty
	&=\eta_j\sum_{t=j}^T\prod_{i=j+1}^t(1-(1-\gamma)i^{-\alpha})
	\le \eta_j\sum_{t=j}^T\exp\left(  - (1-\gamma)\sum_{i=j+1}^t i^{-\alpha} \right)\\
	&\overset{(a)}{\le} e\eta_j\sum_{t=j+1}^{T+1} \exp\left(  - \frac{1-\gamma}{1-\alpha}\left({t}^{1-\alpha} - {j}^{1-\alpha}\right) \right) \\
	&\le e\eta_j  \int_j^{\infty} \exp\left(  - \frac{1-\gamma}{1-\alpha}\left({t}^{1-\alpha} - {j}^{1-\alpha}\right) \right) dt\\
	&\overset{(b)}{\le}\frac{e\eta_j  }{1-\gamma}\int_{0}^{\infty}\left(  \frac{1-\alpha}{1-\gamma} y + j^{1-\alpha} \right)^{\frac{\alpha}{1-\alpha}} \exp\left(  - y \right) dy \\
	&\overset{(c)}{\le}\frac{e\eta_j }{1-\gamma}\max\left\{2^{\frac{\alpha}{1-\alpha}}, 1 \right\}\int_{0}^{\infty}
	\left[
	\left(  \frac{1-\alpha}{1-\gamma} y\right)^{\frac{\alpha}{1-\alpha}} +
	j^{\alpha}
	\right]
	\exp\left(  - y \right) dy \\
	&= \frac{e }{(1-\gamma)j^{\alpha}}\max\left\{2^{\frac{\alpha}{1-\alpha}}, 1 \right\}\left[  \left(  \frac{1-\alpha}{1-\gamma} \right)^{\frac{\alpha}{1-\alpha}} \Gamma\left( \frac{1}{1-\alpha} \right) + j^{\alpha} \right] \\
	&\overset{(d)}{\le}
	e\max\left\{2^{\frac{\alpha}{1-\alpha}}, 1 \right\}\left[   \frac{\sqrt{2\pi e}}{\sqrt{1-\alpha}(1-\gamma)^{\frac{1}{1-\alpha}}} + \frac{1}{1-\gamma} \right] \\
	&\le \frac{c 2^{\frac{1}{1-\alpha}} }{\sqrt{1-\alpha}} \frac{1}{(1-\gamma)^{\frac{1}{1-\alpha}}},
	\end{align*}
	where $(a)$ uses $\sum_{i=j}^t i^{-\alpha} \ge \frac{1}{1-\alpha}((t+1)^{1-\alpha}-j^{1-\alpha})$ and $\exp((1-\gamma)j^{-\alpha}) \le e$, $(b)$ uses the change of variable $y= \frac{1-\gamma}{1-\alpha}(t^{1-\alpha} - j^{1-\alpha})$, $(c)$ uses $(a+b)^p\le\max\{2^{p-1},1\}(a^p+b^p)$ for any $p >0$, and $(d)$ uses $(1-\alpha)^{\frac{\alpha}{1-\alpha}}\Gamma\left(\frac{1}{1-\alpha}\right) \le \frac{\sqrt{2\pi}e^{1/2}}{\sqrt{1-\alpha}}$ from~\eqref{eq:Gamma} and $\max\left\{2^{\frac{\alpha}{1-\alpha}}, 1 \right\} \le 2^{\frac{1}{1-\alpha}}$.
\end{proof}

\subsubsection{Proof of Lemma~\ref{lem:G-poly}}
\label{proof:G-poly}
\begin{proof}[Proof of Lemma~\ref{lem:G-poly}]
	For (S1), we have
	\begin{align*}
	\frac{1}{T}\sum_{j=1}^T\|\A_j^T-\G^{-1}\|_\infty^2
	&\le \frac{2}{T}\sum_{j=1}^T(\|\A_j^T\|_{\infty}^2+ \|\G^{-1}\|_\infty^2)\\
	&\le
	2+\frac{8}{(1-\gamma)^2 }\frac{1}{T}\sum_{j=1}^T\ln^2\frac{1+(1-\gamma)T}{1+(1-\gamma)(j-1)}\\
	&\le 2+\frac{8}{(1-\gamma)^2 }
	\left[\frac{\ln^2(1+(1-\gamma)T)}{T} +
	\frac{1}{T}\sum_{j=1}^{T-1}\ln^2\frac{T}{j} \right] \\
	&\overset{(a)}{\le} 2 + \frac{7}{1-\gamma} + \frac{16}{(1-\gamma)^2 }
	\le \frac{25}{(1-\gamma)^2 },
	%&\le1+\frac{2}{(1-\gamma)T}\int_{-1}^{T}\ln^2\frac{1+(1-\gamma)T}{1+(1-\gamma)x}dx\\
	%&\le1+\frac{2}{(1-\gamma)T}\cdot (T+1)\\
	%&=1+\frac{4}{1-\gamma}.
	\end{align*}
	where $(a)$ uses $\ln^2(1+x)/x \le \frac{7}{8}$ for all $x \ge 0$ and $\int_0^1 \ln^2 x dx = \Gamma(3) = 2 \Gamma(1)= 2$.
	
	For (S2), based on~\eqref{eq:A_jT} and $\G = \eta_j^{-1}(\I-(\I-\eta_j\G))$, we have
	\begin{align}
	\A_j^T-\G^{-1}
	&=(\A_j^T\G-\I)\G^{-1}=\sum_{t=j}^T\left(\prod_{i=j+1}^t(\I-\eta_i\G)-\prod_{i=j}^t(\I-\eta_i\G)\right)\G^{-1} -\G^{-1}  \nonumber \\
	&=\sum_{t=j+1}^T\left(\prod_{i=j+1}^t(\I-\eta_i\G)-\prod_{i=j}^{t-1}(\I-\eta_i\G)\right)\G^{-1}   - \prod_{t=j}^T(\I-\eta_t\G)\G^{-1}   \nonumber \\
	&=\sum_{t=j+1}^T\left(  \eta_j -\eta_t \right) \prod_{i=j+1}^{t-1}(\I-\eta_i\G) - \prod_{t=j}^T(\I-\eta_t\G)\G^{-1} \nonumber  \\
	&:=\M_{T,j}^{(1)}+\M_{T,j}^{(2)}. \label{eq:AG-I-split}
	\end{align}
	On the one hand,
	\begin{align*}
	\left\| \M_{T,j}^{(2)} \right\|_{\infty}
	\le \|\G^{-1} \|_{\infty}\prod_{t=j}^T\| \I-\eta_t\G \|_{\infty}
	\le \frac{\prod_{t=j}^T (1-\eeta_{t})}{1-\gamma} \le
	\frac{(1-\eeta_{T})^{T-j+1}}{1-\gamma}.
	\end{align*}
	On the other hand, 
	\begin{align*}
	\left\|\M_{T,j}^{(1)}\right\|_{\infty}
	&= \left\| \sum_{t=j+1}^T\left( \eta_t - \eta_j \right) \prod_{i=j+1}^{t-1}(\I-\eta_i\G) \right\|_{\infty}  \\
	&\le \sum_{t=j+1}^T | \eta_t - \eta_j | \exp\left(  - \sum_{i=j+1}^{t-1} \eeta_i \right)  \\
	&\le \sum_{t=j+1}^T \sum_{k=j}^{t-1} | \eta_{k+1} - \eta_k | \exp\left(  - \sum_{i=j+1}^{t-1} \eeta_i \right)   \\
	&\overset{(a)}{\le} \sum_{t=j+1}^T \sum_{k=j}^{t-1} \frac{\alpha}{k} \eta_k  \exp\left(  - \sum_{i=j+1}^{t-1} \eeta_i \right)   \\
	&\overset{(b)}{\le} \frac{e \alpha}{(1-\gamma)j}  \sum_{t=j+1}^T  \tim_{j, t-1}  \exp\left(  - \tim_{j, t-1} \right)  
	= \frac{e \alpha}{(1-\gamma)j}  \sum_{t=j}^{T-1}  \tim_{j, t}  \exp\left(  - \tim_{j, t} \right) \\
	&\overset{(c)}{\le} \frac{ec\alpha}{(1-\gamma)j}
	\left[\frac{2^{\frac{1}{1-\alpha}}}{(1-\alpha)^{\frac{3}{2}}} \frac{1}{(1-\gamma)^{\frac{1}{1-\alpha}}} + \frac{ 2^{\frac{1}{1-\alpha}}}{1-\gamma} (j-1)^{\alpha}\right],
	\end{align*}
	where $(a)$ uses the fact that for $\eta_t = t^{-\alpha}$, we have
	\begin{equation*}
	\frac{\eta_t-\eta_{t+1}}{\eta_t}=1-\left(1-\frac{1}{t+1}\right)^\alpha\le1-\exp(-\frac{\alpha}{t})\le\frac{\alpha}{t},
	\end{equation*}
	where we use $\ln(1+x)\ge x/(1+x)$ in the first inequality and $\ln(1+x)\le x$ in the second inequality.
	$(b)$ uses the notation $\tim_{j, t} := \sum_{i=j}^{t} \eeta_i$ and $\exp(\eeta_j) \le \exp(1) = e$.
	$(c)$ uses the following lemma.

	\begin{lem}
		\label{lem:m-integral}
		Let $\tim_{j, t} := \sum_{i=j}^{t} \eeta_i$ and recall $\eeta_i = (1-\gamma)i^{-\alpha}$.
		Then $T \ge j \ge 1$, for some constant $c>1$,
		\begin{equation*}
		\label{eq:m-integral}
		\sum_{t=j}^{T}  \tim_{j, t}  \exp\left(  - \tim_{j, t} \right) \le 
		c \left[\frac{2^{\frac{1}{1-\alpha}}}{(1-\alpha)^{\frac{3}{2}}} \frac{1}{(1-\gamma)^{\frac{1}{1-\alpha}}} + \frac{ 2^{\frac{1}{1-\alpha}}}{1-\gamma} (j-1)^{\alpha}\right].
		\end{equation*}
	\end{lem}
	
	Therefore, 
	\begin{align*}
	\frac{1}{T}\sum_{j=1}^T\|\A_j^T-\G^{-1}\|_\infty^2
	&\le \frac{2}{T}\sum_{j=1}^T\left[ \left\|\M_{T,j}^{(1)}\right\|_{\infty}^2 + \left\|\M_{T,j}^{(2)}\right\|_{\infty}^2  \right]\\
	&\le \frac{2c}{T}\sum_{j=1}^T\left[ 
	\frac{\alpha^2}{j^2} \frac{2^{\frac{2}{1-\alpha}}}{(1-\alpha)^3} \frac{1}{(1-\gamma)^{2+\frac{2}{1-\alpha}}} + \frac{\alpha^22^{\frac{2}{1-\alpha}}}{(1-\gamma)^4} \frac{1}{j^{2(1-\alpha)}}
	+ \frac{(1-\eeta_{T})^{2(T-j+1)}}{(1-\gamma)^2} 
	\right]\\
	&\le \frac{c\alpha^22^{2+\frac{2}{1-\alpha}}}{T} \left[ \frac{1}{(1-\alpha)^3} \frac{1}{(1-\gamma)^{2+\frac{2}{1-\alpha}}} 
	+ 
	\frac{1}{(1-\gamma)^4} \sum_{j=1}^T \frac{1}{j^{2(1-\alpha)}}\right]
	+\frac{1}{(1-\gamma)^2} \frac{1}{T \eeta_{T}}.
	\end{align*}
\end{proof}

\begin{proof}[Proof of Lemma~\ref{lem:m-integral}]
	Clearly we have
	\[
	\frac{1-\gamma}{1-\alpha} \left( (t+1)^{1-\alpha}- j^{1-\alpha}  \right)
	\le \tim_{j, t} = \sum_{i=j}^{t} \eeta_i
	\le 	\frac{1-\gamma}{1-\alpha} \left( t^{1-\alpha}- {(j-1)}^{1-\alpha}  \right).
	\]
	Then $ \tim_{j, t} \le 	\frac{1-\gamma}{1-\alpha} \left( t^{1-\alpha}- {(j-1)}^{1-\alpha}  \right) \le \tim_{j-1, t-1}$.
	Hence,
	\begin{align*}
	\sum_{t=j}^{T}  \tim_{j, t}  \exp\left(  - \tim_{j, t} \right)
	&=\sum_{t=j}^{T}  \tim_{j, t} \exp(-\tim_{j-1, t-1})\exp(\eeta_{j-1}-\eeta_{t})\\
	&= e\sum_{t=j}^{T} \frac{1-\gamma}{1-\alpha} \left( t^{1-\alpha}- {(j-1)}^{1-\alpha}  \right) \exp\left( -	\frac{1-\gamma}{1-\alpha} \left( t^{1-\alpha}- {(j-1)}^{1-\alpha}  \right)  \right)\\
	&\le 2e\int_{j-1}^{\infty}\frac{1-\gamma}{1-\alpha} \left( t^{1-\alpha}- {(j-1)}^{1-\alpha}  \right) \exp\left( -	\frac{1-\gamma}{1-\alpha} \left( t^{1-\alpha}- {(j-1)}^{1-\alpha}  \right)  \right) dt\\
	&\overset{(a)}{=}\frac{2e}{1-\gamma}\int_0^{\infty}y\exp\left( -	y \right) \left( \frac{1-\alpha}{1-\gamma}y + (j-1)^{1-\alpha} \right)^{\frac{\alpha}{1-\alpha}}dt\\
	&\overset{(b)}{\le}\frac{e\max\{ 2^{\frac{\alpha}{1-\alpha}}, 2 \}}{1-\gamma}\int_0^{\infty}y\exp\left( -	y \right) \left[\left( \frac{1-\alpha}{1-\gamma}y\right)^{\frac{\alpha}{1-\alpha}} +  (j-1)^{\alpha}\right] dt\\
	&\overset{(c)}{\le} \frac{e 2^{\frac{1}{1-\alpha}}}{1-\gamma}
	\left[ \left( \frac{1-\alpha}{1-\gamma}\right)^{\frac{\alpha}{1-\alpha}}  \Gamma\left(1+ \frac{1}{1-\alpha} \right) + (j-1)^{\alpha} \right]\\
	&\overset{(d)}{\le}
	\frac{\sqrt{2\pi} e^{\frac{3}{2}}2^{\frac{1}{1-\alpha}}}{(1-\alpha)^{\frac{3}{2}}} \frac{1}{(1-\gamma)^{\frac{1}{1-\alpha}}} + \frac{e 2^{\frac{1}{1-\alpha}}}{1-\gamma} (j-1)^{\alpha},
	\end{align*}
	where $(a)$ uses the change of variable $y=\frac{1-\gamma}{1-\alpha} \left( t^{1-\alpha}- {(j-1)}^{1-\alpha}  \right)$, 
	$(b)$ uses $(a+b)^p\le\max\{2^{p-1},1\}(a^p+b^p)$ for any $p >0$,
	$(c)$ uses $\max\left\{2^{\frac{\alpha}{1-\alpha}}, 2 \right\} \le 2^{\frac{1}{1-\alpha}}$,
	$(d)$ uses $\Gamma\left(1 + \frac{1}{1-\alpha}\right) = \frac{1}{1-\alpha}\Gamma\left(\frac{1}{1-\alpha}\right) $ and $(1-\alpha)^{\frac{\alpha}{1-\alpha}}\Gamma\left(\frac{1}{1-\alpha}\right) \le \frac{\sqrt{2\pi e}}{\sqrt{1-\alpha}}$ from~\eqref{eq:Gamma}.
\end{proof}

\subsubsection{Proof of Lemma~\ref{lem:T4}}
\label{proof:T4}
\begin{proof}[Proof of Lemma~\ref{lem:T4}]
	By Lemma~\ref{lem:gap} and Lemma~\ref{lem:AjT-bound}, it follows that
	\begin{align*}
	\EB \|\TM_5\|_{\sup}
	&=\EB\sup_{r \in [0, 1]}\|\TM_5(r)\|_{\infty} 
	\le \frac{\gamma}{\sqrt{T}}  \EB \sup_{r \in [0, 1]} \sum_{j=1}^{\Tr}\left\|\A_j^{\Tr}( \PP^{\pi_{j-1}} -\PP^{\pi^*}) \DDelta_{j-1}\right\|_{\infty}\\
	&\le \frac{\gamma}{\sqrt{T}}  \EB \sum_{j=1}^T\sup_{r \in [0, 1]}
	\left\|\A_j^{\Tr}\right\|_{\infty} 
	\left\| (\PP^{\pi_{j-1}} -\PP^{\pi^*})\DDelta_{j-1}\right\|_{\infty} \\
	&\le \gamma LC_0 \cdot\frac{1}{\sqrt{T}}  \EB \sum_{j=1}^T
	\left\| \DDelta_{j-1}\right\|_{\infty}^2.
	\end{align*}
	Here we use $\sup_{r \in [0, 1]}
	\left\|\A_j^{\Tr}\right\|_{\infty}  \le C_0$ due to Lemma~\ref{lem:general-step-size}.
\end{proof}

\section{UNIFORM NEGLIGIBILITY OF NOISE RECURSION}
%Uniform Negligibility of Noise Recursion}
\label{appen:uniform}

\begin{defn}[Hurwitz matrix]
	We say $-\sG \in \RB^{d \times d}$ is a Hurwitz (or stable) matrix if $\mathrm{Re}\lambda_i(\sG) > 0$ for $i \in [d]$.
	Here $\lambda_i(\cdot)$ denotes the $i$-th eigenvalue.
\end{defn}

\begin{lem}[A generalization of Lemma B.7 in~\citep{li2021statistical}]
	\label{lem:ignore-error}
	Let $\{\varepsi_t\}_{t \ge 0}$ be a  martingale difference sequence adapting to the filtration $\FM_t$.
	Define an auxiliary sequence $\{ \y_t \}_{t\ge0}$ as following: $\y_0 = \0$ and for $t\ge 0$,
	\begin{equation}
		\label{eq:y-recursion}
		\y_{t+1} =(\sI - \eta_t\sG) \y_t + \eta_t\seps_t.
	\end{equation}
	It is easy to verify that
	\begin{equation}
		\label{eq:y}
		\y_{t+1} = \sum_{j=0}^t \left(\prod\limits_{i=j+1}^{t}\left( \sI - \eta_i\sG \right)\right) \eta_j  \seps_j.
	\end{equation}
	Let $\{\eta_t\}_{t \ge 0}$ satisfy Assumption~\ref{asmp:lr}.
	If $-\sG \in \RB^{d \times d}$ is Hurwitz, and $\sup_{t \ge 0}\EB \|\varepsi_t\|^4 < \infty$, then we have that
	\[
	\frac{1}{\sqrt{T}} \sup_{0 \le t \le T} \frac{\|\y_{t+1}\|}{\eta_{t+1}} \overset{p.}{
		\to} 0.
	\]
\end{lem}
\begin{proof}[Proof of Lemma~\ref{lem:ignore-error}]
	In the sequel, we denote $\check{\y}_t = \frac{\y_t}{\sqrt{\eta_{t-1}}}$.
	We will also use $a \precsim b$ to denote $a \le C b$ for unimportant positive constants $C$ with the specific value of $C$ changing according to the context.
	Then the update rule~\eqref{eq:y-recursion} can be rewritten as
	\begin{equation}\label{check_y_rec}
		\begin{aligned}
			\check{\y}_{t+1} &= \frac{\y_{t+1}}{\sqrt{\eta_t}} = \frac{1}{\sqrt{\eta_t}}(\y_t - \eta_t\sG\y_t + \eta_t\seps_t)\\
			&= \check{\y}_{t} + \left(\sqrt{\frac{\eta_{t-1}}{\eta_t}} - 1\right)\check{\y}_t - \sqrt{\eta_t\eta_{t-1}}\sG\check{\y}_t + \sqrt{\eta_t}\seps_t.
		\end{aligned}
	\end{equation}

	\paragraph{Step 1: Divide the time interval.}
	For a specific $\lambda > 0$, we divide the the time interval $[0, T]$ into several disjoint portions with the $t_k$ the $k$-th endpoint such that $\sum_{t_k}^{t_{k+1}-1}\eta_{s} \ge \lambda$.
	In particular, $\{t_k\}_{k \ge 0}$ is defined iteratively by $t_0=0$ and
	\begin{align*}
		t_{k+1} = \min\left\{ n : \sum\limits_{s = t_k}^{n-1}\eta_s \ge \lambda\right\}\wedge\{T\}.
	\end{align*}
	Clearly, $K$ is the number of portions and we have $0= t_0 < t_1 < \cdots < t_K  = T$.
	Since $\sum_{t=1}^\infty \eta_t = \infty$, we know that $K \to \infty$ as $T \to \infty$
	What's more, $K$ is upper bounded by $\frac{1}{\lambda}\sum\limits_{t=0}^{T} \eta_t$ due to
	\begin{equation}
		\label{eq:K-bound}
		\sum\limits_{t=0}^{T-1} \eta_t = \sum\limits_{k=0}^{K-1}\sum\limits_{s = t_k}^{t_{k+1}-1}\eta_s \ge \lambda K.
	\end{equation}
	The fact $\sup\limits_{t\le T}\frac{\|\y_t\|}{\eta_{t-1}} \le \sup\limits_{t\le T}\frac{\|\check{\y_t}\|}{\sqrt{\eta_{T}}}$ implies we have for any $\eps > 0$,
	\begin{equation}
		\label{eq:y-help-bound}
		\PB\left(\frac{1}{\sqrt{T}}\sup\limits_{t\ge T}\frac{\|\y_t\|}{\eta_{t-1}} > \eps\right)\le \PB\left(\sup\limits_{t\le T}\|\check{\y}_t\| > \eps \sqrt{T\eta_T}\right).
	\end{equation}

	\begin{lem}
		\label{lem:y-p-consistent}
		Let $\{\y_t\}_{t \ge 0}$ be defined in the way of~\eqref{eq:y-recursion}.
		If $-\sG \in \RB^{d \times d}$ is Hurwitz and $\sup_{t \ge 0}\EB\|\varepsi_t\|^{p} < \infty$ for $p\ge 4$, then the sequence $\{\y_t\}_{t \ge 0}$ is $(L^{4}, \sqrt{\eta_t})$-consistency, that is, there exists a universal constant $C_4>0$ such that $\EB\|\y_t\|^{4} \le C_4 \eta_t^{2}$ for all $t \ge 0$.
	\end{lem}
	The proof of Lemma~\ref{lem:y-p-consistent} is deferred in Section~\ref{proof:y-consistency}.
	Lemma~\ref{lem:y-p-consistent} implies that $\sup_{t \ge 0}\EB\|\check{\y}_t\|^4 \precsim 1$.
	Let $\BM := \left\{ \sup_{1 \le k \le K} \|\check{\y}_{t_k}\|\le \eps\sqrt{T\eta_T}\right\}$ be the event where all $\|\check{\y}_{t_k}\|$'s are smaller than $\eps\sqrt{T\eta_T}$ for $1 \le k \le K$.
	By the union bound and Markov inequality,
	\begin{align*}
		\PB(\BM^c) &\le \ssum{k}{1}{K}\PB\left(\|\check{\y}_{t_k}\|> \eps\sqrt{T \eta_T}\right)
		\le \ssum{k}{1}{K}\frac{\EB \|\check{\y}_{t_k}\|^4}{\eps^4 (T \eta_T)^{2}} 
		\le \sup\limits_{t \ge 0}\EB\|\check{\y}_t\|^4 \cdot \frac{K}{\eps^4 (T\eta_T)^2}
		\precsim \frac{\sum_{t=0}^T \eta_t}{\lambda\eps^4 (T\eta_T)^2} \to 0.
	\end{align*}
	Here the last inequality uses $\eqref{eq:K-bound}$ and the condition on $\{\eta_t\}_{t \ge 0}$ that $\frac{\sum_{t=0}^T \eta_t}{(T\eta_T)^2} \to 0$ due to $\frac{\sum_{t=0}^T \eta_t}{T\eta_T} \le C$ and $\eta_T T \to \infty$. 
	The above result implies for given $\lambda, \eps >0$, the event $\BM$ holds with probability approaching one.
	Hence, we focus our analysis on the event $\BM$.
	Conditioning on the event $\BM$, we split our target event into several disjoint events whose probability will be analyzed latter.
	\begin{align}
		\label{eq:y-bound-B}
		\PB\left(\sup\limits_{t\le T}\|\check{\y}_t\| > 3\eps\sqrt{T\eta_T}\right) &\le \PB(\BM^c) + \PB\left(\sup\limits_{t\le T}\|\check{\y}_t\|> \eps \sqrt{T \eta_T}; \BM\right)\nonumber \\ 
		&\le \PB(\BM^c) + \ssum{k}{0}{K-1}\PB\left(\sup\limits_{t\in [t_k, t_{k+1}-1]}\|\check{\y}_t\| > 3\eps \sqrt{T \eta_T};  \BM\right) \nonumber \\ 
		&\le \PB(\BM^c) + \ssum{k}{0}{K-1}\PB\left(\sup\limits_{t\in [t_k, t_{k+1}-1]}\|\check{\y}_t-\check{\y}_{t_k}\| > 2\eps \sqrt{T \eta_T};  \BM\right)\nonumber\\ 
		&\le \PB(\BM^c) + \ssum{k}{0}{K-1}\PB\left(\sup\limits_{t\in [t_k, t_{k+1}-1]}\|\check{\y}_t-\check{\y}_{t_k}\| > 2\eps \sqrt{T \eta_T}\right)\nonumber\\ 
		&:= \PB(\BM^c) + \ssum{k}{1}{K-1}\PM_k.
	\end{align}
	\paragraph{Step 2: Bound each $\PM_k$.}
	Leveraging \eqref{check_y_rec} recursively implies for given $r< t$,
	\begin{align*}
		\check{\y}_{t} - \check{\y}_{r} = \ssum{s}{r}{t-1}\left\{\left(\sqrt{\frac{\eta_{s-1}}{\eta_s}}-1\right)\check{\y}_s - \sqrt{\eta_{s-1}\eta_s}\sG\check{\y}_s + \sqrt{\eta_s}\seps_s\right\}.
	\end{align*}
	As a result,
	\begin{align}
		\PM_k  &= \PB\left(\sup\limits_{t\in [t_k,t_{k+1}-1]}\left\|\ssum{s}{t_k}{t-1}\left\{\left(\sqrt{\frac{\eta_{s-1}}{\eta_s}}-1\right)\check{\y}_s - \sqrt{\eta_{s-1}\eta_s}\sG\check{\y}_s + \sqrt{\eta_s}\seps_s\right\}\right\|> 2\eps\sqrt{T \eta_T}\right)\nonumber \\ 
		&\le \PB\left(\sup\limits_{t\in[t_k, t_{k+1}-1]}\left\|\ssum{s}{t_k}{t-1}\left\{\left(\sqrt{\frac{\eta_{s-1}}{\eta_s}}-1\right)\check{\y}_s - \sqrt{\eta_{s-1}\eta_s}\sG\check{\y}_s\right\}\right\|> \eps\sqrt{T \eta_T}\right) \nonumber\\ 
		&~~ +\PB\left(\sup\limits_{t\in[t_k,t_{k+1}-1]}\left\|\ssum{s}{t_k}{t-1}\sqrt{\eta_s}\seps_s\right\|> \eps\sqrt{T \eta_T}\right)\nonumber\\ 
		&=: \PM_k^{(1)} + \PM_k^{(2)}. \label{eq:P-bound-two}
	\end{align}
	% Before analyzing the two terms $\PM_k^{(1)}, \PM_k^{(2)}$, we concentrate on the property of the time points sequence $\{t_k\}_{k=0}^K$. 
	In the following, we highlight the dependence on $T$ and $\lambda$ and use $\precsim$ to omit universal constants.
	
	We consider to bound $\PM_k^{(1)}$ first. 
	Since $\frac{\eta_t}{\eta_{t-1}}= 1-o(\eta_{t-1})$, we have $\sqrt{\frac{\eta_{t-1}}{\eta_t}}-1 = \frac{1}{\sqrt{1-o(\eta_{t-1})}} -1 = o(\eta_{t-1})$.
	Hence, there exists a universal positive $C>0$ such that
	$\left\|\left(\sqrt{\frac{\eta_{t-1}}{\eta_t}}-1\right)\check{\y}_t - \sqrt{\eta_{t-1}\eta_t}\sG\check{\y}_t\right\| \le C \eta_t \|\check{\y}_t\|$ for all $t \ge 0$.
	As a result,
	\begin{align}
		\label{bd_p_k1}
		\PM_k^{(1)} 
		&= \PB\left(\sup\limits_{t\in [t_k,t_{k+1}-1]}\ssum{s}{t_k}{t-1}\left\|\left(\sqrt{\frac{\eta_{s-1}}{\eta_s}}-1\right)\check{\y}_s - \sqrt{\eta_{s-1}\eta_s}\sG\check{\y}_s\right\| > \eps\sqrt{T \eta_T}\right) \nonumber \\ 
		&\le \PB\left(\ssum{s}{t_k}{t_{k+1}-1}C_0\eta_s\|\check{\y}_s\| > \eps\sqrt{T \eta_T}\right)
		\precsim  \frac{1}{\eps^2 T \eta_T} \cdot \EB\left(\ssum{s}{t_k}{t_{k+1}-1}\eta_s\|\check{\y}_s\|\right)^2 \nonumber\\ 
		&\le \frac{1}{\eps^2 T \eta_T}\left\{\left(\ssum{s}{t_k}{t_{k+1}-1}\eta_s\right)\left(\ssum{s}{t_k}{t_{k+1}-1}\eta_s \EB\|\check{\y}_s\|^2\right)\right\} \nonumber \\ 
		&\le \frac{\sup\limits_{t \ge 0}\EB\|\check{\y}_t\|^2}{\eps^2 T \eta_T}\left(\ssum{s}{t_k}{t_{k+1}-1}\eta_s\right)^2
		\precsim 
		\frac{1}{\eps^2 T \eta_T}\left(\ssum{s}{t_k}{t_{k+1}-1}\eta_s\right)^2.
	\end{align}
	Let $K_0 = \max\{m\ge0:\eta_m \ge \lambda\}$. 
	Since $\eta_t$ decreases in $t$ and converges to 0, we know $K_0$ also decreases in $\lambda$.
	If $t_{k}\le K_0$, we have $t_{k+1}=t_k+1$ and thus $\ssum{s}{t_k}{t_{k+1}-1}\eta_s = \eta_{t_k}\le \eta_0$; otherwise, $\ssum{s}{t_k}{t_{k+1}-1}\eta_s \le 2\lambda$ by definition. 
	Summing over $\PM_k^{(1)}$ from $0$ to $K-1$ and using \eqref{bd_p_k1} yield
	\begin{align*}
		\ssum{k}{0}{K-1}\PM_k^{(1)} &= \ssum{k}{0}{K_0-1}\PM_k^{(1)} + \ssum{k}{K_0}{K-1}\PM_k^{(1)}
		\precsim \ssum{k}{0}{K_0-1}\frac{\eta_0^2}{\eps^2 T\eta_T} + \ssum{k}{K_0}{K-1}\frac{4\lambda^2}{\eps^2 T\eta_T}\\ 
		&\le \frac{1}{\eps^2 T\eta_T}\left(K_0\eta_0^2 + 4K\lambda^2\right)
		\overset{\eqref{eq:K-bound}}{\precsim} \frac{1}{\eps^2 T\eta_T}\left(K_0 + \lambda\sum_{t=0}^T\eta_t\right)\\ 
		&\precsim \frac{K_0}{\eps^2 T\eta_T} + \frac{\lambda}{\eps^2}.
	\end{align*}
	The last inequality uses $\sum_{t=0}^T \eta_t \le C T\eta_T$ for all $T \ge 1$.
	For a given $\lambda$, letting $T\to \infty$ can make the first term go to zero. 
	Then letting $\lambda \to 0$ make the second term vanish too.
	Hence, we have 
	\[
	\lim_{\lambda \to 0}
	\lim_{T \to \infty }\ssum{k}{0}{K-1}\PM_k^{(1)} = 0.
	\]
	
	Next, we consider to bound $\PM_k^{(2)}$.
	To than end, we will use the Burkholder inequality which relates a martingale with its quadratic variation.
	\begin{lem}[Burkholder's inequality~\citep{burkholder1988sharp}]
		\label{lem:burkholder}
		Fix any $p \ge 2$.
		For a martingale difference $\{ \x_t \}_{t \in [T]}$ in a real (or complex) Hilbert space, each with finite $L^p$-norm, one has
		\[
		\EB\left\| \sum_{t=1}^T \x_t\right\|^p 
		\le B_p \EB \left(\sum_{t=1}^T \|\x_t\|^2 \right)^{\frac{p}{2}}
		\]
		where $B_p$ is a universal positive constant depending only on $p$.
	\end{lem}

	Hence,
	\begin{align*}
		\PM_k^{(2)} &= \PB\left(\sup\limits_{t\in [t_k, t_{k+1}-1]}\left\|\ssum{s}{t_k}{t-1}\sqrt{\eta_s}\seps_s\right\| > \eps\sqrt{T\eta_T}\right)\\
		&\le \frac{1}{\eps^4 (T\eta_T)^{2}}\EB \sup\limits_{t\in [t_k,t_{k+1}-1]}\left\|\ssum{s}{t_k}{t-1}\sqrt{\eta_s}\seps_s\right\|^4\\ 
		&\overset{(a)}{\precsim} \frac{1}{\eps^4 (T\eta_T)^{2}}
		\EB\left(\ssum{s}{t_k}{t_{k+1}-1}\eta_s\|\seps_s\|^2\right)^{2}\\
		&\overset{(b)}{\precsim} \frac{\left(\ssum{s}{t_k}{t_{k+1}-1}\eta_s\right)^{2}}{\eps^4 (T\eta_T)^{2}}\ssum{s}{t_k}{t_{k+1}-1}\frac{\eta_s}{\ssum{l}{t_k}{t_{k+1}-1}\eta_l}\EB\|\seps_s\|^4 \\
		&\overset{(c)}{\precsim} \frac{1}{\eps^4 (T\eta_T)^{2}}\left(\ssum{s}{t_k}{t_{k+1}-1}\eta_s\right)^{2}
	\end{align*}
	where $(a)$ uses Lemma~\ref{lem:burkholder}; $(b)$ uses Jensen's inequality; and $(c)$ uses $\sup\limits_{t \ge 0}\EB\|\seps_t\|^4 < \infty$.
	As before, we will discuss two cases depending on whether $\eta_t$ is larger than $\lambda$ or not.
	It is equivalent to whether $t_k$ is greater than $K_0$.
	Similar to the argument in bounding $\ssum{k}{0}{K-1}\PM_k^{(1)}$, we have
	\begin{align*}
		\ssum{k}{0}{K-1}\PM_k^{(2)} &= \ssum{k}{0}{K_0-1}\PM_k^{(2)} + \ssum{k}{K_0}{K-1}\PM_k^{(2)}\\ 
		&\precsim \frac{1}{\eps^4 (T\eta_T)^{2}} \left(  \ssum{k}{0}{K_0-1}\eta_0^{2} + \ssum{k}{K_0}{K-1}2^{p/2}\lambda^{2}\right)\\ 
		&\precsim \frac{1}{\eps^4 (T\eta_T)^{2} } \left(
		K_0 + K \lambda^{2}
		\right)\\
		&\precsim \frac{K_0}{\eps^4 (T\eta_T)^{2} } +
		\frac{\lambda}{\eps^4} \cdot
		\frac{\sum_{t=0}^T \eta_t}{ (T\eta_T)^{2} }\\
		&\precsim \frac{K_0}{\eps^4 (T\eta_T)^{2} } +
		\frac{\lambda}{\eps^4} \cdot
		\frac{C}{ T\eta_T }
	\end{align*}
	where the last inequality uses $\sum_{t=0}^T \eta_t \le C T\eta_T$ for all $T \ge 1$.
	From the last inequality, letting $T\to \infty$ makes these two terms converge to zero.
	Hence, we have 
	\[
	\lim_{\lambda \to 0}
	\lim_{T \to \infty }\ssum{k}{0}{K-1}\PM_k^{(2)} = 0.
	\]
	\paragraph{Step 3: Putting the pieces together.}
	Therefore, 
	\begin{align*}
		\lim\limits_{T\to \infty}\PB\left(\frac{1}{\sqrt{T}}\sup\limits_{t\ge T}\frac{\|\y_t\|}{\eta_{t-1}} > \eps\right)
		&\overset{\eqref{eq:y-help-bound}}{\le} \lim\limits_{T\to \infty}\PB\left(\sup\limits_{t\le T}\|\check{\y}_t\| > \eps \sqrt{T\eta_T}\right)\\
		&\overset{\eqref{eq:y-bound-B}}{\le}\lim\limits_{T\to \infty} \left(\PB(\BM^c) + \ssum{k}{1}{K-1}\PM_k  \right)=\lim\limits_{T\to \infty} \ssum{k}{1}{K-1}\PM_k   \\
		&\overset{\eqref{eq:P-bound-two}}{\le}\lim_{T\to \infty} \ssum{k}{1}{K-1}\left(\PM_k^{(1)} +\PM_k^{(2)}\right).
	\end{align*}
	Since the probability of the left-hand side has nothing to do with $\lambda$, letting $\lambda \to 0$ gives
	\[\lim\limits_{T\to \infty}\PB\left(\frac{1}{\sqrt{T}}\sup\limits_{t\ge T}\frac{\|\y_t\|}{\eta_{t-1}} > \eps\right)
	\le \lim_{\lambda \to 0} \lim_{T\to \infty}\ssum{k}{1}{K-1}\left(\PM_k^{(1)} +\PM_k^{(2)}\right) = 0.
	\]
\end{proof}

\subsection{Proof of Lemma~\ref{lem:y-p-consistent}}
\label{proof:y-consistency}
For the proof in the section, we will consider random variables (or matrices) in the complex field $\mathbb{C}$.
Hence, we will introduce new notations for them.
For a vector $\sv \in \mathbb{C}$ (or a matrix $\sU \in \mathbb{C}^{d \times d}$), we use $\sv^\mthH$ (or $\sU^\mthH$) to denote its Hermitian transpose or conjugate transpose.
For any two vectors $\sv, \su \in \mathbb{C}$, with a slight abuse of notation, we use $\langle\sv, \su\rangle = \sv^\mthH \su$ to denote the inner product in $\mathbb{C}$.
For simplicity, for a complex matrix $\sU \in \mathbb{C}^{d \times d}$, we use $\|\sU\|$ to denote the its operator norm introduced by the complex inner product $\langle \cdot, \cdot \rangle$.
When $\sU \in\RB^{d \times d}$, $\|\sU\|$ is reduced to the spectrum norm.
\begin{proof}[Proof of Lemma~\ref{lem:y-p-consistent}]
	By Lemma~\ref{lem:hurwitz}, $\sG = \sU \sD \sU^{-1}$ for two non-singular matrices $\sU, \sD \in \mathbb{C}^{d \times d}$ that satisfies $2\mu  \cdot \sI  \preceq \sD + \sD^\mthH$ with $\mu := \min_{i \in [d]}\lambda_i(\sG)$ for simplicity.
	\begin{lem}[Property of Hurwitz matrices, Lemma 1 in~\citep{mou2020linear}]
		\label{lem:hurwitz}
		If $-\sG \in \RB^{d \times d}$ be a Hurwitz matrix (i.e., $\mathrm{Re} \lambda_i(\sG) > 0$ for all $i \in [d]$), there exists a non-degenerate matrix $\sU \in \mathbb{C}^{d \times d}$
		such that $\sG = \sU \sD \sU^{-1}$ for some matrix $\sD \in \mathbb{C}^{d \times d}$ that satisfies
		\[
		2\min_{i \in [d]}\lambda_i(\sG) \cdot \sI  \preceq \sD + \sD^\mthH   
		\]
		where $\sD^\mthH$ denotes the conjugate transpose or Hermitian transpose.
	\end{lem}
	
	Notice that 
	\begin{align*}
		\|\sU^{-1}\y_{t+1}\|^2 
		&= 	\|\sU^{-1}\left[(\sI - \eta_t\sG) \y_t + \eta_t\seps_t\right]\|^2\\
		&= \|\sU^{-1}(\sI - \eta_t\sG) \y_t \|^2 + \eta_t^2\|\sU^{-1}\seps_t\|^2 + 2\eta_t\mathrm{Re}  \langle\sU^{-1}(\sI - \eta_t\sG) \y_t , \sU^{-1}\seps_t \rangle\\
		&\le\|\sI - \eta_t\sD\|^2 \|\sU^{-1} \y_t \|^2+\eta_t^2\|\sU^{-1}\seps_t\|^2  + 2\eta_t \mathrm{Re} \langle(\sI - \eta_t\sD) \sU^{-1}\y_t , \sU^{-1}\seps_t \rangle.
	\end{align*}
	We then bound $\|\sI - \eta_t\sD\|$ as following.
	\begin{align*}
		\|\sI -\eta_t \sD\|^2 
		&=  \sup_{\sv \in \mathbb{C}^d, \|\sv\| =1} 
		\sv^\mthH (\sI - \eta_t \sD)^\mthH(\sI - \eta_t \sD) \sv\\
		&=\sup_{\sv \in \mathbb{C}^d, \|\sv\| =1}  \left(
		\|\sv\|^2 - \eta_t \sv^\mthH (\sD^\mthH + \sD) \sv + \eta_t^2 \sv^\mthH \sD^\mthH \sD \sv 
		\right)\\
		&\le 1 - 2\eta_t \mu + \eta_t^2 \|\sD\|^2.
	\end{align*}
	For simplicity, we define
	\[
	h_{t} =\frac{\|\sU^{-1}\y_{t}\|^2}{\eta_{t}}.
	\]
	Then we have
	\begin{align}
		\label{eq:h-recur}
		h_{t+1} =\frac{\|\sU^{-1}\y_{t+1}\|^2}{\eta_{t+1}}
		&\le \left(1 - 2\eta_t \mu + \eta_t^2 \|\sD\|^2\right) \frac{\eta_t}{\eta_{t+1}}
		h_t
		+ \frac{\eta_t^2}{\eta_{t+1}}\|\sU^{-1}\seps_t\|^2 \nonumber \\
		& \qquad	+\frac{2\eta_t }{\eta_{t+1}} \mathrm{Re} \langle(\sI - \eta_t\sD) \sU^{-1}\y_t , \sU^{-1}\seps_t \rangle. \nonumber\\
		&:= \left(1 - 2\eta_t \mu + \eta_t^2 \|\sD\|^2\right) \frac{\eta_t}{\eta_{t+1}}
		h_t + z_t
	\end{align}
	where for simplicity we denote
	\[
	z_t = 	\frac{\eta_t^2}{\eta_{t+1}}\|\sU^{-1}\seps_t\|^2 
	+\frac{2\eta_t }{\eta_{t+1}} \mathrm{Re} \langle(\sI - \eta_t\sD) \sU^{-1}\y_t , \sU^{-1}\seps_t \rangle.
	\]
	% 	In the following, we use the fact that for scalars $A>0$, any real number $x\ge -A$ and $\alpha \in (0, 1]$, 
	% 	\begin{equation}
		% 	   \label{eq:help-ineq1}
		% 	   	(A+x)^{1+\alpha} \le A^{1+\alpha} + (1+\alpha)A^{\alpha}x + |x|^{1+\alpha}.
		% 	\end{equation}
	% 	The proof of this inequality is straightforward: by homogeneity, we only need to prove for the case of $A = 1$.
	% 	Let $f(x) = 1+  ({1+\alpha})x + |x|^{1+\alpha}   -  (1+x)^{1+\alpha}$ and its derivative is $f'(x) = (1+\alpha) \left(1+ |x|^{\alpha}\mathrm{sign}(x)- (1+x)^{\alpha}\right)$.
	% 	When $1>\alpha>0$, we have $(1+x)^{\alpha} \le x^{\alpha} + 1$ for $x \ge 0$ and $1 \le (1-x)^{\alpha}+x^{\alpha}$ for $x \in [0, 1]$.
	% 	It implies that $f'(x)\ge0$ for $x\ge0$ and $f'(x)<0$ for $-1\le x<0$.
	% 	Hence, $f(x) \ge f(0)=0$ for any $x \ge -1$.
	Taking the second-order moment on the both sides of~\eqref{eq:h-recur}, we obtain
	\begin{align*}
		\EB h_{t+1}^2
		&\le \left[\left(1 - 2\eta_t \mu + \eta_t^2 \|\sD\|^2\right) \frac{\eta_t}{\eta_{t+1}} \right]^2
		\EB h_t^2 + \EB|z_t|^2\\
		&\qquad+ 2 \left[\left(1 - 2\eta_t \mu + \eta_t^2 \|\sD\|^2\right) \frac{\eta_t}{\eta_{t+1}} \right] \EB h_t z_t.
	\end{align*}
	Due to $\eta_{t+1} = (1-o(\eta_t))\eta_t$ and $\eta_t = o(1)$, there exists $t_0 > 0$ so that for any $t \ge t_0$, $\eta_t \le 2 \eta_{t+1}$ and
	\[
	0<\frac{\eta_t}{\eta_{t+1}}(1 - 2\eta_t \mu + \eta_t^2 \|\sD\|^2)
	= (1+o(1))(1+o(\eta_t))^2 (1 - 2\eta_t \mu +o(\eta_t) )
	\le 1-\mu \eta_t < 1.
	\]
	By Jensen's inequality,
	\begin{align*}
		\EB|z_t|^2
		&\le 2 \left(4 \eta_t^2 \EB \|\sU^{-1}\eps_t\|^4  + 16\EB|\mathrm{Re} \langle(\sI - \eta_t\sD) \sU^{-1}\y_t , \sU^{-1}\seps_t \rangle|^2 \right)\\
		& \precsim \eta_t^2 \EB \|\sU^{-1}\eps_t\|^4 + \EB\|\sU^{-1}\y_t \|^2\cdot\|\sU^{-1}\seps_t\|^2\\
		& = \eta_t^2 \EB \|\sU^{-1}\eps_t\|^4 + \EB(\eta_t h_t)\cdot\|\sU^{-1}\seps_t\|^2\\
		&\precsim \eta_t^2  \EB \|\sU^{-1}\eps_t\|^4+ \eta_t\sqrt{\EB h_t^2\cdot\EB \|\sU^{-1}\eps_t\|^4 }.
	\end{align*}
	Since $\EB[\mathrm{Re} \langle(\sI - \eta_t\sD) \sU^{-1}\y_t , \sU^{-1}\seps_t \rangle|\FM_t]=0$, it follows that
	\[
	\EB h_t z_t = \EB h_t\frac{\eta_t^2}{\eta_{t+1}}\|\sU^{-1}\seps_t\|^2 
	\le 2\eta_t  \cdot \EB h_t\|\sU^{-1}\seps_t\|^2 
	\le 2\eta_t \left(\EB\|\sU^{-1}\seps_t\|^4 \right)^{\frac{1}{2}} \left( \EB h_t^{2}\right)^{\frac{1}{2}}
	\]
	where the last inequality follows from Hölder's inequality.
	Notice that $\sup_{t \ge 0}\EB\|\seps_t\|^4 \precsim 1$ by assumption.
	Putting the pieces together, we have that there exists some $c>0$ such that
	\[
	\EB h_{t+1}^2
	\le (1-\mu\eta_t)	\EB h_{t}^2 + 
	c\left( \eta_t \left( \EB h_t^{2}\right)^{\frac{1}{2}}
	+ \eta_t^2 \right).
	\]
	By induction, one can show that 
	\[
	\EB h_t^2 \le \frac{c+\sqrt{c^2+4c\mu \eta_0}}{2\mu} \precsim 1
	\]
	of which the right hand side is the solution of the quadratic equation $\mu x = c \left(\sqrt{x} + \eta_0 \right)$.
	Since $\sU$ is non-singular, $\EB h_t^2 \precsim 1$ is equivalent to $\EB\|\y_t\|^4 \eta_t^{-2} \precsim 1$.
\end{proof}

\section{A CONVERGENCE RESULT}
\label{proof:convergence}
Denote $\DDelta_t = \Q_t  - \Q^*$ as the error of the Q-function estimate $\Q_t$ in the $t$-th iteration.
In this section, we study both asymptotic and non-asymptotic convergence of $\frac{1}{T}\sum_{t=0}^T \EB \|\DDelta_t\|_{\infty}^2$.

\subsection{For General Step Sizes}
\label{proof:general}
We first show that $\frac{1}{T}\sum_{t=0}^T \EB \|\DDelta_t\|_{\infty}^2 = o\left(\frac{1}{\sqrt{T}}\right)$ when using the general step size in Assumption~\ref{asmp:lr}.

\begin{thm}
	\label{thm:general-Linfty-pw2}
	Under Assumption~\ref{asmp:reward} and using the general step size in Assumption~\ref{asmp:lr}, we have 
	\begin{equation}
		\label{eq:help1}
		\lim_{T \to \infty} \frac{1}{\sqrt{T}}\sum_{t=0}^T \EB \|\DDelta_t\|_{\infty}^2 = 0.
	\end{equation}
\end{thm}

\begin{proof}[Proof of Theorem~\ref{thm:general-Linfty-pw2}]
	We will make use of the convergence result in~\citep{chen2020finite}.
	
	\begin{thm}[Theorem 2.1 and Corollary 2.1.3 in~\citep{chen2020finite}]
		\label{thm:chen}
		Consider the algorithm $\x_{t+1} = \x_{t} + \eta_{t}( \HM(\x_t)-\x_t + \varepsi_t)$ and $\x^*$ is the solution of $\HM(\x) = \x$.
		Assume $(i)$ $\|\HM(\x)-\HM(\y)\|_{\infty} \le \gamma \|\x -\y\|_{\infty}$ for any $\x, \y \in \RB^D$; $(ii)$ $\EB[ \varepsi_t|\FM_t] = \0$ and $\EB[\|\varepsi_t\|_{\infty}^2| \FM_t ] \le A + B \|\x_t\|_{\infty}^2$ and $(iii)$ $\eta_t$ is positive and non-increasing.
		If $\eta_0 \le \frac{\alpha_2}{\alpha_3}$, it follows that
		\[
		\EB\|\x_{t+1}-\x^*\|_{\infty}^2
		\le \alpha_1 \|\x_0-\x^*\|_{\infty}^2 \prod_{j=0}^{t} (1-\alpha_2 \eta_j) + \alpha_4 (A+2B\|\x^*\|_{\infty}^2) \sum_{j=1}^t \eta_j^2 \prod_{i=j+1}^{t} (1-\alpha_2 \eta_i). 
		\]
		where
		\[
		\alpha_1 \le \frac{3}{2},  \alpha_2 \ge \frac{1-\gamma}{2},
		\alpha_3 \le \frac{32e(B+2)\log D}{1-\gamma},
		\alpha_4 \le \frac{16e\log D}{1-\gamma}.
		\]
	\end{thm}
	
	Recall the update rule is $ \Q_t 
	= (1-\eta_t) \Q_{t-1} + \eta_t(\rr_t + \gamma \PP_t \V_{t-1})
	= \Q_{t-1} + \eta_t (\rr + \gamma \PP\V_{t-1} - \Q_{t-1} + \varepsi_t)
	$ where $\varepsi_t=\rr_t-\rr + \gamma (\PP_t- \PP)\V_{t-1}$.
	Let $\FM_t = \sigma(\{ (\rr_\tau, \PP_\tau) \}_{0 \le \tau < t})$.
	Hence,  $\EB[ \varepsi_t|\FM_t] = \0$ and $\EB[\|\varepsi_t\|_{\infty}^2| \FM_t ] \le 2\EB \|\rr_t-\rr\|_{\infty}^2 + 2 \gamma^2 \EB\|\PP_t-\PP\|_{\infty}^2 \|\V_{t-1}\|_{\infty}^2:= A + B  \|\Q_{t-1}\|_{\infty}^2$ 
	where the last equation uses $A = 2\EB \|\rr_t-\rr\|_{\infty}^2, B = 2 \gamma^2 \EB\|\PP_t-\PP\|_{\infty}^2 $ and $ \|\V_{t-1}\|_{\infty} = \|\Q_{t-1}\|_{\infty}$.
	Then setting $\eeta_t = (1-\gamma) \eta_t$, by Theorem~\ref{thm:chen}, we have
	\begin{equation}
		\label{eq:help3}
			\EB\|\De_{t}\|_{\infty}^2
		\le 2\|\De_0\|_{\infty}^2 \prod_{j=1}^{t} (1- 0.5\eeta_j) + C_1 \sum_{j=1}^t \eta_j \cdot 0.5\eeta_j \prod_{i=j+1}^{t} (1-0.5 \eeta_i). 
	\end{equation}
	where
	\[
	C_1 = \frac{32e\log D}{(1-\gamma)^2} (A+2B\|\Q^*\|_{\infty}^2).
	\]
	To simplify the notation, we denote
	\begin{equation}
		\label{eq:eeta-tT0}
		\eeta_{(t, T)}
		= \left\{\begin{array}{ll}
			\prod_{j=1}^{T}\left(1-0.5\eeta_{j}\right), & \text { if } t=0 \\
			0.5\eeta_{t} \prod_{j=t+1}^{T}\left(1-0.5\eeta_{j}\right), & \text { if } 0<t<T \\
			0.5\eeta_{T}, & \text { if } t=T .
		\end{array}\right.
	\end{equation}
	It is clear that we have $\sum_{t=0}^T \eeta_{(t, T)} = 1$.
	Then it follows that
			\[
	\EB\|\De_{t}\|_{\infty}^2
\le 2\|\De_0\|_{\infty}^2  \eeta_{(0, T)}+ C_1 \sum_{j=1}^t \eta_j \cdot 	\eeta_{(j, t)}.
	\]
	Therefore, it follows that
	\begin{align*}
		\frac{1}{\sqrt{T}} \sum_{t=1}^T \EB\|\DDelta_t\|_{\infty}^2
		&\le \frac{1}{\sqrt{T}} \sum_{t=1}^T  \left[2 \eeta_{(0, t)} \|\DDelta_0\|_{\infty}^2
		+ C_1
		\sum_{s=1}^t \eeta_{(s, t)} \eta_{s-1}
\right]
=\frac{2\|\DDelta_0\|_{\infty}^2}{\sqrt{T}} \sum_{t=1}^T  \eeta_{(0, t)} 
+  \frac{C_1}{\sqrt{T}} \sum_{t=1}^T
\sum_{s=1}^t \eeta_{(s, t)} \eta_{s-1}.
	\end{align*}

	Recall that Assumption~\ref{asmp:lr} requires the step size satisfies
\begin{itemize}
	\item [(C1)] $0 \le \sup_t \eta_t \le 1, \eta_t \downarrow 0$ and $t \eta_t \uparrow \infty$ when $t \to \infty$;
	\item [(C2)] $\frac{\eta_{t-1} - \eta_{t}}{\eta_{t-1}} =o(\eta_{t-1})$ for all $t \ge 1$;
	\item [(C3)] $\frac{1}{\sqrt{T}} \sum_{t=0}^T \eta_{t} \to 0$ when $T \to \infty$.
\end{itemize}
Noticing $t \eta_t \uparrow \infty$ due to (C1), we must have $\sum_{t=1}^{T} \eeta_t - \frac{1}{4}\ln T \to + \infty$ and thus implies
\[
\sqrt{T} \eeta_{(0, T)} = 
\sqrt{T} \prod_{t=1}^T (1-0.5\eeta_t)
\le \exp\left( \frac{1}{2}\ln T-2\sum_{t=1}^T\eeta_{t}  \right) \to 0,
\]
which, together with the Stolz–Cesaro theorem, implies $\frac{1}{\sqrt{T}} \sum_{t=1}^T  \eeta_{(0, t)}^2\to 0.$

On the other hand, by Lemma~\ref{lem:bounded-eeta} and (C3), it follows that
\begin{align*}
	\frac{1}{\sqrt{T}} \sum_{t=1}^T   \sum_{s=1}^t \eeta_{(s, t)} \eta_{s-1}
	= \frac{1}{\sqrt{T}} \sum_{s=1}^T \eta_{s-1} \cdot \sum_{t = s}^T
	\eeta_{(s, t)}
	\le
	\frac{c}{\sqrt{T}} \sum_{t=1}^T \eta_{t-1}   \to 0.
\end{align*}
\begin{lem}
	\label{lem:bounded-eeta}
	There exists some $c > 0$ such that $ \sum_{l = t}^T
	\eeta_{(t, l)} \le c$ for any $T \ge t \ge 1$.
	Here $\{\eeta_{(t, l)}\}_{l \ge t \ge 0}$ is defined in~\eqref{eq:eeta-tT0} and $\{ \eeta_t\}_{t\ge0}$ satisfies Assumption~\ref{asmp:lr}.
\end{lem}
Putting all pieces together, we have established~\eqref{eq:help1}.
\end{proof}

\begin{proof}[Proof of Lemma~\ref{lem:bounded-eeta}]
	We define $\tim_{t, l} := \sum_{i=t}^{l} \eeta_i$.
	Due to $t \eeta_{t} \uparrow \infty$, we have $t\eeta_t \le i\eeta_i$ for all $i \ge t$ and thus
	\[
	\tim_{t, l} := \sum_{i=t}^{l} \eeta_i \ge t\eeta_t  \sum_{i=t}^l \frac{1}{i} \ge  t\eeta_t  \left( \ln\frac{l}{t} - \frac{1}{2t} \right)
	= -\frac{\eeta_{t}}{2} +  t\eeta_t\ln\frac{l}{t}.
	\]
	Since $t \eeta_{t} \uparrow \infty$, there exists some $t_0 > 0$ such that any $t \ge t_0$, we have $t\eeta_t \ge 2$.
	Therefore, we have for all $l \ge t \ge t_0$,
	\begin{equation}
		\label{eq:help5}
		\frac{1}{\eeta_{l}} \le \frac{l}{t \eeta_{t}} 
		\le \frac{1}{\eeta_{t}} \exp\left( \frac{\tim_{t, l} + \frac{\eeta_{t}}{2}}{t\eeta_t} \right)
		\le \frac{\sqrt{e}}{\eeta_{t}} \exp\left( \frac{\tim_{t, l}}{2}\right).
	\end{equation}
	In the following, we will discuss three cases.
	\begin{itemize}
		\item 	If $T \ge t \ge t_0$, by definition, it follows that
		\begin{align*}
			\sum_{l = t}^T
			\eeta_{(t, l)}
			&=\sum_{l = t}^T\eeta_{t} \prod_{j=t+1}^{l}\left(1-\eeta_{j}\right) 
			\le \frac{\eeta_{t} }{1-\eeta_{t} }\sum_{l = t}^T\exp\left( -\tim_{t, l} \right)\\
			& \overset{(a)}{\le}\frac{\eeta_{t} }{1-\eeta_{t} }  \sum_{l = t}^T 
			\eeta_{l} \cdot \frac{\sqrt{e}}{\eeta_{t}} \exp\left( -\frac{\tim_{t, l}}{2}\right)\\
			&  \overset{(b)}{\le}\frac{\sqrt{e}}{\gamma}  \sum_{l = t}^T  \eeta_{l}  \exp\left( -\frac{\tim_{t, l}}{2}\right)   \overset{(c)}{\le} \frac{2\sqrt{e}}{\gamma},
		\end{align*}
		where $(a)$ follows from~\eqref{eq:help5}; $(b)$ uses $ 1- \eeta_{t} \ge 1-\eeta_{0} = \gamma$; and $(c)$ uses $ \sum_{l = t}^T  \eeta_{l}  \exp\left( -\frac{\tim_{t, l}}{2}\right)  \le \int_0^{\infty} \exp(-x/2)dx = 2$ due to $\tim_{t, l} \uparrow \infty$ as $l \to \infty$.
		\item 	If $T \ge t_0 \ge t$, by definition, $\eeta_{(t, l)} = \eeta_{(t, t_0)} \eeta_{(t_0, l)}/\eeta_{t_0} \le C_2 \eeta_{(t_0, l)} $
		where $C_2 = \sup_{0\le t\le t_0} \eeta_{(t, t_0)}/\eeta_{t_0}$.
		Then we have $\sum_{l = t}^T
		\eeta_{(t, l)} = \sum_{l = t}^{t_0}
		\eeta_{(t, l)} + \sum_{l = t_0}^T
		\eeta_{(t, l)} 
		\le t_0 + C_2\sum_{l = t_0}^T \eeta_{(t_0, l)}
		\le t_0 + C_2\frac{2\sqrt{e}}{\gamma}$.
		\item 	If $t_0 \ge T \ge t$, we have $\sum_{l = t}^T
		\eeta_{(t, l)} \le t_0$.
	\end{itemize}
	Putting the three cases together, we can set $c = t_0 + 2\max\{C_0, 1\}\sqrt{e}/\gamma$ which ensures that $ \sum_{l = t}^T
	\eeta_{(t, l)} \le c$ for any $T \ge t \ge 1$.
\end{proof}

\subsection{For Two Specific Step Sizes}
To obtain an $\log D$ dependence (which implies the rewards are distributed either sub-gaussian or sub-exponential), we use a almost-surely bounded rewards assumption as follows.

\begin{asmp}
	\label{asmp:lr-strong}
	We assume $0 \le R(s, a) \le 1$ for all $(s, a) \in \SM \times \AM$.
\end{asmp}

\begin{thm}
	\label{thm:ave_rate}
	Under Assumption~\ref{asmp:lr-strong}, there exist some positive constant $c > 0$ such that
	\begin{itemize}
		\item If $\eta_t = \frac{1}{1+(1-\gamma)t}$, it follows that
		\[
		\frac{1}{T} \sum_{t=0}^T \EB\|\DDelta_t\|_{\infty}^2
		\le c\left[  \frac{\|\DDelta_{0}\|_{\infty}^2}{(1-\gamma)^2} \frac{1}{T} + \frac{\ln(2eD)}{(1-\gamma)^5} \frac{
			\ln^2(eT)}{T}\right].
		\]
		\item If $\eta_t = t^{-\alpha}$ with $\alpha \in (0, 1)$ for $t\ge 1$ and $\eta_0 = 1$, it follows that
		\[
		\frac{1}{T} \sum_{t=0}^T \EB\|\DDelta_t\|_{\infty}^2
		\le c \left[
		\frac{ \Delta_0}{\sqrt{1-\alpha}(1-\gamma)^{\frac{1}{1-\alpha}}} \frac{1}{T} +   \frac{\ln(2eD)}{(1-\alpha)(1-\gamma)^4} \frac{1}{T^{\alpha}}\right],
		\]
		where
		\[
		\Delta_0 = 
		3 \|\DDelta_{0}\|_{\infty}^2 +  \frac{48\gamma^2\ln(2eD)}{(1-\gamma)^3}   \left(\frac{2\alpha}{1-\gamma}\right)^{\frac{1}{1-\alpha}}.
		\]
	\end{itemize}
\end{thm}

\subsection{Proof of Theorem~\ref{thm:ave_rate}}

Our proof is divided into three steps.
The first is a upper bound for $\|\DDelta_{t}\|_{\infty}$ provided by Lemma~\ref{lem:infty}: $\|\DDelta_{t}\|_{\infty} \le a_t + b_t + \|\N_t\|_{\infty}$, 
As a result, $\|\DDelta_{t}\|_{\infty}^2 \le 3(a_t^2 + b_t^2 + \|\N_t\|_{\infty}^2)$.
Lemma~\ref{lem:infty} follows from Theorem 1 in~\citep{wainwright2019stochastic} which views Q-learning as a cone-contractive operator and establishes a $\ell_{\infty}$-norm bound.

\begin{lem}[Theorem 1 in~\citep{wainwright2019stochastic}]
	\label{lem:infty}
	For any sequence of step sizes $\{\eta_t\}_{t \ge 0}$ in the interval $(0, 1)$, the iterates $\{ \DDelta_t \}_{t \ge 0}$ satisfies the sandwich relation
	\begin{equation}
		\label{eq:Delta-infty}
		-(a_t + b_t) \1 + \N_t   \le \DDelta_t  \le (a_t + b_t) \1 + \N_t
	\end{equation}
	where $\{a_t \}_{t \ge 0}, \{b_t\}_{t\ge0}$ are non-negative scalars and $\{\N_t\}_{t\ge0}$ are random vectors collecting noise terms from empirical Bellman operators.
	The three sequences are defined in a recursive way: they are initialized as $a_0 =\|\DDelta_{0}\|_{\infty}, b_0 = 0$ and $\N_0 = \0$ and satisfy the following recursion:
	\begin{gather*}
		a_t = (1-\eta_t(1-\gamma))a_{t-1}\\
		b_t = (1-\eta_t(1-\gamma))b_{t-1} + \eta_t \gamma \|\N_{t-1}\|_{\infty}\\
		\N_t = (1-\eta_t) \N_{t-1} + \eta_t \Z_t,
	\end{gather*}
	where $\Z_t =  (\rr_t-\rr) + \gamma( \PP_t - \PP) \V^*$ is the empirical Bellman error at iteration $t$.
\end{lem}

The second step is to bound $\EB\|\N_T\|_{\infty}^2$ which is an autoregressive process of independent Bellman noise terms.
One can prove the result following a similar argument of  Lemma 2 in~\citep{wainwright2019stochastic}.

\begin{lem}
	\label{lem:N}
	Under Assumption~\ref{asmp:lr-strong} and assuming $(1-\eta_t)\eta_{t-1} \le \eta_t$ for any $t \ge 1$, we have
	\[
	\EB \|\N_t\|_{\infty}^2 \le \frac{2\eta_{t}\ln(2eD) }{(1-\gamma)^{2}}.
	\]
\end{lem}

The final step is to establish the dependence of $\EB\|\DDelta_T\|_{\infty}^2$ on $\{\eta_t\}_{t\ge0}$.
\citet{wainwright2019stochastic} finds it is crucial to set $\eta_t$ to be proportional to $1/(1-\gamma)$ to ensure the sample complexity has polynomial dependence on $1/(1-\gamma)$.
We then set $\eeta_t = (1-\gamma)\eta_t$ as the rescaled step size.
We first redefine 
\begin{equation}
	\label{eq:eeta-tT}
	\eeta_{(t, T)}
	= \left\{\begin{array}{ll}
		\prod_{j=1}^{T}\left(1-\eeta_{j}\right), & \text { if } t=0 \\
		\eeta_{t} \prod_{j=t+1}^{T}\left(1-\eeta_{j}\right), & \text { if } 0<t<T \\
		\eeta_{T}, & \text { if } t=T .
	\end{array}\right.
\end{equation}
It is clear that we have $\sum_{t=0}^T \eeta_{(t, T)} = 1$.

\begin{lem}
	\label{lem:error}
	Under Assumption~\ref{asmp:reward}, if $(1-\eta_t)\eta_{t-1} \le \eta_t$ for any $t \ge 1$, then we have
	\begin{align}
		\label{eq:help8}
		\EB\|\DDelta_T\|_{\infty}^2
		\le 3 \eeta_{(0, T)}^2 \|\DDelta_0\|_{\infty}^2
		+ \frac{6\gamma^2\ln(2eD)}{(1-\gamma)^4} 
		\sum_{t=1}^T \eeta_{(t, T)}\eta_{t-1}
		+ \frac{6\ln(2eD) }{(1-\gamma)^{2}}\eta_{T},
	\end{align}
	where $\{\eeta_{(t, T)}\}_{T \ge t \ge0}$  defined in~\eqref{eq:eeta-tT} and $\{\N_t\}_{t \ge 0}$ is defined in Lemma~\ref{lem:infty}.
\end{lem}

\begin{proof}[Proof of Lemma~\ref{lem:error}]
		By the recursion of $\{a_t\}_{t\ge0}$ and $\{b_t\}_{t\ge0}$ in Lemma~\ref{lem:infty}, it follows that
	\begin{gather*}
		a_T = \prod_{t=1}^T (1-\eeta_t) \|\DDelta_0\|_{\infty} 
		= \eeta_{(0, T)}\|\DDelta_0\|_{\infty} \\
		b_T = \gamma\sum_{t=1}^T \prod_{j={t+1}}^{T} (1-\eeta_j) \eta_t \|\N_{t-1}\|_{\infty} = \frac{\gamma}{1-\gamma}\sum_{t=1}^T \eeta_{(t, T)}\|\N_{t-1}\|_{\infty}.
	\end{gather*}
	Hence, $a_T^2 = \eeta_{(0, T)}^2 \|\DDelta_0\|_{\infty}^2$ and 
	\begin{align*}
		\EB b_T^2 
		&= \frac{\gamma^2}{(1-\gamma)^2} \EB\left(\sum_{t=1}^T \eeta_{(t, T)} \|\N_{t-1}\|_{\infty}\right)^2
		\overset{(a)}{\le} \frac{\gamma^2}{(1-\gamma)^2} \sum_{t=1}^T \eeta_{(t, T)} \EB \|\N_{t-1}\|_{\infty}^2 
		%\\&\overset{(b)}{\le} \frac{\gamma^2\eeta_T}{(1-\gamma)^2} \sum_{t=1}^T \EB \|\N_{t-1}\|_{\infty}^2,
	\end{align*}
	where $(a)$ uses $\sum_{t=1}^T\eeta_{(t, T)}  = 1 -\eeta_{(0, T)}\le 1$ and Jensen's inequality.
	%, $(b)$ uses $\prod_{j={t+1}}^{T} (1-\eeta_j) \eeta_t  \le \eeta_{T}$ as a result of $(1-\eeta_j)\eeta_{j-1} \le \eeta_{j}$ for all $j$.
	
	Therefore,
	\begin{align}
		\EB\|\DDelta_T\|_{\infty}^2
		&\le 3(a_T^2 + \EB b_T^2 + \EB \|\N_T\|_{\infty}^2) \nonumber \\
		&\le 3 \eeta_{(0, T)}^2 \|\DDelta_0\|_{\infty}^2
		+ \frac{3\gamma^2}{(1-\gamma)^2} 
		\sum_{t=1}^T \eeta_{(t, T)}\EB \|\N_{t-1}\|_{\infty}^2 
		+ 3\EB \|\N_T\|_{\infty}^2.
		\label{eq:help4}
	\end{align}
	Given the condition $(1-\eta_t)\eta_{t-1} \le \eta_t$, we can apply Lemma~\ref{lem:N} which implies
	\[
	\EB \|\N_t\|_{\infty}^2 \le 
	\frac{2\eta_{t}\ln(2eD) }{(1-\gamma)^{2}}.
	\]
	Plugging these bounds into~\eqref{eq:help4} yields~\eqref{eq:help8}.
\end{proof}
With these lemmas, we are ready to prove the following theorem.
\begin{thm}
	\label{thm:Linfty-pw2}
	Under Assumption~\ref{asmp:reward}, we have the following bounds for $\EB\|\DDelta_T\|_{\infty}^2$. 
	Here $c > 0$ is a universal positive constant and might be overwritten (and thus different) in different statements.
	The specific value of different $c$'s can be found in our proof.
	\begin{itemize}
		\item If $\eta_t = \frac{1}{1+(1-\gamma)t}$, it follows that for all $T \ge 1$,
		\[
		\EB\|\DDelta_T\|_{\infty}^2
		\le  \frac{12\|\DDelta_{0}\|_{\infty}^2}{(1-\gamma)^2} \frac{1}{(1+T)^2}
		+  \frac{12\gamma^2\ln(2eD) }{(1-\gamma)^5} \frac{\ln(eT)}{T}.
		\]
		\item If $\eta_t = t^{-\alpha}$ with $\alpha \in (0, 1)$ for $t\ge 1$ and $\eta_0 = 1$, it follows that for all $T \ge 1$,
		\[ \EB\|\DDelta_T\|_{\infty}^2 \le \Delta_0  \exp\left(  - \frac{1-\gamma}{1-\alpha}\left( (1+T)^{1-\alpha} -1 \right)\right) +  \frac{114\ln(2eD)}{(1-\gamma)^4} \frac{1}{T^{\alpha}} \],
		where
		\[
		\Delta_0 = 
		3 \|\DDelta_{0}\|_{\infty}^2 +  \frac{48\gamma^2\ln(6D)}{(1-\gamma)^3}   \left(\frac{2\alpha}{1-\gamma}\right)^{\frac{1}{1-\alpha}}.
		\]
		%		\[
		%		\EB\|\DDelta_T\|_{\infty}^2
		%		\le \Delta_0  \exp\left(  - \frac{1-\gamma}{1-\alpha}\left( (1+T)^{1-\alpha} -1 \right)\right) + 
		%		\frac{c\ln(2eD)}{(1-\gamma)^2} \frac{\|\Var(\Z)\|_{\infty}}{T^\alpha} + \frac{c\ln^2(2eD)}{(1-\gamma)^4} \frac{1}{T^{2\alpha}}
		%		\]
		%		where 
		%		\begin{equation}
			%		\label{eq:Delta0}
			%		\Delta_0 = 
			%		c\|\DDelta_{0}\|_{\infty}^2 +  \frac{c\gamma^2\ln(2eD)}{(1-\gamma)^2}   \left(\frac{2\alpha}{1-\gamma}\right)^{\frac{1}{1-\alpha}}\|\Var(\Z)\|_{\infty} +
			%		\frac{c\gamma^2\ln^2(2eD) }{(1-\gamma)^3}   \left(\frac{3\alpha}{1-\gamma}\right)^{\frac{1}{1-\alpha}}.
			%		\end{equation}
	\end{itemize}
\end{thm}

\begin{proof}[Proof of Theorem~\ref{thm:Linfty-pw2}]
	We discuss the two cases separately.

	\paragraph{(I) Linearly rescaled step size.} If we use a linear rescaled step size, i.e., $\eta_t =  \frac{1}{1+ (1-\gamma)t}$ (equivalently $\eeta_t = \frac{1-\gamma}{1+ (1-\gamma)t}$), then we have (i) $1- \eta_t  \le 1- \eeta_t = \frac{1+(1-\gamma)(t-1)}{1+(1-\gamma)t} = {\eeta_t}/{\eeta_{t-1}} = {\eta_t}/{\eta_{t-1}}$ for $t \ge 1$ and (ii) $\eeta_{(t, T)} \le \eeta_T$.
	It implies Lemma~\ref{lem:error} is applicable.
	Notice that $\sum_{t=1}^{T} \eeta_{t-1}
	\le 1 + \sum_{t=1}^{T-1} \frac{1}{t}
	\le 1 + \ln (T-1) \le 
	\ln(eT)$ and $\ln \frac{(1-\gamma)(T+1)}{2}
	\le \ln \frac{1+(1-\gamma)(T+1)}{1+(1-\gamma)}
	= \int_1^{T+1}  \frac{1-\gamma}{1+ (1-\gamma)t}  dt
	\le \sum_{t=1}^T \frac{1-\gamma}{1+ (1-\gamma)t} 
	= \sum_{t=1}^{T} \eeta_t .$
	Hence, 
	\begin{gather*}
		\eeta_{(0, T)}^2 =	\prod_{t=1}^T (1-\eeta_t)^2  \le \exp\left(-2\sum_{t=1}^T\eeta_{t}\right)	 \le \frac{4}{(1-\gamma)^2} \frac{1}{(1+T)^2}\\
		\sum_{t=1}^T \eeta_{(t, T)}\eta_{t-1}
		=\frac{1}{1-\gamma}\sum_{t=1}^T \eeta_{(t, T)}\eeta_{t-1}
		\le \frac{\eeta_T}{1-\gamma}\sum_{t=1}^T \eeta_{t-1} \le \frac{\eeta_T\ln(eT)}{1-\gamma}.
	\end{gather*}
	Finally, plugging these inequalities into~\eqref{eq:help8}, we have
	\begin{align}
		\EB\|\DDelta_T\|_{\infty}^2
		&\le \frac{12\|\DDelta_{0}\|_{\infty}^2}{(1-\gamma)^2} \frac{1}{(1+T)^2}
		+  \frac{12\gamma^2\ln(2eD) }{(1-\gamma)^5} \frac{\ln(eT)}{T}.
		\label{eq:linear-last-step}
	\end{align}
	
	\paragraph{(II) Polynomial step size.} If we choose a polynomial step size, i.e., $\eta_t = t^{-\alpha}$ with $\alpha \in (0, 1)$ for $t\ge 1$ and $\eta_0 = 1$, then we again have $1-\eta_t = 1 - \frac{1}{t^\alpha} \le  \left(\frac{t-1}{t}\right)^{\alpha} = {\eta_t}/{\eta_{t-1}}$ for $t \ge 1$, which implies Lemma~\ref{lem:N} is applicable.
	%	Recalling that~\eqref{eq:Delta-constant-lr}, we have
	%	\[
	%	\EB\|\DDelta_T\|_{\infty}^2
	%	\le 3\prod_{t=1}^T (1-\eeta_t)^2 \|\DDelta_0\|_{\infty}^2
	%	+ \frac{6\gamma^2\ln(6D)}{(1-\gamma)^2} \sum_{t=1}^T \prod_{j={t+1}}^{T} (1-\eeta_j) \eta_t  \eta_{t-1}
	%	+ \frac{6\ln(6D)}{1-\gamma} \eta_T.
	%	\]
	Note that 
	\begin{equation}
		\label{eq:power-bound}
		\frac{(T+1)^{1-\alpha}- (t+1)^{1-\alpha}}{1-\alpha}
		=
		\int_{t+1}^{T+1}  j^{-\alpha}  dj  \le 
		\sum_{j=t+1}^{T} j^{-\alpha} 
		\le \int_t^T  j^{-\alpha}  dj = \frac{T^{1-\alpha}- t^{1-\alpha}}{1-\alpha},
	\end{equation}
	which implies that $\sum_{t=1}^{T} \eta_t \ge \sum_{t=1}^{T} t^{-\alpha} \ge \frac{1}{1-\alpha}\left((T+1)^{1-\alpha}-1\right) $ and $(T+1)^{1-\alpha} \le 1 + T^{1-\alpha}$.
	Hence, 
	\[
	\eeta_{(0, T)}^2=
	\prod_{t=1}^T (1-\eeta_t)^2 
	\le \exp\left(  -  2(1-\gamma)\sum_{t=1}^T \eta_{t} \right) \le \exp\left(  - 2\frac{1-\gamma}{1-\alpha}\left( (1+T)^{1-\alpha} -1 \right)\right).
	\]
	Additionally, using $\eta_{t-1} \le 2 \eta_{t}$ for all $t \ge 1$ and~\eqref{eq:power-bound}, we have,
	\begin{align*}
		\frac{\eeta_{(t, T)}}{1-\gamma} \eta_{t-1} =
		\prod_{j={t+1}}^{T} (1-\eeta_j) \eta_t  \eta_{t-1} 
		&\le 8 \prod_{j={t+1}}^{T} (1-\eeta_j) \eta_{t+1}^2
		\le 8\exp\left( - \sum_{j=t+1}^T \eeta_{j} \right) \eta_{t+1}^2\\
		&\le 8 \exp\left( - \frac{1-\gamma}{1-\alpha}  (1+T)^{1-\alpha}  \right)  \frac{\exp\left(  \frac{1-\gamma}{1-\alpha}  (t+1)^{1-\alpha} \right)}{(t+1)^{2\alpha}},
	\end{align*}
	which implies
	\begin{align*}
		\frac{1}{1-\gamma} 	\sum_{t=1}^T\eeta_{(t, T)}  \eta_{t-1} 
		&\le 	\frac{1}{1-\gamma} 	\sum_{t=1}^{T-1}\eeta_{(t, T)}  \eta_{t-1}  + 	\eta_T\eta_{T-1} 
		\le\frac{1}{1-\gamma} 	\sum_{t=1}^{T-1}\eeta_{(t, T)}  \eta_{t-1}  + 	\eta_T^2\\
		&\le 8 \sum_{t=2}^{T} \exp\left( - \frac{1-\gamma}{1-\alpha}  (1+T)^{1-\alpha}  \right)  \frac{\exp\left(  \frac{1-\gamma}{1-\alpha}  t^{1-\alpha} \right)}{t^{2\alpha}} + \frac{2}{T^{2\alpha}}.
	\end{align*}
	At the the end of this subsection, we will prove that
	\begin{lem}
		\label{lem:help1}
		For any $\alpha \in (0, 1)$ and $\beta > 0$, it follows that
		\begin{equation}
			\sum_{t=1}^T \frac{\exp\left(  \frac{1-\gamma}{1-\alpha}t^{1-\alpha} \right)}{t^{\beta}}
			\le \left(\frac{\beta}{1-\gamma}\right)^{\frac{1}{1-\alpha}} \exp\left(\frac{1-\gamma}{1-\alpha} \right)
			+  \frac{\beta}{(1-\gamma)\alpha} \frac{\exp\left( \frac{1-\gamma}{1-\alpha} (1+T)^{1-\alpha} \right)}{ (1+T)^{\beta-\alpha}}.
		\end{equation}
	\end{lem}
	By setting $\beta = 2 \alpha$, we have
	\[
	\sum_{t=1}^T \frac{\exp\left(  \frac{1-\gamma}{1-\alpha}t^{1-\alpha} \right)}{t^{2\alpha}}
	\le \left(\frac{2\alpha}{1-\gamma}\right)^{\frac{1}{1-\alpha}} \exp\left(\frac{1-\gamma}{1-\alpha} \right)
	+  \frac{2}{1-\gamma} \frac{\exp\left( \frac{1-\gamma}{1-\alpha} (1+T)^{1-\alpha} \right)}{ (1+T)^\alpha}.
	\]
	Therefore,
	\begin{align*}
		\frac{1}{1-\gamma} 	\sum_{t=1}^T\eeta_{(t, T)}  \eta_{t-1} 
		&\le 8 \left(\frac{2\alpha}{1-\gamma}\right)^{\frac{1}{1-\alpha}} \exp\left( - \frac{1-\gamma}{1-\alpha} \left( (1+T)^{1-\alpha} -1\right) \right)  + \frac{16}{1-\gamma} \frac{1}{(1+T)^\alpha} + \frac{2}{T^{2\alpha}}.
	\end{align*}
	Putting together the pieces, we can safely conclude that
	\begin{align*}
		\EB\|\DDelta_T\|_{\infty}^2
		&\le 3 \|\DDelta_{0}\|_{\infty}^2 \exp\left(  - 2\frac{1-\gamma}{1-\alpha}\left( (T+1)^{1-\alpha} -1 \right)\right) +
		\frac{6\ln(2eD)}{(1-\gamma)^2} \frac{1}{T^{\alpha}}
		+ \frac{96\gamma^2\ln(2eD)}{(1-\gamma)^4}  \frac{1}{ (1+T)^\alpha} \\
		&  + 
		\frac{12\gamma^2\ln(2eD)}{(1-\gamma)^3} \frac{1}{T^{2\alpha}} 
		+	 \frac{48\gamma^2\ln(2eD)}{(1-\gamma)^3} 
		\exp\left( - \frac{1-\gamma}{1-\alpha}  \left((1+T)^{1-\alpha} -1 \right) \right) 
		\left(\frac{2\alpha}{1-\gamma}\right)^{\frac{1}{1-\alpha}} \\
		&\le \Delta_0  \exp\left(  - \frac{1-\gamma}{1-\alpha}\left( (1+T)^{1-\alpha} -1 \right)\right) +  \frac{114\ln(2eD)}{(1-\gamma)^4} \frac{1}{T^{\alpha}},
	\end{align*}
	where
	\[
	\Delta_0 = 
	3 \|\DDelta_{0}\|_{\infty}^2 +  \frac{48\gamma^2\ln(6D)}{(1-\gamma)^3}   \left(\frac{2\alpha}{1-\gamma}\right)^{\frac{1}{1-\alpha}}.
	\]

\end{proof}

\begin{proof}[Proof of Lemma~\ref{lem:help1}]
	We do this via a similar argument of Lemma 4 in~\citep{wainwright2019stochastic}.
	Let $f(t) = \frac{\exp\left(  \frac{1-\gamma}{1-\alpha}t^{1-\alpha} \right)}{t^{\beta}}$.
	By taking derivatives, we find that $f(t)$ is decreasing in $t$ on the interval $[0, t^*]$ and increasing for $[t^*, \infty)$, where $t^* = \left(\frac{\beta}{1-\gamma}\right)^{\frac{1}{1-\alpha}}$.
	Hence,
	\begin{align*}
		\sum_{t=1}^T f(t) \le
		\left\{\begin{array}{ll}
			T f(1)  & \text { if } T \le \lfloor t^* \rfloor, \\
			\lfloor t^* \rfloor f(1) + \int_{ t^* }^{T+1} f(t) dt & \text { if } T > \lfloor t^* \rfloor.
		\end{array}\right.
	\end{align*}
	Using integrating by parts, it follows that
	\begin{align*}
		I^* := \int_{ t^* }^{T+1} f(t) dt
		&= \frac{\exp\left( \frac{1-\gamma}{1-\alpha} t^{1-\alpha} \right)}{(1-\gamma) t^{\beta-\alpha}} \bigg|_{t^*}^{T+1} + \frac{\beta -\alpha}{1-\gamma} \int_{t^*}^{T+1} \frac{\exp\left(  \frac{1-\gamma}{1-\alpha}t^{1-\alpha} \right)}{t^{1+\beta -\alpha}} dt \\
		&\le  \frac{\exp\left( \frac{1-\gamma}{1-\alpha} (1+T)^{1-\alpha} \right)}{(1-\gamma) (1+T)^{\beta -\alpha}} + \frac{\beta -\alpha}{1-\gamma}\int_{t^*}^{T+1} \frac{f(t)}{t^{1-\alpha}} dt \\
		&\le  \frac{\exp\left( \frac{1-\gamma}{1-\alpha} (1+T)^{1-\alpha} \right)}{(1-\gamma) (1+T)^{\beta -\alpha}} + \frac{\beta -\alpha}{1-\gamma} \frac{1}{(t^*)^{1-\alpha}}\int_{t^*}^{T+1} {f(t)} dt\\
		&=\frac{\exp\left( \frac{1-\gamma}{1-\alpha} (1+T)^{1-\alpha} \right)}{(1-\gamma) (1+T)^{\beta-\alpha}} + \frac{\beta -\alpha}{\beta} I^*,
	\end{align*}
	where the last equality uses definition of $t^*$ and $I^*$.
	Hence, we have
	\[
	I^* = \int_{ t^* }^{T+1} f(t) dt \le \frac{\beta}{(1-\gamma)\alpha} \frac{\exp\left( \frac{1-\gamma}{1-\alpha} (1+T)^{1-\alpha} \right)}{ (1+T)^{\beta -\alpha}}.
	\]
	Putting together the pieces, we have shown that if $ T > \lfloor t^* \rfloor$,
	\[
	\sum_{t=1}^T f(t) 
	\le t^* f(1) + I^*
	=  \left(\frac{\beta}{1-\gamma}\right)^{\frac{1}{1-\alpha}} \exp\left(\frac{1-\gamma}{1-\alpha} \right)
	+  \frac{\beta}{(1-\gamma)\alpha} \frac{\exp\left( \frac{1-\gamma}{1-\alpha} (1+T)^{1-\alpha} \right)}{ (1+T)^{\beta-\alpha}}.
	\]
	If  $T \le \lfloor t^* \rfloor$, then
	\[
	\sum_{t=1}^T f(t) 
	\le \lfloor t^* \rfloor f (1) \le t^* f(1) =  \left(\frac{\beta}{1-\gamma}\right)^{\frac{1}{1-\alpha}} \exp\left(\frac{1-\gamma}{1-\alpha} \right).
	\]
	Thus we have proved the inequality is true for any choice of $T$.
\end{proof}

Based on Theorem~\ref{thm:Linfty-pw2}, we now can prove Theorem~\ref{thm:ave_rate} by averaging the individual error bounds.
\begin{proof}[Proof of Theorem~\ref{thm:ave_rate}]
	The result directly follows from Theorem~\ref{thm:Linfty-pw2}.
	\begin{itemize}
		\item	For the first item, we already have $\EB\|\DDelta_T\|_{\infty}^2
		\le  \frac{12\|\DDelta_{0}\|_{\infty}^2}{(1-\gamma)^2} \frac{1}{(1+T)^2}
		+  \frac{12\gamma^2\ln(2eD) }{(1-\gamma)^5} \frac{\ln(eT)}{T}$.
		Using $\sum_{t=1}^{\infty}t^{-2} = \frac{\pi^2}{6}$ and $\sum_{t=1}^T t^{-1} \le 1 + \ln T = \ln(eT)$, we have for some universal constant $c>0$,
		\begin{align*}
			\frac{1}{T} \sum_{t=0}^T \EB\|\DDelta_t\|_{\infty}^2
			&\le  \frac{1}{T} \|\DDelta_0\|_{\infty}^2 +  \frac{1}{T}\sum_{t=1}^T\left[  \frac{12\|\DDelta_{0}\|_{\infty}^2}{(1-\gamma)^2} \frac{1}{(1+t)^2}
			+  \frac{12\gamma^2\ln(2eD) }{(1-\gamma)^5} \frac{\ln(eT)}{T}
			\right]\\
			&= c\left[  \frac{\|\DDelta_{0}\|_{\infty}^2}{(1-\gamma)^2} \frac{1}{T} + \frac{\ln(2eD)}{(1-\gamma)^5} \frac{
				\ln^2(eT)}{T}\right].
		\end{align*}
		\item 
		For the second item, we have $\EB\|\DDelta_T\|_{\infty}^2 \le \Delta_0  \exp\left(  - \frac{1-\gamma}{1-\alpha}\left( (1+T)^{1-\alpha} -1 \right)\right) +  \frac{114\ln(2eD)}{(1-\gamma)^4} \frac{1}{T^{\alpha}}$ with $\Delta_0 = 
		3 \|\DDelta_{0}\|_{\infty}^2 +  \frac{48\gamma^2\ln(2eD)}{(1-\gamma)^3}   \left(\frac{2\alpha}{1-\gamma}\right)^{\frac{1}{1-\alpha}}$.
		Notice that
		\begin{align*}
			\sum_{t=2}^{\infty} \exp\left(  - \frac{1-\gamma}{1-\alpha}\left( t^{1-\alpha} -1 \right)\right) 
			&\le \int_1^{\infty} \exp\left(  - \frac{1-\gamma}{1-\alpha}\left( t^{1-\alpha} -1 \right)\right) dt\\
			&\overset{(a)}{=} \frac{\exp\left( \frac{1-\gamma}{1-\alpha}\right)}{1-\gamma} \int_0^{\infty} e^{-x}  \left(\frac{1-\alpha}{1-\gamma} x \right)^{\frac{\alpha}{1-\alpha}} dx\\
			&\overset{(b)}{=} \frac{\exp\left( \frac{1-\gamma}{1-\alpha}\right)(1-\alpha)^{\frac{\alpha}{1-\alpha}}\Gamma(\frac{1}{1-\alpha})}{(1-\gamma)^{\frac{1}{1-\alpha}}}  \\
			&\overset{(c)}{\le} \frac{\sqrt{2\pi e}}{\sqrt{1-\alpha}}\frac{1}{(1-\gamma)^{\frac{1}{1-\alpha}}} 
		\end{align*}
		and $\sum_{t=1}^T t^{-\alpha} \le \int_0^T t^{-\alpha} dt = \frac{T^{1-\alpha}}{1-\alpha}$.
		Here $(a)$ uses the change of variable $x = \frac{1-\gamma}{1-\alpha} t^{1-\alpha}$ and $(b)$ uses the definition of gamma function $\Gamma(z) = \int_0^{\infty}e^{-x}x^{z-1}dx$.
		Finally $(c)$ follows from a numeral inequality about gamma function.
		Since $\Gamma(1+x) < \sqrt{2\pi}\left(\frac{x+1/2}{e}\right)^{x+1/2}$ for any $x > 0$ (see Theorem 1.5 of~\citep{batir2008inequalities}), then
		\[
		\Gamma\left(\frac{1}{1-\alpha}\right) \le \sqrt{2\pi} \left( \frac{1+\alpha}{2(1-\alpha)} \right)^{\frac{1+\alpha}{2(1-\alpha)}}\exp\left(-\frac{1+\alpha}{2(1-\alpha)}\right),
		\]
		which implies that
		\begin{equation}
			\label{eq:Gamma}
			\exp\left( \frac{1-\gamma}{1-\alpha}\right)(1-\alpha)^{\frac{\alpha}{1-\alpha}}\Gamma\left(\frac{1}{1-\alpha}\right) \le \frac{\sqrt{2\pi e}}{\sqrt{1-\alpha}}.
		\end{equation}
		Therefore,
		\begin{align*}
			\frac{1}{T} \sum_{t=0}^T \EB\|\DDelta_t\|_{\infty}^2
			&\le  \frac{1}{T} \|\DDelta_0\|_{\infty}^2 +  \frac{1}{T}\sum_{t=1}^T\left[  
			\Delta_0  \exp\left(  - \frac{1-\gamma}{1-\alpha}\left( (1+t)^{1-\alpha} -1 \right)\right) +  \frac{114\ln(2eD)}{(1-\gamma)^4} \frac{1}{t^{\alpha}}
			\right]\\
			&\le	c \left[
			\frac{ \Delta_0}{\sqrt{1-\alpha}(1-\gamma)^{\frac{1}{1-\alpha}}} \frac{1}{T} +   \frac{\ln(2eD)}{(1-\alpha)(1-\gamma)^4} \frac{1}{T^{\alpha}}\right].
		\end{align*}

	\end{itemize}
\end{proof}

\section{PROOF OF THEOREM~\ref{thm:con-linear}}
\label{proof:con-linear}
In the section, we provide the proof for our finite-sample analysis of averaged Q-learning in the $\ell_{\infty}$-norm.
Our main idea is similar to Appendix~\ref{proof:fclt}.
The average Q-learning estimator $\bar{\Q}_T$ has the error
\begin{equation}
\label{eq:Q-estimator}
\BDelta_T 
:= \frac{1}{T}  \sum_{t=1}^T \DDelta_t 
= \frac{1}{T}  \sum_{t=1}^T (\Q_t-\Q^*).
\end{equation}
Using two auxiliary sequences $\{ \DDelta_t^1\}_{t\ge0}$ and $\{ \DDelta_t^2\}_{t\ge0}$ defined in Lemma~\ref{lem:sandwitch}, we similarly define 
\[
\BDelta_T^1 := \frac{1}{T}  \sum_{t=1}^T \DDelta_t^1
\ \text{and} \
\BDelta_T^2 := \frac{1}{T}  \sum_{t=1}^T \DDelta_t^2.
\]
Because $\DDelta_t^2 \le \DDelta_t \le \DDelta_t^1$ coordinate-wise, it is valid that 
\begin{equation}
\label{eq:sandwitch-nonasy}
\BDelta_T^2  \le \BDelta_T \le  \BDelta_T^1.
\end{equation}
As a result, $\EB\|\BDelta_T \|_{\infty} \le \EB \max\{ \|\BDelta_T^1 \|_{\infty} , \|\BDelta_T^2 \|_{\infty}  \}$.
Hence, bounding $\|\BDelta_T \|_{\infty}$ in expectation is reduced to bound the maximum between $\|\BDelta_T^1 \|_{\infty}$ and $\|\BDelta_T^2 \|_{\infty}$.
Given $\BDelta_T^1$ and $\BDelta_T^2$ are defined in a similar way (see Lemma~\ref{lem:sandwitch}), they share a similar error decomposition.

\subsection{Error Decomposition}
\label{proof:non-linear-error}
Setting $r=1$ in~\eqref{eq:Delta1-begin-fclt}, we obtain
\[
\BDelta_T^1 = \frac{1}{T}  \sum_{t=1}^T \DDelta_t^1 = \frac{1}{\eta_0T}  (\A_0^{T}-\eta_0 \I) \DDelta_{0}
+\frac{1}{T}  \sum_{j=1}^{T}\A_j^{T}\left(\Z_j+\gamma \D_{j-1}^1\right).
\]
Similar to~\eqref{eq:Delta1-decom-fclt}, we decompose $\BDelta_T^1$ into five separate terms
\begin{align}
\label{eq:Delta1-decom}
\BDelta_T^1
&=\frac{1}{\eta_0T}   (\A_0^T-\eta_0 \I)  \DDelta_{0}
+\frac{1}{T}  \sum_{j=1}^T\G^{-1} \Z_j  + \frac{1}{T}  \sum_{j=1}^T  (\A_j^T - \G^{-1})\Z_j \nonumber \\
&\quad \quad 
+  \frac{\gamma}{T}  \sum_{j=1}^T\A_j^T( \PP_j - \PP) (\V_{j-1} - \V^*)
+ \frac{\gamma}{T}  \sum_{j=1}^T\A_j^T( \PP^{\pi_{j-1}} -\PP^{\pi^*}) \DDelta_{j-1}
\nonumber  \\
&:= \TM_0  + \TM_1 + \TM_2 + \TM_3 + \TM_4.
\end{align} 
Here one should distinguish $\TM_i$ with $\Bpsi_i$, the former a random variable and the latter a random function.
Comparing~\eqref{eq:D1} and~\eqref{eq:D2}, we find that $\D_{j-1}^1 = \D_{j-1}^2+( \PP^{\pi_{j-1}} -\PP^{\pi^*}) \DDelta_{j-1}$.
Repeating the same argument to $\BDelta_T^2$, we obtain
\begin{align}
\label{eq:Delta2-decom}
 \BDelta_T^2
 &= \frac{1}{T}  \sum_{t=1}^T \DDelta_t^2  = \frac{1}{\eta_0T}  (\A_0^{T}-\eta_0 \I) \DDelta_{0}
 +\frac{1}{T}  \sum_{j=1}^{T}\A_j^{T}\left(\Z_j+\gamma \D_{j-1}^2\right) \nonumber\\
&=\frac{1}{\eta_0T}  (\A_0^T -\eta_0\I)\DDelta_{0}
+\frac{1}{T}  \sum_{j=1}^T\G^{-1} \Z_j  + \frac{1}{T}  \sum_{j=1}^T  (\A_j^T - \G^{-1})\Z_j \nonumber \\
&\quad \quad 
+  \frac{\gamma}{T}  \sum_{j=1}^T\A_j^T( \PP_j - \PP) (\V_{j-1} - \V^*)  \nonumber
\\
&= \TM_0  + \TM_1 + \TM_2 + \TM_3.
\end{align} 
Here $\{\TM_i\}_{i=0}^3$ are exactly the same as in~\eqref{eq:Delta1-decom}.
Putting the pieces together, we have
\begin{equation}
\label{eq:non-asym-help0}
\EB\|\BDelta_T \|_{\infty} 
\le \EB \max\{ \|\BDelta_T^1 \|_{\infty} , \|\BDelta_T^2 \|_{\infty} \}\le  \sum_{i=0}^4\EB \|\TM_i\|_{\infty} .
\end{equation}

 \subsection{Bounding the Separate Terms}
\label{proof:five-term}
\paragraph{For $\|\TM_0\|_{\infty}$.}
 Recall that $C_0 = \sup_{T\ge j \ge 0} \|\A_j^T\|_{\infty}$. 
 Since $\eta_0 = 1 \le C_0$, it is obvious that 
\begin{align}
\label{eq:T0}
\|\TM_0\|_{\infty}
= \frac{1}{\eta_0T}  \| (\A_0^T-\eta_0 \I) \DDelta_{0}\|_{\infty} \le \frac{1}{\eta_0T} (\| \A_0^T\|_{\infty} + \eta_0)\|\DDelta_{0}\|_{\infty} \le  \frac{2C_0}{1-\gamma}\frac{1}{T}.
\end{align}

\paragraph{For $\|\TM_1\|_{\infty}$.}
We apply~\eqref{eq:expe-freedman} in Lemma~\ref{lem:freedman} to bound $\TM_1:=\frac{1}{T}  \sum_{j=1}^T\G^{-1} \Z_j $.
Indeed, by setting $\B_j \equiv \I, \X_j = \frac{1}{T}\G^{-1}\Z_j$, we have $B=1, X=\frac{1}{(1-\gamma)^2T}$ and $\|\W_T\|_{\infty} \le \frac{ \| \diag(\VQ)\|_{\infty}}{T}$ defined therein.
Hence,
\begin{equation}
\label{eq:T1}
\EB \|\TM_1\|_{\infty} \le  
6  \sqrt{\|\diag(\VQ)\|_{\infty}} \sqrt{\frac{\ln(2D)}{T}} + \frac{4\ln(6D)}{3(1-\gamma)^2T}.
\end{equation}

\paragraph{For $\|\TM_2\|_{\infty}$.}
We also apply~\eqref{eq:expe-freedman} in Lemma~\ref{lem:freedman} to analyze $\TM_{2} := \frac{1}{T}  \sum_{j=1}^T  (\A_j^T - \G^{-1})\Z_j$.
Indeed, by setting $\B_j =\A_j^T - \G^{-1}, \X_j = \frac{1}{T}\Z_j$, we have $B=2C_0, X=\frac{1}{(1-\gamma)T}$ and $\|\W_T\|_{\infty} \le\frac{1}{T^2} \sum_{j=1}^T \| \A_j^T - \G^{-1}\|_{\infty}^2\|\Var(\Z)\|_{\infty}$ defined therein.
Hence,
\begin{equation}
\label{eq:T2}
\EB \|\TM_2\|_{\infty} \le  
6\sqrt{\|\Var(\Z)\|_{\infty}}\sqrt{\frac{\ln(2D)}{T}}\sqrt{ \frac{1}{T} \sum_{j=1}^T \| \A_j^T - \G^{-1}\|_{\infty}^2 }+ \frac{8C_0\ln(6D)}{3(1-\gamma)T}.
\end{equation}

%\paragraph{For $\|\TM_1+\TM_2\|_{\infty}$.}
%As we will see, for the linearly rescaled step size, it is better to consider $\EB \| \TM_1+\TM_2\|_{\infty}$ rather than $\EB \| \TM_1\|_{\infty} + \EB \|\TM_2\|_{\infty}$.
%Notice that $\TM_1+\TM_2 := \frac{1}{T}  \sum_{j=1}^T\A_j^T\Z_j$.
%By setting $\B_j \equiv \A_j^T, \X_j = \frac{1}{T}\Z_j, B=C_0, X=\frac{1}{T(1-\gamma)}$ and $\|\W_T\|_{\infty} \le \frac{C_0^2}{T} \| \Var(\Z)\|_{\infty}$ in Lemma~\ref{lem:freedman}, we have
%\begin{equation}
%\label{eq:T1+T2}
%\EB \|\TM_1+\TM_2\|_{\infty} \le  
%6C_0 \sqrt{\|\Var(\Z)\|_{\infty}} \sqrt{\frac{\ln(2D)}{T}}+ \frac{4C_0\ln(6D)}{3(1-\gamma)T}.
%\end{equation}

\paragraph{For $\|\TM_3\|_{\infty}$.}
We apply~\eqref{eq:general} in Lemma~\ref{lem:freedman} to analyze $\TM_{3} := \frac{\gamma}{T}  \sum_{j=1}^T\A_j^T( \PP_j - \PP) (\V_{j-1} - \V^*)$.
Because $\TM_{3}$ is more complex than $\TM_{1}$ and $\TM_{2}$, we defer the detailed proof in Appendix~\ref{proof:T3}.

\begin{lem}
	\label{lem:T3}
	\begin{equation}
	\label{eq:T3}
		\EB\|\TM_3\|_{\infty}
	\le     {4\gamma C_0\sqrt{ \frac{\ln(2DT^2)}{T}}}\cdot
	\sqrt{ \frac{1}{T}\sum_{j=1}^T\EB\left\| \DDelta_{j-1}\right\|_{\infty}^2 }
	+  \frac{32\gamma C_0 \ln(3DT^2)}{3(1-\gamma)T}.
	\end{equation}
	where $C_0$ is the uniform bound given in Lemma~\ref{lem:AjT-bound} and $D=|\SM \times \AM|$. 
\end{lem}

\paragraph{For $\|\TM_4\|_{\infty}$.}
We have already analyzed $\TM_{4} := \frac{\gamma}{T}  \sum_{j=1}^T\A_j^T( \PP^{\pi_{j-1}} -\PP^{\pi^*}) \DDelta_{j-1}$ in Lemma~\ref{lem:T4}.
It follows that
\begin{equation}
\label{eq:T4}
\EB\|\TM_4\|_{\infty}  = \frac{1}{\sqrt{T}}\EB\|\Bpsi_{5}(1)\|_{\infty}  \le \frac{1}{\sqrt{T}}\EB\|\Bpsi_{5}\|_{\sup}  \le    \gamma LC_0\cdot \frac{1}{T}   \sum_{j=1}^T\EB
\left\| \DDelta_{j-1}\right\|_{\infty}^2.
\end{equation}

\begin{rem}
	\label{rmk:T}
Under Assumption~\ref{asmp:reward}~\ref{asmp:gap} and~\ref{asmp:lr}, we assert that $\sqrt{T}\EB\|\TM_i \|= o(1)$ for $i=0, 2, 3, 4$.
It is handy to verify $\sqrt{T}\|\TM_0\|=o(1)$.
Lemma~\ref{lem:general-step-size} implies $\frac{1}{T}\sum_{j=1}^T \| \A_j^T - \G^{-1}\|_{\infty}^2 = o(1)$, by which we conclude $\sqrt{T}\EB\|\TM_2\|=o(1)$.
Theorem~\ref{thm:general-Linfty-pw2} shows $\frac{1}{\sqrt{T}} \sum_{t=0}^T \EB\|\DDelta_t\|_{\infty}^2 \to 0$ when we use the general step size.
We then know that both $\sqrt{T}\EB\|\TM_3\|$ and $\sqrt{T}\EB\|\TM_4\|$ converge to zero when $T$ goes to infinity.
\end{rem}

\subsection{Specific Rates for Two Step Sizes}
\paragraph{(I) Linearly rescaled step size.} 
If we use a linear rescaled step size, i.e., $\eta_t =  \frac{1}{1+ (1-\gamma)t}$ (equivalently $\eeta_t = \frac{1-\gamma}{1+ (1-\gamma)t}$), then Lemma~\ref{lem:AjT-bound} and Lemma~\ref{lem:G-poly} give
\[
C_0 = \frac{2}{1-\gamma}\ln(1+(1-\gamma)T) = \OM\left(\frac{\ln T}{1-\gamma}\right)
\ \text{and} \
\frac{1}{T}\sum_{j=1}^T \| \A_j^T- \G^{-1}\|_{\infty}^2 \le \frac{25}{(1-\gamma)^2}.
\]
Hiding constant factors in $c$, Theorem~\ref{thm:ave_rate} gives
\[
\frac{1}{T} \sum_{t=0}^T \EB\|\DDelta_t\|_{\infty}^2
\le c\left[  \frac{\|\DDelta_{0}\|_{\infty}^2}{(1-\gamma)^2} \frac{1}{T} + \frac{\ln(2eD)}{(1-\gamma)^5} \frac{
	\ln^2(eT)}{T}\right].
\]
Hence, combining these bounds with~\eqref{eq:T0},~\eqref{eq:T1},~\eqref{eq:T2},~\eqref{eq:T3}, and~\eqref{eq:T4}, we have
\begin{align*}
\EB\|\BDelta_T\|_{\infty}
= 
\OM&\left(
\frac{\ln T}{(1-\gamma)^2T} +
\sqrt{\frac{\|\Var(\Z)\|_{\infty}}{(1-\gamma)^2}}\sqrt{\frac{\ln D}{T}}  + \frac{\ln D}{(1-\gamma)^2}  \frac{ \ln T}{T} \right.\\
&\qquad +\frac{\gamma\ln T \sqrt{\ln(DT)}}{(1-\gamma)^3}
\left(\frac{1}{T}+ \sqrt{\frac{\ln D}{1-\gamma}} \frac{\ln T}{T} \right) +
 \frac{\gamma\ln(DT) }{(1-\gamma)^2} \frac{\ln T 
	}{T} 
 \\
&\qquad + \left. + \frac{\gamma L \ln T}{1-\gamma} \left(
\frac{1}{(1-\gamma)^4} \frac{1}{T} + \frac{\ln D}{(1-\gamma)^5} \frac{
	\ln^2T}{T}
\right)
\right)\\
= \OM &\left(
\sqrt{\frac{\|\Var(\Z)\|_{\infty}}{(1-\gamma)^2}}\sqrt{\frac{\ln D}{T}} \right)    + \TOM\left( \frac{L}{(1-\gamma)^6} \frac{1}{T} 
\right),
\end{align*}
where $\widetilde{\OM}(\cdot)$ hides polynomial dependence on logarithmic terms namely $\ln D$ and $\ln T$.
Here we use $\|\diag(\VQ)\|_{\infty} \le \frac{\|\Var(\Z)\|_{\infty}}{(1-\gamma)^2}$ to simplify the final inequality.

\paragraph{(II) Polynomial step size.} 
%Using~\eqref{eq:help7},~\eqref{eq:help5} and other bounds $\EB \| \TM_i\|_{\infty} (i=0, 2, 3, 4)$ derived in Appendix~\ref{proof:five-term}, we have
%\begin{align*}
%\EB\|\BDelta_T\|_{\infty}
%&\le \frac{2C_0}{\eta_0(1-\gamma)} \frac{1}{T}
%+ 6 \sqrt{\ln(2D)} \sqrt{\frac{\|\diag(\VQ)\|_{\infty}}{T}} + \frac{4\ln(6D)}{3T(1-\gamma)^2} \\
%& \qquad
%+\sqrt{ \frac{2\ln(6D)}{1-\gamma^2} \frac{1}{T} } \cdot\sqrt{\frac{1}{T}\sum_{j=1}^T \| \A_j^T - \G^{-1}\|_{\infty}^2 }\\
%& \qquad +  4\gamma C_0\sqrt{  \frac{\ln(2DT^2)}{T} } \cdot
%\sqrt{ \frac{1}{T}\sum_{j=1}^T\EB\left\| \DDelta_{j-1}\right\|_{\infty}^2 }
%+  \frac{8\gamma C_0 \ln(3DT^2)}{3T(1-\gamma)} \\
%& \qquad + \gamma LC_0 \cdot \frac{1}{T}   \sum_{j=1}^T\EB
%\left\| \DDelta_{j-1}\right\|_{\infty}^2.
%\end{align*}
If we choose a polynomial step size, i.e., $\eta_t = t^{-\alpha}$ with $\alpha \in (0.5, 1)$ for $t\ge 1$ and $\eta_0 = 1$, then hiding constant factors in $c$, Lemma~\ref{lem:AjT-bound} and Lemma~\ref{lem:G-poly} give
\begin{gather*}
C_0 =\OM\left( \frac{1}{(1-\gamma)^{\frac{1}{1-\alpha}}}\right)\\
\sqrt{\frac{1}{T}\sum_{j=1}^T \| \A_j^T- \G^{-1}\|_{\infty}^2}
= \OM\left(\frac{1}{(1-\gamma)^{1+\frac{1}{1-\alpha}}} \frac{1}{\sqrt{T}} 
+  \frac{1}{(1-\gamma)^2}  \frac{1}{T^{1-\alpha}} 
+\frac{1}{(1-\gamma)^{\frac{3}{2}}} \frac{1}{T^{\frac{1-\alpha}{2}}}\right),
\end{gather*}
where $\OM(\cdot)$ hides constant factors on $\alpha$.
Theorem~\ref{thm:ave_rate} gives
\[
\frac{1}{T} \sum_{t=0}^T \EB\|\DDelta_t\|_{\infty}^2
\le \OM \left(
\frac{ \ln D}{(1-\gamma)^{3+\frac{1}{1-\alpha}}} \frac{1}{T} + \frac{\ln D}{(1-\gamma)^4} \frac{1}{T^{\alpha}}
\right).
\]
Hence, combining these bounds with~\eqref{eq:T0},~\eqref{eq:T1},~\eqref{eq:T2},~\eqref{eq:T3}, and~\eqref{eq:T4}, we have
\begin{align*}
\EB\|\BDelta_T\|_{\infty}
= 
\OM&\left(
\frac{1}{(1-\gamma)^{1+\frac{1}{1-\alpha}}T} +
\sqrt{\|\diag(\VQ)\|_{\infty}}\sqrt{\frac{\ln D}{T}}
+ \frac{\ln(D)}{(1-\gamma)^2T} \right. \\
& \qquad
+\sqrt{\frac{\ln D}{(1-\gamma)^2T}}
\left(\frac{1}{(1-\gamma)^{1+\frac{1}{1-\alpha}}} \frac{1}{\sqrt{T}} 
+  \frac{1}{(1-\gamma)^2}  \frac{1}{T^{1-\alpha}} 
+\frac{1}{(1-\gamma)^{\frac{3}{2}}} \frac{1}{T^{\frac{1-\alpha}{2}}} \right)  \\
& \qquad + 
\frac{\gamma}{(1-\gamma)^{\frac{1}{1-\alpha}}} \sqrt{\frac{\ln(DT)}{T}} \left(  
\frac{ \sqrt{\ln D}}{(1-\gamma)^{1.5+\frac{1}{2(1-\alpha)}}} \frac{1}{\sqrt{T}} +   \frac{\sqrt{\ln D}}{(1-\gamma)^2} \frac{1}{T^{\frac{\alpha}{2}}} \right)\\
&\qquad +  \left. \frac{\gamma}{(1-\gamma)^{1+\frac{1}{1-\alpha}}} \frac{
	\ln DT }{T} +
\frac{\gamma L}{(1-\gamma)^{\frac{1}{1-\alpha}}} \left(
\frac{ \ln D}{(1-\gamma)^{3+\frac{1}{1-\alpha}}} \frac{1}{T} +   \frac{\ln D}{(1-\gamma)^4} \frac{1}{T^{\alpha}}
\right)
\right)\\
= \OM &\left(
 \sqrt{\|\diag(\VQ)\|_{\infty}}\sqrt{\frac{\ln D}{T}} + 
\frac{\sqrt{\ln D}}{(1-\gamma)^{3}}\frac{1}{T^{1-\frac{\alpha}{2}}} \right)
+ \widetilde{\OM}  \left(
\frac{L}{(1-\gamma)^{3+\frac{2}{1-\alpha}}} \frac{1}{T}
+ \frac{\gamma L}{(1-\gamma)^{4+\frac{1}{1-\alpha
}}} \frac{1}{T^\alpha}
\right),
\end{align*}
where $\widetilde{\OM}(\cdot)$ hides polynomial dependence on logarithmic terms, namely $\ln D $ and $\ln T$.
Here we use $\|\Var(\Z)\|_{\infty} \le \frac{1}{(1-\gamma)^2}$, $
T^{-\frac{1+\alpha}{2}}\le T^{-\alpha}$ to simplify the final inequality.

\subsection{A Useful Inequality}
The following is a useful inequality which will be used frequently in the subsequent proof.
\begin{lem}
	\label{lem:V}
	For any matrices $\A, \V$ with a compatible order, we have
	\begin{equation}
	\label{eq:infty-max-infty}
	\|\diag(\A \V \A^\top)\|_{\infty} \le  \|\V\|_{\max} \|\A\|_{\infty}^2,
	\end{equation}
	where $\|\V\|_{\max} = \max_{i, k} |\V(i, k)|$.
\end{lem}
\begin{proof}[Proof of Lemma~\ref{lem:V}]
	For any diagonal entry $i$, it follows that
	\begin{align*}
	|(\A \V \A^\top)(i, i)|
	&= \bigg|\sum_{l}(\A \V)(i, l) \A(i, l)\bigg|
	= \bigg|\sum_{l}\sum_{k}\A(i, k) \V(k, l) \A(i, l)\bigg| \\
	&\le \sum_{l}\sum_{k}|\A(i, k)| \cdot | \V(k, l)| \cdot |\A(i, l)| \\
	&\le \|\V\|_{\max} \sum_{k}|\A(i, k)| \cdot\sum_{l}  |\A(i, l)| \\
	& \le \|\V\|_{\max} \|\A\|_{\infty}^2.
	\end{align*}
\end{proof}

\subsection{Proof of Lemma~\ref{lem:T3}}
\label{proof:T3}
\begin{proof}[Proof of Lemma~\ref{lem:T3}]
	Recall that $\TM_3 =  \frac{\gamma}{T}  \sum_{j=1}^T\A_j^T( \PP_j - \PP) (\V_{j-1} - \V^*)$ and $\FM_j$ is the $\sigma$-field generated by all randomness before (and including) iteration $j$.
	We will apply Lemma~\ref{lem:freedman} to prove our lemma.
	Using the notation defined therein, we set $\X_j = \frac{\gamma}{T} ( \PP_j - \PP) (\V_{j-1} - \V^*)$ and $\B_j =\A_j^T$.
	Clearly, $\{\X_j\}_{j \ge 0}$ is a martingale difference sequence since $\EB[\X_j|\FM_{j-1}] = \frac{\gamma}{T} \EB[ \PP_j - \PP|\FM_{j-1}] (\V_{j-1} - \V^*) = \0$.
	As a result,  $X = \frac{4\gamma}{T(1-\gamma)}, B = C_0, D = |\SM \times \AM|$ and $\U_j = \Var[\X_j|\FM_{j-1}]$.\footnote{To distinguish $\Var[\X_j|\FM_{j-1}]$ and the value function $\V_j$, we use $\U_j$ to denote the conditional variance.}
	
	Recall that $\W_T = \diag(\sum_{j=1}^T \B_j \U_j \B_j^\top)$.
	To upper bound $\EB\|\W_T \|_{\infty}$, we aim to find a upper bound for $\|\W_T\|_{\infty}$.
	We first note that
	\begin{align*}
	\|\W_T\|_{\infty} 
	= \left\|\diag\left(\sum_{j=1}^T \B_j \U_j \B_j^\top\right)\right\|_{\infty} 
	\le\sum_{j=1}^T   \left\|\diag\left(\B_j \U_j \B_j^\top\right)\right\|_{\infty}  \le \sum_{j=1}^T   \left\|\B_j \right\|_{\infty}^2 \|\U_j \|_{\max}.
	\end{align*}
	Here the last inequality uses~\eqref{eq:infty-max-infty}.
	To bound $\|\U_j \|_{\max}$, 
	we find that for any $i \neq k$, $\U_{j}(i, k) =\EB[ \e_i^\top \X_j \X_j^\top\e_k |\FM_{j-1}]  = 0$ due to each coordinate of $\X_j$ are independent conditioning on $\FM_{j-1}$.
	Hence,
	\begin{align*}
	\|\U_j \|_{\max} 
	&= \max_{i,k} |\U_{j}(i, k)| 
	= \max_{i} |\U_{j}(i, i)| 
	= \left\| \EB[ \diag( \X_j \X_j^\top)|\FM_{j-1}] \right\|_{\infty} \\
	&\le \EB\left[  \left\| \diag( \X_j \X_j^\top)\right\|_{\infty} \bigg|\FM_{j-1}\right] 
	\overset{(a)}{\le} \EB[  \left\| \X_j \right\|_{\infty}^2 |\FM_{j-1}] \\
	&= \frac{\gamma^2}{T^2}   \EB[  \left\| ( \PP_j - \PP) (\V_{j-1} - \V^*)\right\|_{\infty}^2 |\FM_{j-1}] \\
	&\le \frac{\gamma^2}{T^2} \|\V_{j-1} - \V^*\|_{\infty}^2  
	\EB \| \PP_j - \PP\|_{\infty}^2 \overset{(b)}{\le} \frac{4\gamma^2}{T^2} \|\V_{j-1} - \V^*\|_{\infty}^2,
	\end{align*}
	where $(a)$ again uses~\eqref{eq:infty-max-infty} and $(b)$ uses $ \| \PP_j - \PP\|_{\infty} \le  \| \PP_j\|_{\infty} + \| \PP\|_{\infty} =2$.
	
	Putting the pieces together, we have
	\begin{align*}
	\|\W_T\|_{\infty}  
	\le \frac{4\gamma^2}{T} \sum_{j=1}^T   \left\|\B_j \right\|_{\infty}^2 \|\V_{j-1} - \V^*\|_{\infty}^2  
	\le \frac{4\gamma^2C_0^2 }{T^2} \sum_{j=1}^T  \|\V_{j-1} - \V^*\|_{\infty}^2,
	\end{align*}
	where we use $\sup_{j} \|\B_j\|_{\infty}  \le B = \frac{C_0}{1-\gamma}$.
	The rest follows from~\eqref{eq:general} in Lemma~\ref{lem:freedman} by plugging the corresponding $B, X, D$ and $\sigma^2$ and the inequality $\|\V_{j-1} - \V^*\|_{\infty} \le \|\Q_{j-1} - \Q^*\|_{\infty} = \|\DDelta_{j-1}\|_{\infty}$.
\end{proof}

\section{PROOF OF THE INFORMATION-THEORETIC LOWER BOUND}
\label{proof:infolb}

\subsection{Proof of Theorem~\ref{thm:infolb}}

The semiparametric model $\PM_{\theta} \in \PM_P \times \PM_R$ described in Section~\ref{sec:infolb} is described through an infinite-dimensional parameter $\theta = (\PP, R)$, which is partitioned into a finite-dimensional parameter $\PP \in \RB^{D \times S}$ and an infinite-dimensional parameter $R$.
The reason why $R$ is infinite dimensional is because we don't specify the probability model of each $R(s, a)$, which is equivalent to considering the class of all p.d.f.'s on the interval $[0, 1]$, which is infinite dimensional.
The parameter of interest is a smooth function of $\theta$, denoted by $\beta(\theta) = \Q^* \in \RB^D$.
To compute the semiparametric Cramer-Rao lower bound (see Definition 4.7 of~\citep{vermeulen2011semiparametric}), we need to compute
\begin{equation}
\label{eq:bound}
\sup_{\PM_\gamma \subset \PM} \GammaSym(\gamma_0) \I(\gamma_0)^{-1}\GammaSym^\top(\gamma_0),
\end{equation}
where $\PM_\gamma$ is any parametric submodel containing the truth, i.e., $\PM_{\gamma_0} = \PM_{\theta}$.
Hence, under one kind of parameterization, the true model $\PM_\theta$ can be recovered by setting $\gamma=\gamma_0$ in the parametric submodel $\PM_{\gamma}$.
Here, $\GammaSym(\gamma_0) = \frac{\partial \Q^*}{\partial \gamma}|_{\gamma = \gamma_0}$ is the score and $\I(\gamma_0)$ is the corresponding Fisher information matrix.
Let $\gamma_0(R)$ (resp. $\gamma_0(\PP)$) be the finite-dimensional part of $\gamma_0$ that relates with $R$ (resp. $\PP$).
Due to the (variational) independence between $\PP$ and $R$, $\gamma_0(\PP)$ doesn't intersect with $\gamma_0(R)$.
Hence,~\eqref{eq:bound} can be divided into two parts
\begin{align*}
&\sup_{\PM_\gamma(\PP) \subset \PM_P}
\GammaSym(\gamma_0(\PP)) \I(\gamma_0(\PP))^{-1}\GammaSym^\top(\gamma_0(\PP)) + \sup_{\PM_\gamma(R) \subset \PM_R}
\GammaSym(\gamma_0(R)) \I(\gamma_0(R))^{-1}\GammaSym^\top(\gamma_0(R))\\
\overset{(*)}{=}&\GammaSym(\PP) \I(\PP)^{-1}\GammaSym^\top(\PP)
+ \sup_{\PM_\gamma(R) \subset \PM_R}
\GammaSym(\gamma_0(R)) \I(\gamma_0(R))^{-1}\GammaSym^\top(\gamma_0(R)),
\end{align*}
where $\PM_\gamma(R)$ (resp. $\PM_\gamma(\PP)$) denotes the parametric submodel depending only on $R$ (resp. $\PP$).
The equality $(*)$ follows because in the case the parametric model $\PM_P$ is the full model and the parametric Cramer-Rao lower bound is not affected by any one-to-one reparameterization.
Here, $\GammaSym(\PP) = \frac{\partial \Q^*}{\partial \PP}$ and $\I(\PP)$ is the (constrained) information matrix.

In the following, we will first handle the parametric part (i.e., the transition kernel $P$) by computing the (constrained) information matrix and then cope with the nonparametric part (i.e., the random reward $R$) by using semiparametric tools.
Combining the two parts together, we find that the semiparametric efficiency bound is
\begin{align*}
&\frac{1}{T}\cdot(\I-\gamma\PP^{\pi^*})^{-1}\Var(\gamma\PP_j \V^*)(\I-\gamma\PP^{\pi^*})^{-\top}
+
\frac{1}{T}\cdot(\I-\gamma\PP^{\pi^*})^{-1}\Var(\rr_j)(\I-\gamma\PP^{\pi^*})^{-\top}\\
=&\frac{1}{T}\cdot (\I-\gamma\PP^{\pi^*})^{-1}\Var(\Z_j)(\I-\gamma\PP^{\pi^*})^{-\top},
\end{align*}
using the notation $\Z_j = \rr_j  + \gamma\PP_j \V^*$ and the independence of $\rr_j$ and $\PP_j$.

\subsubsection{Parametric Part}
We first investigate the Cramer-Rao lower bound for estimating $\Q^*$ using samples from $\{\PP_t\}_{t \in [T]} $ whose distribution is determined by  $\PP\in\PM$ with $\PM$ defined in~\eqref{eq:PM}.
Note that $\PP \in \PM$ is linearly constrained, i.e.,
\[
\h(\PP)=0,
\]
where $\h: \RB^{D \times S} \to \RB^D$ with its $(\tilde{s},\tilde{a})$-th coordinate of $\h$ given by
\begin{align}
\label{eq:h}
h_{\tilde{s},\tilde{a}}(\PP)=\sum_{s,a,s'}P(s'|s,a)\mathsf{1}_{  \{ (s,a)=(\tilde{s},\tilde{a}) \}  } -1.
\end{align}
Hence, we encounter the Cramer-Rao lower bound for constrained parameters.
Let $\C_T(\PP)$ is the inverse Fisher information matrix using $T$ i.i.d.\ samples under the constraint $\h(\PP)=0$. 
Hence, $ \C_T(\PP) = \frac{\C_1(\PP)}{T}$ and the constrained Cramer-Rao lower bound~\citep{moore2010theory} is
\begin{align}
\label{eq:help2}
 \GammaSym(\PP) \I(\PP)^{-1}\GammaSym^\top(\PP)=
    \left(\frac{\partial \Q^*}{\partial \PP}\right)^\top\C_T(\PP)\frac{\partial \Q^*}{\partial \PP}=\frac{1}{T}\cdot\left(\frac{\partial \Q^*}{\partial \PP}\right)^\top\C_1(\PP)\frac{\partial \Q^*}{\partial \PP},
\end{align}
where $\frac{\partial \Q^*}{\partial \PP}$ is the partial derivatives computed ignoring the linear constraint $\h(\PP) = 0$.

To give a precise formulation of the bound~\eqref{eq:help2}, we first compute $\frac{\partial \Q^*}{\partial \PP}$. 
\begin{lem}
	\label{lem:Q_diff}
	Under Assumption~\ref{asmp:gap}, $\Q^*$ is differentiable w.r.t.\ $\PP$ with the partial derivatives given by
	\begin{align*}
	\frac{\partial Q^*(s,a)}{\partial P(s'|\tilde{s},\tilde{a})}=\gamma V^*(s')\cdot(\I-\gamma\PP^{\pi^*})^{-1}((s,a),(\tilde{s},\tilde{a})).
	\end{align*}
\end{lem}
We then compute $\C_1(\PP)$ via the following lemma.

\begin{lem}
	\label{lem:ccrb}
	The $(s, a)$-th row of the random matrix $\PP_t$ is given by $P_t(s'|s,a)=\mathsf{1}_{ \{ s_t(s, a)=s'\}  }$ where $s_t(s, a)$ is the generated next-state from $(s, a)$ at iteration $t$ with probability given as the $(s,a)$-th row of $\PP$.
	Hence $\PP = \EB \PP_t$ and $\PP$ belongs to the following parametric space 
	\[
	\PM=\left\{\PP \in \RB^{D \times S}: \hspace{2pt} P(s'|s,a) \ge 0  \hspace{2pt}\text{for all $(s,a,s')$}  \ \text{and} \ \h(\PP)=\0 \right\},
	\]
	with $\h$ defined in~\eqref{eq:h}.
	The constrained inverse Fisher information matrix $\C_1(\PP)$ is
	\begin{align*}
	\C_1(\PP)=\diag\left(\left\{\diag(P(\cdot|s,a))-P(\cdot|s,a)P(\cdot|s,a)^\top)\right\}_{(s,a)}\right).
	\end{align*}
\end{lem}

By Lemma~\ref{lem:Q_diff} and \ref{lem:ccrb}, we have
\begin{align*}
    \left(\frac{\partial \Q^*}{\partial \PP}\right)^\top\C_1(\PP)\frac{\partial \Q^*}{\partial \PP}((s,a),(\bar{s},\bar{a}))=&\sum_{(\tilde{s},\tilde{a})}\gamma^2 \G^{-1}((s,a),(\tilde{s},\tilde{a}))\G^{-1}((\bar{s},\bar{a}),(\tilde{s},\tilde{a}))\\
    &\cdot\left(\sum_{\tilde{s}'}V^*(\tilde{s}')^2P(\tilde{s}'|\tilde{s},\tilde{a})-(\sum_{\tilde{s}'}V^*(\tilde{s}')P(\tilde{s}'|\tilde{s},\tilde{a}))^2\right).
\end{align*}
The Cramer-Rao lower bound is thus equal to
\begin{align*}
    \left(\frac{\partial \Q^*}{\partial \PP}\right)^\top\C_T(\PP)\frac{\partial \Q^*}{\partial \PP}=T\cdot(\I-\gamma\PP^{\pi^*})^{-1}\Var(\gamma\PP_j \V^*)(\I-\gamma\PP^{\pi^*})^{-\top}.
\end{align*}
At the end of this part, we provide the deferred proof for Lemma~\ref{lem:Q_diff} and~\ref{lem:ccrb}.
\begin{proof}[Proof of Lemma~\ref{lem:Q_diff}]
    Notice that $\Q^*=\R+\gamma\PP\V^*$.
    Then by the chain rule, we have
    \begin{align*}
        &\frac{\partial Q^*(s,a)}{\partial P(s'|s,a)}=\gamma V^*(s')+\gamma\sum_{s_1}P(s_1|s,a)\frac{\partial V^*(s_1)}{\partial P(s'|s,a)},\\
        &\frac{\partial Q^*(s,a)}{\partial P(s'|\tilde{s},\tilde{a})}=\gamma\sum_{s_1}P(s_1|\tilde{s},\tilde{a})\frac{\partial V^*(s_1)}{\partial P(s'|\tilde{s},\tilde{a})}
        \ \text{for any} \ (s,a) \neq (\tilde{s},\tilde{a}).
    \end{align*}
    Assumption~\ref{asmp:gap} implies the optimal policy $\pi^*$ is unique.
    Hence, using $V^*(s_1)=\max_a Q^*(s_1,a) = Q^*(s_1, \pi^*(s_1))$, we have
        \begin{align*}
    \frac{\partial V^*(s_1)}{\partial P(s'|s,a)}=\frac{\partial Q^*(s_1,\pi^*(s_1))}{\partial P(s'|s,a)}.
    \end{align*}
    Notice that $\PP^*( (s, a), (\tilde{s},\tilde{a}) ) = P(\tilde{s}|s, a)\mathsf{1}_{ \{ \tilde{a} = \pi^*( \tilde{s}) \} }$.
    Putting all the pieces together and solving $\{  \frac{\partial Q^*(s,a)}{\partial P(s'|\tilde{s},\tilde{a})} \}_{s, a, s',\tilde{s},\tilde{a} }$ from the linear system, we have
    \begin{align*}
        \frac{\partial Q^*(s,a)}{\partial P(s'|\tilde{s},\tilde{a})}=\gamma V^*(s')\cdot(\I-\gamma\PP^*)^{-1}((s,a),(\tilde{s},\tilde{a})).
    \end{align*}
\end{proof}

\begin{proof}[Proof of Lemma~\ref{lem:ccrb}]
    We write our the log-likelihood of sample $\PP_t$ as
    \begin{align*}
        \log f_{\PP}(\PP_t)=\sum_{s,a,s'}\mathsf{1}_{ \{s_t(s, a)=s'\}  }\log P(s'|s,a),
    \end{align*}
    which implies $ \frac{\partial}{\partial\PP}\log f_{\PP}(\PP_t) \in \RB^{S^2A}$ with the $(s, a, s')$-th entry given by
    \begin{equation}
    \label{eq:log-prob-P}
        \frac{\partial \log f_{\PP}(\PP_t)}{\partial P(s'|s, a)} = \frac{\mathsf{1}_{ \{s_t(s, a)=s'\}  }}{P(s'|s, a)}.
    \end{equation}
    By definition of the Fisher information matrix, we have
    \begin{align*}
       \I_1(\PP)=\EB\left\{\frac{\partial}{\partial\PP}\log f_{\PP}(\PP_t)\left[\frac{\partial}{\partial\PP}\log f_{\PP}(\PP_t)\right]^\top\right\} \in \RB^{S^2A \times S^2A},
    \end{align*}
    which implies
    \begin{align*}
    \I_1(\PP)((s,a,s'), (\tilde{s},\tilde{a},\tilde{s}'))
    =  
    \left\{ 
    \begin{array}{ll}
\frac{\mathsf{1}_{ \{ s'=\tilde{s}' \}  }}{P(s'|s,a)} & \text{if} \ (s,a)=(\tilde{s},\tilde{a}), \\
1 & \text{if} \ (s,a)\neq(\tilde{s},\tilde{a}). \\
    \end{array}
     \right.
    \end{align*}
    By definition of $\h(\PP)$, we rearrange $\h(\PP)$ into an $S^2A \times SA$ matrix given by
    \begin{align*}
        \H(\PP)((s,a,s'),(\tilde{s},\tilde{a})):=\frac{\partial h_{\tilde{s},\tilde{a}}(\PP)}{\partial P(s'|s,a)}=\mathsf{1}_{\{ (\tilde{s},\tilde{a})=(s,a) \} }.
    \end{align*}
    Let $\U(\PP)\in\RB^{S^2A\times(S^2A-SA)}$ be the orthogonal matrix whose column space is the orthogonal complement of the column space of $\H(\PP)$, which stands for $\H(\PP)^\top \U(\PP)=\0$ and $\U(\PP)^\top\U(\PP)=\I$. 
    Using results in~\citep{moore2010theory}, the constrained CRLB is
    \begin{align*}
        \C_1(\PP)=\U(\PP)\left(\U(\PP)^\top\I_1(\PP)\U(\PP)\right)^{-1}\U(\PP)^\top.
    \end{align*}
    We define an auxiliary matrix $\X\in\RB^{SA\times S^2A}$ satisfying
    \begin{align*}
        \X((s,a),(\tilde{s},\tilde{a},\tilde{s}'))=-\frac{1}{2}\cdot\mathsf{1}_{  \{ (s,a)\not=(\tilde{s},\tilde{a}) \}   } .
    \end{align*}
    By $\H(\PP)^\top \U(\PP)=\0$, we have
    \begin{align*}
        \C_1(\PP)&=\U(\PP)\left(\U(\PP)^\top(\H(\PP)\X+\I_1(\PP)+\X^\top\U(\PP)^\top)\U(\PP)\right)^{-1}\U(\PP)^\top\\
        &:=\U(\PP)\left(\U(\PP)^\top\D(\PP)\U(\PP)\right)^{-1}\U(\PP)^\top,
    \end{align*}
    where $\D(\PP)((s,a,s'),(s,a,s'))=1/P(s'|s,a)$ and takes value $0$ elsewhere. 
    Now we reformulate $\D(\PP)$ as a block diagonal matrix  $\D(\PP)=\diag(\{\D_{(s,a)}\}_{(s,a)}):=\diag(\{1/P(\cdot|s,a)\}_{(s,a)})$ where $\D_{(s, a)}$ is a diagonal matrix with $\D_{(s, a)}(s', s') = 1/P(s'|s,a)$.
    Similarly, we have $\H(\PP)=\diag(\{\1_{S}\}_{(s,a)})$, where $\1_{S}$ is an all-1 vector with dimension $S$, and $\U(\PP)=\diag(\{\U_{(s,a)}\}_{(s,a)})$, where $\U_{(s,a)}\in\RB^{S \times S-1}$ satisfying $\U_{(s,a)}^\top\1_{S}=\0$. 
    In this way, $\C_1(\PP)$ has a equivalent block diagonal formulation
    \begin{align*}
        \C_1(\PP)=\diag\left(\left\{\U_{(s,a)}\left(\U_{(s,a)}^\top\D_{(s,a)}\U_{(s,a)}\right)^{-1}\U_{(s,a)}^\top\right\}_{(s,a)}\right).
    \end{align*}
    For each block $(s,a)$ of $\C_1(\PP)$, the submatrix is exactly the constrained Cramer-Rao bound of a multinomial distribution $\PP_{s,a} =\{P(\cdot|s,a)\}$, which is equal to $\diag(\PP_{s,a})-\PP_{s,a}\PP_{s,a}^\top$.
    Therefore,
    \begin{align*}
        \C_1(\PP)=\diag\left(\left\{\diag(P(\cdot|s,a))-P(\cdot|s,a)P(\cdot|s,a)^\top)\right\}_{(s,a)}\right).
    \end{align*}
\end{proof}

\subsubsection{Nonparametric Part}
Next, we move on discussing the efficiency on rewards. 
Unlike $\PP_t$ that is generated according to a parametric model, the generating mechanism of $\rr_t$ can be arbitrary. 
In other words, a finite dimensional parametric space is not enough to cover the possible distributions of $\rr_t$. 
Thus, semiparametric theory is needed here. 
Fortunately, our interest parameter $\Q^*=(\I-\gamma\PP^{\pi^*})^{-1}\rr$ is linear in $\rr:=\EB\rr_t$, implying only the expectation of $\rr_t$ matters.
In semiparametric theory~\citep{van2000asymptotic,tsiatis2006semiparametric}, the efficienct influence function for mean estimation is exatly the random variable minus its expectation. 
Lemma~\ref{lem:r_eif} shows it is still true in our case.

\begin{lem}
    \label{lem:r_eif}
    Let Assumption~\ref{asmp:gap} hold.
    Given a random sample $\rr_t$, the most efficient influence function for estimating $\Q^*(s,a)$ for any $(s, a)$ is
    \begin{align*}
        \phi(s,a)=(\I-\gamma\PP^{\pi^*})^{-1}(\rr_t-\rr) (s,a),
    \end{align*}
    where $\rr=\EB\rr_t$.
    Hence, the semiparametric efficiency bound of estimating $\Q^*$ with $\{\rr_t\}_{t\in [T]}$ is
    \begin{align*}
    \sup_{\PM_\gamma(R) \subset \PM_R}
\GammaSym(\gamma_0(R)) \I(\gamma_0(R))^{-1}\GammaSym^\top(\gamma_0(R)) =
    \frac{1}{T}\cdot (\I-\gamma\PP^{\pi^*})^{-1}\Var(\rr_t)(\I-\gamma\PP^{\pi^*})^{-1}.
    \end{align*}
\end{lem}

\begin{proof}[Proof of Lemma~\ref{lem:r_eif}]
    As $r_t(s,a)$ are independent with different $(s',a')$ pairs, we can only consider randomness of one pair $(s,a)$.
    
    Firstly, we consider a submodel family $\PM_{R_{\varepsilon}}$ of $\PM_R$ that is parameterized by $\varepsilon$ such that when $\varepsilon =0$, we recover the distribution of $R(s, a)$.
    That is $\PM_{R_{\varepsilon}}=\{R_\varepsilon: \varepsilon\in[-\delta,\delta] \ \text{and} \ R(s, a) = R_\varepsilon(s, a)|_{\varepsilon=0} \} $.
    This can be achieved by manipulating density functions of each $R(s, a)$.
    It is clear that $\PM_{R_{\varepsilon}}$ is a parametric family on rewards and we can make use of results in parametric statistics for our purpose.
    By definition, we have for $(s,a)$,
    \begin{align*}
        \frac{\partial Q^*(s,a)}{\partial\varepsilon}\bigg|_{\varepsilon=0}&=\frac{\partial}{\partial\varepsilon}\left(\EB r_t(s,a)+\gamma\sum_{s'}P(s'|s,a)Q^*(s',\pi^*(s'))\right)\bigg|_{\varepsilon=0}\\
        &=\frac{\partial\EB r_t(s,a)}{\partial\varepsilon}\bigg|_{\varepsilon=0}+\gamma\sum_{s'}P(s'|s,a)\frac{\partial Q^*(s',\pi^*(s'))}{\partial \varepsilon}\bigg|_{\varepsilon=0}.
    \end{align*}
    For any $(\tilde{s},\tilde{a})\not=(s,a)$, we have
    \begin{align*}
        \frac{\partial Q^*(\tilde{s},\tilde{a})}{\partial\varepsilon}\bigg|_{\varepsilon=0}=\gamma\sum_{s'}P(s'|\tilde{s},\tilde{a})\frac{\partial Q^*(s',\pi^*(s'))}{\partial \varepsilon}\bigg|_{\varepsilon=0}.
    \end{align*}
    Recursively expanding the above terms like what we have done in Lemma~\ref{lem:Q_diff}, we have
    \begin{align*}
        \frac{\partial Q^*(\tilde{s},\tilde{a})}{\partial \varepsilon}\bigg|_{\varepsilon=0}=\frac{\partial\EB r_t(s,a)}{\partial\varepsilon}\bigg|_{\varepsilon=0}\cdot(\I-\gamma\PP^{\pi^*})^{-1}((\tilde{s},\tilde{a}),(s,a)).
    \end{align*}
    Let $F_\varepsilon$ denote the cumulative distribution function of $R_\varepsilon(s, a)$.
    Then we have
    \begin{align*}
        \frac{\partial\EB r_t(s,a)}{\partial\varepsilon}\bigg|_{\varepsilon=0}&=\int r_t(s,a)\frac{\partial}{\partial\varepsilon}dF_\varepsilon\bigg|_{\varepsilon=0}  \\
        &=\int(r_t(s,a)-r(s,a))\frac{\partial}{\partial\varepsilon}\log dF_\varepsilon \bigg|_{\varepsilon=0}dF_0,
    \end{align*}
    where $r(s,a)=\EB r_t(s,a)$ and $\frac{\partial}{\partial\varepsilon}\log dF_\varepsilon$ is the score function. 
    Therefore,
    \begin{equation}
    \label{eq:partial-var}
    \frac{\partial Q^*(\tilde{s},\tilde{a})}{\partial \varepsilon}\bigg|_{\varepsilon=0}=
    \int  \phi(\tilde{s},\tilde{a})  \frac{\partial}{\partial\varepsilon}\log dF_\varepsilon \bigg|_{\varepsilon=0}dF_0,
    \end{equation}
    where    
    \begin{align*}
    \phi(\tilde{s},\tilde{a})=(r_t-r)(s,a)\cdot(\I-\gamma\PP^{\pi^*})^{-1}((\tilde{s},\tilde{a}),(s,a)).
    \end{align*}
    Since the parametric submodel family $\mathcal{R}_\varepsilon$ is arbitrary, we conclude that the efficient influence function of $Q^*(\tilde{s},\tilde{a})$ is $\phi(\tilde{s},\tilde{a})$ by Theorem~2.2 in~\citep{newey1990semiparametric}.
    Finally, as $r_t(s,a)$ is independent with each other $r_t(s',a')$'s, our final result is obtained by summing the above equation over all $(s,a)$.
\end{proof}

\subsection{Proof of Theorem~\ref{thm:RAL}}
\begin{proof}[Proof of Theorem~\ref{thm:RAL}]
	Recall that $\BDelta_T = \frac{1}{T}\sum_{t=1}^T (\Q_T - \Q^*)$.
Combining~\eqref{eq:sandwitch-nonasy},~\eqref{eq:Delta1-decom} and~\eqref{eq:Delta2-decom}, we have
\[
\sqrt{T}(\TM_0 + \TM_1 + \TM_2 + \TM_3)
\le  \BDelta_T^1 \le \sqrt{T} \BDelta_T \le \sqrt{T} \BDelta_T^2
\le \sqrt{T} (\TM_0 + \TM_1 + \TM_2 + \TM_3 + \TM_4),
\]
where the inequality holds coordinate-wise.
In Appendix~\ref{proof:five-term}, we have analyze $\EB\|\TM_i\|_{\infty}$ with explicit upper bounds.
It is easy to verify that $\sqrt{T}\EB\|\TM_i \|= o(1)$ for $i=0, 2, 3, 4$ (see Remark~\ref{rmk:T}).
Hence,
\[
\BDelta_T=  \sqrt{T}\TM_1 + o_{\PB}(1) = \frac{1}{\sqrt{T}} \sum_{t=1}^T (\I - \gamma \PP^{\pi^*})^{-1} \Z_t + o_{\PB}(1)  :=  \frac{1}{\sqrt{T}} \sum_{t=1}^T \Bphi(\rr_t, \PP_t) + o_{\PB}(1),
\]
where $\Z_t = (\rr_t-\rr) + \gamma (\PP_t - \PP) \V^*$ is the Bellman noise at iteration $t$. 
This implies $\bar{\Q}_T$ is asymptotically linear with the influence function $\Bphi(\rr_t, \PP_t):= (\I - \gamma \PP^*)^{-1} \Z_t $.

The remaining issue is to prove regularity.
By definition, a RAL estimator is regular for a semiparametric model $\PM = \PM_P\times \PM_R$ if it is a RAL estimator for every parametric submodel $\PM_{\gamma} = \PM_P \times \PM_{R_{\varepsilon}} \subset \PM$ where $\gamma = (\PP, \varepsilon)$ is the finite-dimensional parameter controlling $\PM_{\gamma}$.
In a parametric submodel $\PM_P \times \PM_{R_{\varepsilon}}$, by Theorem~2.2 in~\citep{newey1990semiparametric}, for the asymptotically linear estimator $\bar{\Q}_T$ of $\Q^*$ which has the influence function
\[
\Bphi(\rr_t, \PP_t) =(\I-\gamma \PP^{\pi^*})^{-1} \left[ (\rr_t-\rr) + \gamma(\PP_t - \PP) \V^*\right],
\]
its regularity is equivalent to the equality
\begin{equation}
\label{eq:reg}
\EB \Bphi(\rr_t, \PP_t) S_{\gamma}^\top(\gamma_0) = \frac{\partial \Q^*}{\partial \gamma} \bigg|_{\gamma = \gamma_0},
\end{equation}
where $S_{\gamma}(\cdot)$ is the score function, $\gamma = (\PP', \varepsilon) \in \PM_P\times [-\delta, \delta]$  is the finite-dimensional parameter and $\gamma_0 = (\PP, 0)$ is the true underlying parameter.
Since $\PP$ and $\varepsilon$ are variationally independent, $S_{\gamma}(\gamma_0) = (S_{\PP}(\gamma_0), S_{\varepsilon}(\gamma_0))$.

\paragraph{For the transition kernel $\PP$.}
Since our parametric space $\PM_P$ has a linear constraint, it is not easy to compute the constrained score function.
Hence, for $\PP = \{ P(s'|s, a) \}_{s,a,s'}$, we regard $\{ P(s'|s, a) \}_{s,a,s'\neq s_0}$ as free parameters where $s_0 \in \SM$ is any fixed state and use it as our new parameter.
For a fixed $(s, a)$, once $P(s'|s, a)$ is determined for all $s'\neq s_0$, one can recover $P(s_0|s, a)$ by $P(s_0|s, a) = 1 -\sum_{s' \neq s_0} P(s'|s, a)$.
In this way, each $\{ P(s'|s, a) \}_{s'\neq s_0}$ lies in a open set.
We still denote the set collecting all feasible $\{ P(s'|s, a) \}_{s,a,s'\neq s_0}$ as $\PM$, but readers should remember that current $\PP = \{ P(s'|s, a) \}_{s,a,s'\neq s_0} \in \RB^{SA \times (S-1)}$.
From~\eqref{eq:log-prob-P} and under our new notation of $\PP$, $S_{\PP}(\gamma_0) \in \RB^{SA(S-1)}$ with entries given by 
\[
S_{\PP}(\gamma_0)(s, a, s')= \frac{\mathsf{1}_{ \{s_t(s, a)=s'\}  }}{P(s'|s, a)} - \frac{\mathsf{1}_{ \{s_t(s, a)=s_0\}  }}{P(s_0|s, a)}
\ \text{for any} \ s' \neq s_0.
\]
By Lemma~\ref{lem:Q_diff} and the chain rule, it follows that $ \frac{\partial \Q^*}{\partial \PP} \in\RB^{SA \times SA(S-1)}$ and its $(\tilde{s}, \tilde{a}, s')$-th column is
\begin{equation}
\label{eq:parti-P-col}
\gamma (\I - \gamma \PP^{\pi^*})^{-1} (\cdot, (\tilde{s}, \tilde{a})) \left[ V^*(s') - V^*(s_0) \right].
\end{equation}
Since $(\I - \gamma \PP^{\pi^*})^{-1}$ has a full rank (i.e., $SA$), it is easy to see that $ \frac{\partial \Q^*}{\partial \PP}$ also has rank $SA$ by varying $(\tilde{s}, \tilde{a})$ and fixing $s', s_0$ in~\eqref{eq:parti-P-col}.
On the other hand, the $(\tilde{s}, \tilde{a}, s')$-th column of $\EB \Bphi(\rr_t, \PP_t)  S_{\PP}(\theta_0)^\top$ is
\begin{align*}
(\EB \Bphi(\rr_t, \PP_t) S_{\PP}(\gamma_0)^\top)(\cdot, (\tilde{s}, \tilde{a}, s'))
&=\EB \Bphi(\rr_t, \PP_t) \left[ \frac{\mathsf{1}_{ \{s_t(s, a)=s'\}  }}{P(s'|s, a)} - \frac{\mathsf{1}_{ \{s_t(s, a)=s_0\}  }}{P(s_0|s, a)} \right]
\\
&=\gamma (\I-\gamma \PP^{\pi^*})^{-1}\EB (\PP_t-\PP)\V^* \left[\frac{\mathsf{1}_{ \{s_t(s, a)=s'\}  }}{P(s'|s, a)} - \frac{\mathsf{1}_{ \{s_t(s, a)=s_0\}  }}{P(s_0|s, a)}\right]\\
&= \gamma (\I - \gamma \PP^{\pi^*})^{-1} (\cdot, (\tilde{s}, \tilde{a})) \left[ V^*(s') - V^*(s_0) \right],
\end{align*}
where the last equality uses the following result.
By direct calculation, the $(s, a)$-th entry of $\EB (\PP_t-\PP)\V^* \left[\frac{\mathsf{1}_{ \{s_t(s, a)=s'\}  }}{P(s'|s, a)} - \frac{\mathsf{1}_{ \{s_t(s, a)=s_0\}  }}{P(s_0|s, a)}\right]$ is 0 for all $(s, a) \neq (\tilde{s}, \tilde{a})$ (due to independence) and the $(\tilde{s}, \tilde{a})$-th entry is $V^*(s') - V^*(s_0)$.
Indeed, the $(\tilde{s}, \tilde{a})$-th entry of the mentioned matrix is
\begin{align*}
&\EB \sum_{i \in \SM} (\mathsf{1}_{ \{ s_t(s, a)=i \} } - P(i|s, a) ) V^*(i) \left[\frac{\mathsf{1}_{ \{s_t(s, a)=s'\}  }}{P(s'|s, a)} - \frac{\mathsf{1}_{ \{s_t(s, a)=s_0\}  }}{P(s_0|s, a)}\right]\\
=&\left(V^*(s') - \sum_{i\neq s_0}P(i|s, a) V^*(i)\right) + \sum_{i\in \SM}P(i|s, a) V^*(i) = V^*(s')  -  V^*(s_0).
\end{align*}
Therefore, combining the results for all $(\tilde{s}, \tilde{a}, s') (s' \neq s_0)$, we have
\[
\EB \Bphi(\rr_t, \PP_t) S_{\PP}(\gamma_0)^\top = \frac{\partial \Q^*}{\partial \PP},
\] 
which implies~\eqref{eq:reg} holds for the $\PP$ part.

\paragraph{For the random reward $R$.}
Using the notation in the proof of Lemma~\ref{lem:r_eif}, $S_{\varepsilon}(\gamma_0) = \frac{\partial}{\partial\varepsilon}\log dF_\varepsilon|_{\varepsilon=0}$.
By~\eqref{eq:partial-var}, we have
\[
\frac{\partial \Q^*}{\partial \varepsilon} \bigg|_{\varepsilon=0}= 
\EB (\I-\gamma \PP^{\pi^*})^{-1} (\rr_t-\rr) S_{\varepsilon} (\gamma_0)
= \EB \Bphi(\rr_t, \PP_t) S_{\varepsilon} (\gamma_0)
\]
which implies~\eqref{eq:reg} holds for the $\varepsilon$ part.

$\PM_{R_{\varepsilon}}$ can be arbitrary, so~\eqref{eq:reg} holds for all parametric submodels.
This means $\bar{\Q}_T$ is regular for all parametric submodels and thus is regular for our semiparametric model.
\end{proof}

%\section{Useful Lemmas}
\section{A USEFUL CONCENTRATION INEQUALITY}
\label{proof:freedman}

We introduce a useful concentration inequality in this section. 
It captures the expectation and high probability concentration of a martingale difference sum in terms of $\|\cdot\|_{\infty}$.
It uses a similar idea of Theorem 4 in~\citet{li2021q} and is built on Freedman’s inequality~\citep{freedman1975tail} and the union bound.
\begin{lem}
	\label{lem:freedman}
	Assume $\{\X_j\} \subseteq \RB^d$ are martingale differences adapted to the filtration $\{\FM_j \}_{j \ge 0}$ with zero conditional mean $\EB[\X_j|\FM_{j-1}] = \0$ and finite conditional variance $\V_j = \EB[\X_j\X_j^\top|\FM_{j-1}]$.
	Moreover, assume $\{\X_j\}_{j \ge 0}$ is uniformly bounded, i.e., $\sup_{j} \|\X_j\|_{\infty} \le X$.
	For any sequence of deterministic matrices $\{ \B_j \}_{j \ge 0} \subseteq \RB^{D \times d}$ satisfying $\sup_{j }\|\B_j\|_{\infty} \le B$, we define the weighted sum as 
	\[
	\Y_T = \sum_{j=1}^T \B_j \X_j
	\]
	and let $\W_T = \diag(\sum_{j=1}^T \B_j \V_j (\B_j)^\top)$ be a diagonal matrix that collects conditional quadratic variations.
	Then, it follows that
	\begin{gather}
	\PB\left(\|\Y_T\|_{\infty} \ge \frac{2BX}{3} \ln \frac{2D}{\delta} + \sqrt{ 2\sigma^2\ln \frac{2D}{\delta} }  \ \text{and} \  \|\W_T\|_{\infty} \le \sigma^2\right) \le \delta  \label{eq:high-prob-freedman} \\
	\EB \|\Y_T\|_{\infty}1_{\{ \|\W_T\|_{\infty} \le \sigma^2 \}} \le6\sigma\sqrt{\ln (2D)}  + \frac{4BX}{3}\ln(6D) \label{eq:expe-freedman}  .
	\end{gather}
	Generally, we have
	\begin{equation}
	\label{eq:general}
		\EB \|\Y_T\|_{\infty} \le \frac{8BX}{3} \ln(3DT^2) + 2\sqrt{\EB\|\W_T\|_{\infty}} \sqrt{\ln (2DT^2)}.
	\end{equation}
\end{lem}

\begin{proof}[Proof of Lemma~\ref{lem:freedman}]
	Fixing any $i \in [D]$, we denote the $i$-th row of $\B_j$ as $\bb_j^\top$.
	For simplicity, we omit the dependence of $\bb_j$ on $i$.
	Then the $i$-th coordinate of $\Y_T$ is $\Y_T(i) = \sum_{j=1}^T \bb_j^\top \X_j$ and $\W_T(i, i) = \sum_{j=1}^T  \bb_j^\top \V_j \bb_j$.
	Clearly $\{\bb_j^\top \X_j\}$ is a scalar martingale difference with $\W_T(i, i) =\sum_{j=1}^T \EB[(\bb_j^\top \X_j)^2|\FM_{j-1}] $ the quadratic variation and $|\bb_j^\top \X_j| \le \|\bb_j\|_1 \|\X_j\|_{\infty} \le \|\B_j\|_{\infty}\|\X_j\|_{\infty} = B X$ the uniform upper bound.
	By Freedman’s inequality~\citep{freedman1975tail}, it follows that
	\[
	\PB(|\Y_T(i)| \ge \tau \ \text{and} \ \W_T(i,i) \le \sigma^2 )
	\le 2\exp\left(- \frac{\tau^2/2}{\sigma^2 + BX \tau/3} \right).
	\]
	Then by the union bound, we have 
	\begin{align}
	\label{eq:freedman-0}
	\PB(\|\Y_T\|_{\infty} \ge \tau  \ \text{and} \  \|\W_T\|_{\infty} \le \sigma^2)
	&= \PB\left(\max_{i \in [D]}|\Y_T(i)| \ge \tau
	\ \text{and} \ \max_{i \in [D]}|\W_T(i,i)| \le  \sigma^2\right) \nonumber \\
	&\le \sum_{i \in [D]} \PB\left(|\Y_T(i)| \ge \tau
	\ \text{and} \ \max_{i \in [D]}|\W_T(i,i)| \le  \sigma^2\right) \nonumber \\
	&\le \sum_{i \in [D]} \PB\left(|\Y_T(i)| \ge \tau
	\ \text{and} \ |\W_T(i,i)| \le  \sigma^2\right)  \nonumber \\
	&\le 2D\exp\left(- \frac{\tau^2/2}{\sigma^2 + BX \tau/3} \right).
	\end{align}
	Solving for $\tau$ such that the right-hand side of~\eqref{eq:freedman-0} is equal to $\delta$ gives
	\[
	\tau = \frac{BX}{3} \ln \frac{2D}{\delta} + \sqrt{\left(\frac{BX}{3} \ln \frac{2D}{\delta}\right)^2 + 2\sigma^2\ln \frac{2D}{\delta} }.
	\]
	Using $\sqrt{a+b} \le \sqrt{a} + \sqrt{b}$ gives an upper bound on $\tau$ and provides the high probability result.

	The tail bound of $\|\Y_T\|_{\infty}1_{\{ \|\W_T\|_{\infty} \le \sigma^2 \}}$ has already been derived in~\eqref{eq:freedman-0}.
	For the expectation result, we refer to the conclusion of Exercise 2.8 (a) in~\citep{wainwright2019high} which implies that
	\begin{align*}
	\EB \|\Y_T\|_{\infty}1_{\{ \|\W_T\|_{\infty} \le \sigma^2 \}}
	&\le 2\sigma(\sqrt{\pi} + \sqrt{\ln (2D)})  + \frac{4BX}{3}(1+\ln (2D)) \\
	&\le 6\sigma\sqrt{\ln (2D)}  + \frac{4BX}{3}\ln(6D), 
	\end{align*}
	where the last inequality uses $\sqrt{a} + \sqrt{b} \le \sqrt{2(a+b)}$.
	
	For the last result, we aim to bound $\EB\|\Y_T\|_{\infty}$ without the condition $\|\W_T\|_{\infty} \le \sigma^2$ for some positive number $\sigma$.
	We first assert that there exists a trivial upper bound for $\|\W_T\|_{\infty}$ which is $\|\W_T\|_{\infty} \le TB^2X^2$. 
	This is because 
	\begin{align*}
	\|\W_T\|_{\infty}
	= \left\| \diag\left(\sum_{j=1}^T \B_j \V_j (\B_j)^\top\right)\right\|_{\infty}\le\sum_{j=1}^T \left\| \diag\left( \B_j \V_j (\B_j)^\top\right)\right\|_{\infty}
	\overset{(a)}{\le} \|\V_j\|_{\max} \|\B_j\|_{\infty}^2 \overset{(b)}{\le} T B^2 X^2,
	\end{align*}
	where $(a)$ uses Lemma~\ref{lem:V} and $(b)$ is due to $\|\V_j\|_{\max}  \le X^2$ for all $j \in [T]$.
	However, if we set $\sigma^2 = T B^2 X^2$ in~\eqref{eq:expe-freedman}, the resulting expectation bound of $\EB\|\Y_T\|_{\infty}$ has a poor dependence on $T$.
	
	To refine the dependence, we adapt and modify the argument of Theorem 4 in~\cite{li2021q}.
	For any positive integer $K$, we define 
	\[
	\HM_K = \left\{ \|\Y_T\|_{\infty} \ge \frac{2BX}{3} \ln \frac{2DK}{\delta} + \sqrt{ 4\max\left\{ \|\W_T\|_{\infty},\frac{TB^2 X^2}{2^K} \right\}\ln \frac{2DK}{\delta} }   \right\}
	\]
	and claim that we have $\PB(\HM_K) \le \delta$.
	We observe that the event $\HM_K$ is contained within the union of the following $K$ events: $ \HM_K \subseteq \cup_{k \in [K]} \mathcal{B}_k$ where for $0 \le k < K$, $\mathcal{B}_k$ is defined to be
	\begin{gather*}
	\mathcal{B}_k
	=  \left\{ \|\Y_T\|_{\infty} \ge \frac{2BX}{3} \ln \frac{2DK}{\delta} + \sqrt{ 2\frac{T B^2X^2}{2^{k-1}}\ln \frac{2DT}{\delta} }
	\ \text{and} \   \frac{T B^2X^2}{2^{k}}  \le \|\W_T\|_{\infty} \le  \frac{T B^2X^2}{2^{k-1}}    \right\}\\
	\mathcal{B}_K
	=  \left\{ \|\Y_T\|_{\infty} \ge \frac{2BX}{3} \ln \frac{2DK}{\delta} + \sqrt{ 2\frac{T B^2X^2}{2^{K-1}}\ln \frac{2DT}{\delta} }
	\ \text{and} \   \|\W_T\|_{\infty} \le  \frac{T B^2X^2}{2^{K-1}}    \right\}.
	\end{gather*}
	Invoking~\eqref{eq:high-prob-freedman} with a proper $\sigma^2 = \frac{T B^2X^2}{2^{k-1}} $ and $\delta = \frac{\delta}{K}$, we have $\PB(\mathcal{B}_k) \le \frac{\delta}{K}$ for all $k \in [K]$.
	Taken this result together with the union bound gives $\PB(\HM_K) \le \sum_{k \in [K]}\PB(\mathcal{B}_k) \le \delta$.
	Then we have
	\begin{align*}
	\EB\|\Y_T\|_{\infty}
	&=	\EB\|\Y_T\|_{\infty}1_{\HM_K}  + \EB\|\Y_T\|_{\infty}1_{\HM_K^c}\\
	&\overset{(a)}{\le}   TBX \PB(\HM_K) + 
	\EB \left[\frac{2BX}{3} \ln \frac{2DK}{\delta} + \sqrt{ 4\max\left\{ \|\W_T\|_{\infty},\frac{TB^2 X^2}{2^K} \right\}\ln \frac{2DK}{\delta} } \right]\\
	&\overset{(b)}{\le} BX + \frac{2BX}{3} \ln(2DT^2) + 2\EB \sqrt{ \max\left\{ \|\W_T\|_{\infty},B^2X^2\right\}\ln (2DT^2)}\\
	&\overset{(c)}{\le} BX + \frac{8BX}{3} \ln(2DT^2) + 2\EB \sqrt{ \|\W_T\|_{\infty}\ln (2DT^2)}\\
	&\overset{(d)}{\le} \frac{8BX}{3} \ln(3DT^2) + 2\sqrt{\EB\|\W_T\|_{\infty}} \sqrt{\ln (2DT^2)},
	\end{align*}
	where $(a)$ uses $\|\Y_T\|_{\infty} \le TBX$, $(b)$ follows by setting $\delta = \frac{1}{T}$ and $K = \lceil \log_2 T \rceil \le T$, $(c)$ uses $\sqrt{a+b} \le \sqrt{a} + \sqrt{b}$, and $(d)$ follows from Jensen's inequality and $\exp(\frac{3}{8}) \le \frac{3}{2}$.
\end{proof}

\section{PROOF FOR ENTROPY REGULARIZED Q-LEARNING}
\label{proof:entropy}
In this section, we provide the counterpart results for Q-Learning with entropy.
Since the proof is almost similar to that of Q-Learning, we just provide a sketch for simplicity.
Recall that the matrix-form of the update rule is
\[
\TQ_t = (1-\eta_t) \TQ_{t-1} + \eta_t (\rr_t + \gamma \PP_t \LM_\lambda \TQ_{t-1}).
\]
It is easy to show $\LM_{\lambda}$ is a $1$-contraction with respect to $\|\cdot\|_{\infty}$.

\subsection{Convergence Under the General Step Sizes}
\begin{thm}
	\label{thm:general-entropy}
	Under Assumption~\ref{asmp:reward} and using the general step size in Assumption~\ref{asmp:lr}, we have 
	\begin{equation*}
		\lim_{T \to \infty} \frac{1}{\sqrt{T}}\sum_{t=0}^T \EB \|\TQ_t - \Q_{\lambda}^*\|_{\infty}^2 = 0
	\end{equation*}
where $\Q_{\lambda}^*$ is the unique fixed point of the regularized Bellman equation $\Q_\lambda^* = \rr + \gamma \PP  \LM_\lambda \Q_\lambda^*.$
\end{thm}

\begin{proof}[Proof of Theorem~\ref{thm:general-entropy}]
	Denote $\TDe_t = \TQ_t - \Q_{\lambda}^*$ for simplicity.
	We will show that $	\lim_{T \to \infty} \frac{1}{\sqrt{T}}\sum_{t=0}^T \EB \|\TDe_t\|_{\infty}^2 = 0$ for the sequence generated via~\eqref{eq:r-Q-update}.
	Similar to Theorem~\ref{thm:general-Linfty-pw2}, we notice that the update rule satisfies $ \TQ_t 
	= \TQ_{t-1} + \eta_t (\rr + \gamma \PP \LM_\lambda \TQ_{t-1} - \TQ_{t-1} + \varepsi_t)
	$ 
	where $\varepsi_t=\rr_t-\rr + \gamma (\PP_t- \PP) \LM_\lambda \TQ_{t-1} $.
	Hence, $\EB[ \varepsi_t|\FM_t] = \0$ and $\EB[\|\varepsi_t\|_{\infty}^2| \FM_t ] \le 2\EB \|\rr_t-\rr\|_{\infty}^2 + 2 \gamma^2 \EB\|\PP_t-\PP\|_{\infty}^2 \|\LM_\lambda \TQ_{t-1}\|_{\infty}^2:= A + B  \|\TQ_{t-1}\|_{\infty}^2$ with $A = 2\EB \|\rr_t-\rr\|_{\infty}^2, B = 2 \gamma^2 \EB\|\PP_t-\PP\|_{\infty}^2$. 
	By Theorem~\ref{thm:chen}, we arrive the same inequality as~\eqref{eq:help3}.
	Following the same analysis therein, we can show $	\lim_{T \to \infty} \frac{1}{\sqrt{T}}\sum_{t=0}^T \EB \|\TDe_t\|_{\infty}^2 = 0$ under the general step size in Assumption~\ref{asmp:lr}.
\end{proof}

\subsection{Establishment of FCLT in Proof of Theorem~\ref{thm:fclt-entropy}}

\begin{proof}[Proof of Theorem~\ref{thm:fclt-entropy}]
	Since the analysis is almost similar to that in Theorem~\ref{thm:fclt}, we just specify the differences.
	The three-step analysis in Section~\ref{sec:proof} still applies here except that we show only modify the first step.
	
	\paragraph{Similar error decomposition.}
	Let $\TDe_t = \TQ_t - \Q_{\lambda}^*$.
	By the regularized Bellman equation $\Q_\lambda^* = \rr + \gamma \PP  \LM_\lambda \Q_\lambda^*$, it follows that
	\begin{align*}
		\TDe_{t}
		&=(1-\eta_t) \TDe_{t-1}+ \eta_t \left[
		\rr_t + \gamma \PP_t \LM_\lambda \TQ_{t-1}-(\rr + \gamma \PP  \LM_\lambda \Q_\lambda^*)
		\right]\\
		&= (1-\eta_t) \TDe_{t-1}+ \eta_t\left[(\rr_t-\rr)+\gamma( \PP_t\LM_\lambda \TQ_{t-1}-\PP \LM_\lambda \Q_\lambda^*)\right] \\
		&= (1-\eta_t) \TDe_{t-1}+ \eta_t\left[ \TZ_t+\gamma\PP_t(  \LM_\lambda \TQ_{t-1}-\LM_\lambda \TQ_\lambda^*)\right]\\
			&= (1-\eta_t) \TDe_{t-1}+ \eta_t\left[ \TZ_t + \gamma \TZ_t'+\gamma\PP(  \LM_\lambda \TQ_{t-1}-\LM_\lambda \TQ_\lambda^*)\right],
	\end{align*}
	where we use $\TZ_t =  (\rr_t-\rr) + \gamma( \PP_t - \PP) \LM_\lambda \Q_\lambda^*$ is the regularized Bellman noise and $\TZ_t'=(\PP_t-\PP)(  \LM_\lambda \TQ_{t-1}-\LM_\lambda \TQ_\lambda^*)$ (which is still a martingale difference.)
	
To analyze $ \LM_\lambda \TQ_{t-1}-\LM_\lambda \TQ_\lambda^*$, we introduce an intermediate linear operator $\LM_\lambda^{\pi}$, which is defined by
\[
	(\LM_{\lambda}^{\pi}\Q)(s) :=  \EB_{a \sim \pi(\cdot|s)}\left[Q(s, a)  - \lambda \log \pi(a|s) \right],
\]
for a given policy $\pi$ and regularization coefficient $\lambda$.
As a result of notation, $(\LM_\lambda \Q)(\cdot) = \sup_{\pi \in \Pi} (\LM_{\lambda}^{\pi}\Q)(\cdot)$ for all $\Q \in \RB^D$.
We assume $\LM_\lambda \TQ_t = \LM_{\lambda}^{\widetilde{\pi}_t}\TQ_t$ and $\LM_\lambda \TQ_\lambda^* = \LM_{\lambda}^{\pi_\lambda^*}\TQ_\lambda^*$.
Hence, 
\[
\LM_\lambda \TQ_{t-1}-\LM_\lambda \TQ_\lambda^*
= \LM_{\lambda}^{\widetilde{\pi}_{t-1}}\TQ_{t-1}
- \LM_{\lambda}^{\pi_\lambda^*}\TQ_\lambda^*
=(\LM_{\lambda}^{\widetilde{\pi}_{t-1}} - \LM_{\lambda}^{\pi_\lambda^*})\TQ_{t-1} + \PP^{\pi_\lambda^*} \TDe_{t-1}
\]
where the last equation uses $ \LM_{\lambda}^{\pi_\lambda^*}\TQ_{t-1}-\LM_{\lambda}^{\pi_\lambda^*}\TQ_\lambda^* =  \PP^{\pi_\lambda^*} \TDe_{t-1}$ by definition.
Putting pieces together,
\[
\TDe_{t} =  \TA_t \De_{t-1} + \eta_t \left[  \TZ_t + \gamma \TZ_t' + \gamma \TZ_t'' \right]
\]
where $\TA_t = \I-\eta_t ( \sI - \gamma \PP^{\pi_\lambda^*})$, $\TZ_t'=(\PP_t-\PP)(  \LM_\lambda \TQ_{t-1}-\LM_\lambda \TQ_\lambda^*)$, and $\TZ_t'' = \PP(\LM_{\lambda}^{\widetilde{\pi}_{t-1}} - \LM_{\lambda}^{\pi_\lambda^*})\TQ_{t-1}$.
Recurring the last equality gives
\begin{equation*}
	\TDe_t
	= \prod_{j=1}^t \TA_{j} \TDe_{0}+  \sum_{j=1}^t \prod_{i=j+1}^t \TA_{i} \eta_j \left( \TZ_j + \gamma \TZ_t' + \gamma \TZ_t''\right).
\end{equation*}
Besides, using the general step size in Assumption~\ref{asmp:lr}, we can show $\frac{1}{\sqrt{T}} \sum_{t=1}^T \EB\|\TDe_t\|_{\infty}^2 \to 0$ (in Theorem~\ref{thm:general-entropy}).

\paragraph{Satisfied Lipschitz condition.}
In order to apply the second and third analysis in Section~\ref{sec:proof}, we only need to show that $\|\TZ_t''\|_{\infty} \le L \|\TDe_{t-1}\|_{\infty}^2$ for an appropriate $L > 0$.
Notice that $\LM_{\lambda}^{\pi_{\lambda}^*} \TQ_{t-1}  \le \LM_{\lambda}^{\widetilde{\pi}_{t-1}}\TQ_{t-1}$ and $\LM_{\lambda}^{\widetilde{\pi}_{t-1}}\TQ_\lambda^*  \le \LM_{\lambda}^{\pi_\lambda^*}\TQ_\lambda^*$ coordinately.
It implies that $\TZ_t'' = \PP(\LM_{\lambda}^{\widetilde{\pi}_{t-1}} - \LM_{\lambda}^{\pi_\lambda^*})\TQ_{t-1}$ satisfies
\[
 \0 \le \TZ_t'' \le \PP \left[(\LM_{\lambda}^{\widetilde{\pi}_{t-1}} - \LM_{\lambda}^{\pi_\lambda^*})\TQ_{t-1}-(\LM_{\lambda}^{\widetilde{\pi}_{t-1}} - \LM_{\lambda}^{\pi_\lambda^*})\TQ_\lambda^*\right]
 = (\PP^{\Tpi_{t-1}}-\PP^{\pi_{\lambda}^*})  \TDe_{t-1}.
\]
Hence, $\|\TZ_t''\|_{\infty} \le \|(\PP^{\Tpi_{t-1}}-\PP^{\pi_{\lambda}^*})  \TDe_{t-1}\|_{\infty} \le \|\PP^{\Tpi_{t-1}}-\PP^{\pi_{\lambda}^*}\|_{\infty} \|\TDe_{t-1}\|_{\infty} \le \|\Pi^{\Tpi_{t-1}}-\Pi^{\pi_{\lambda}^*}\|_{\infty} \|\TDe_{t-1}\|_{\infty}.
$
By definition of $\Pi^\pi$, we know that
\[
\|\PP^{\Tpi_{t-1}}-\PP^{\pi_{\lambda}^*}\|_{\infty}
\le \sup_{s \in \SM} \| \Tpi_{t-1}(\cdot|s) - \pi^*_{\lambda}(\cdot|s)\|_{\infty}.
\]
On the other hand, $\Tpi_{t-1}, \pi_\lambda$ has a closed form in terms of $\TQ_{t-1}$ and $\Q_\lambda^*$ respectively.
Actually, we have that $\Tpi_{t-1}(\cdot|s)  \propto \exp(\widetilde{Q}_{t-1}(s, \cdot)/\lambda)$ and $\pi_{\lambda}^*(\cdot|s)  \propto \exp(Q_\lambda^*(s, \cdot)/\lambda)$.
By the following lemma, we know that $\| \pi_{t-1}(\cdot|s) - \pi_{\lambda}(\cdot|s)\|_{\infty} \le \frac{1}{\lambda} \|Q_\lambda^*(s, \cdot)-\widetilde{Q}_{t-1}(s, \cdot)\|_{\infty}$.
As a result, we have $\|\TZ_t''\|_{\infty} \le   L \|\TDe_{t-1}\|_{\infty}^2$ with $L = \frac{1}{\lambda}$.
\begin{lem}
	\label{lem:softmax}
	For any vector $\v \in \RB^d$, let $\softmax:  \RB^d \to \RB^d$ be defined by $\softmax(\v)(i) = \exp(\v(i))/ \sum_{j \in [d]} \exp(\v(j))$.
	Then,
	$\|\softmax(\v_1)-\softmax(\v_2)\|_{\infty} \le \|\v_1 -\v_2\|_{\infty}$.
\end{lem}
\begin{proof}[Proof of Lemma~\ref{lem:softmax}]
	For any $\v$, it is easy to find that $\softmax(\v) = \frac{\partial L(\v)}{\partial \v}$ where $L(\v) = \log(\sum_{j \in [d]} \exp(\v(j)))$.
	It is easy to show that $\left\|\frac{\partial^2 L(\v)}{\partial^2 \v}\right\|_{\infty} \le 1$ for any $\v$.
	Hence, the result follows from Taylor's expansion.
\end{proof}

The rest proof is almost the same as that in Section~\ref{sec:proof}.
\end{proof}

\subsection{Non-asymptotic Bounds in Proof of Theorem~\ref{thm:fclt-entropy}}
The error decomposition in Appendix~\ref{proof:non-linear-error} still apply here.
Let $\De_t = \TQ_t - \Q_\lambda^*$ for simplicity.
Hence, it follows that
\[
\EB \| \bar{\De}_T - \Q_\lambda^*\|_{\infty}
\le \sum_{i=0}^4 \EB \|\TTM_i\|_{\infty}
\]
where
\begin{gather*}
    \TTM_0  = \frac{1}{\eta_0T}   (\A_0^T-\eta_0 \I)  \TDe_{0},
    \TTM_1 = \frac{1}{T}  \sum_{j=1}^T\G^{-1} \TZ_j,
    \TTM_2 = \frac{1}{T}  \sum_{j=1}^T  (\A_j^T - \G^{-1})\TZ_j\\
    \TTM_3 =  \frac{\gamma}{T}  \sum_{j=1}^T\A_j^T( \PP_j - \PP) (\LM_\lambda\Q_{j-1} - \LM_\lambda\Q_\lambda^*),
        \TTM_4 = \frac{\gamma}{T}  \sum_{j=1}^T\A_j^T( \PP^{\pi_{j-1}} -\PP^{\pi^*}) \DDelta_{j-1}.
\end{gather*}
Pay attention that the $\A_j^T$ used above depends on $\pi_\lambda^*$ rather than $\pi^*$ now.
As argued in last subsection, Assumption~\ref{asmp:gap} is satisfied here with $L = \frac{1}{\lambda}$.

The remaining thing are to repeat what we have done in Appendix~\ref{proof:five-term}, analyzing each term $\TTM_i$'s using non-asymptotic concentration inequalities.
There are some important aspects to notice.
First, for any $j$, $\|\TZ_j\|_{\infty} \le 2(1+\gamma \|\LM_\lambda \Q_\lambda^*\|_{\infty}) \le  2(1+\gamma \|\Q_\lambda^*\|_{\infty} + \gamma \lambda \mathrm{Entropy}(\pi_\lambda^*)) \le \frac{1+\lambda\log\frac{1}{|\AM|}}{1-\gamma} = \TOM(\frac{1}{1-\gamma})$
where we use $\mathrm{Entropy}(\pi_\lambda^*) \le \log \frac{1}{|\AM|}$ and $\|\Q^*-\Q_\lambda^*\|_{\infty} \le \frac{\lambda }{1-\gamma}\log\frac{1}{|\AM|}$ (which is proved in Theorem 5 of~\citep{yang2019regularized}).
Second, the properties of $\A_j^T$'s in Lemma~\ref{lem:AjT-bound} and~\ref{lem:G-poly} still hold with the same parameters.
Finally, we have a counterpart Theorem~\ref{thm:Linfty-pw2} due to Theorem 1 in~\citep{wainwright2019stochastic} also holds here.
The possible difference is that $\|\TZ_j\|_{\infty}$ is bounded $\frac{\lambda }{1-\gamma}\log\frac{1}{|\AM|}$ instead of $\frac{1}{1-\gamma}$, which is equivalent up to log factors.
Hence, up to log factors, Theorem~\ref{thm:Linfty-pw2} also holds for entropy regularized Q-Learning. 
Putting pieces together, we complete the proof.
\begin{figure}[t!]
	\centering
	\includegraphics[width=0.5\columnwidth]{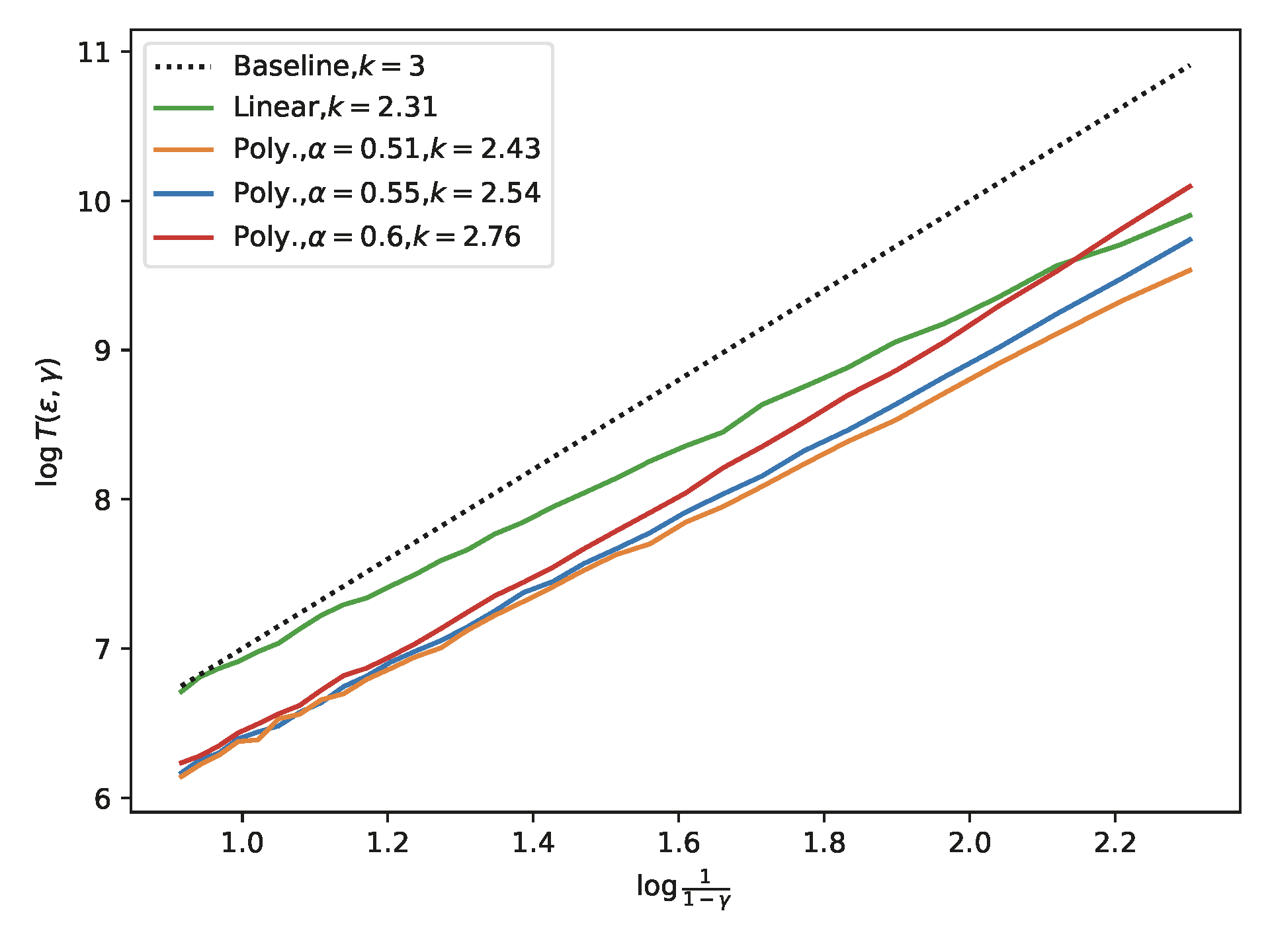} \hspace{-0.1in}
		\includegraphics[width=0.5\columnwidth]{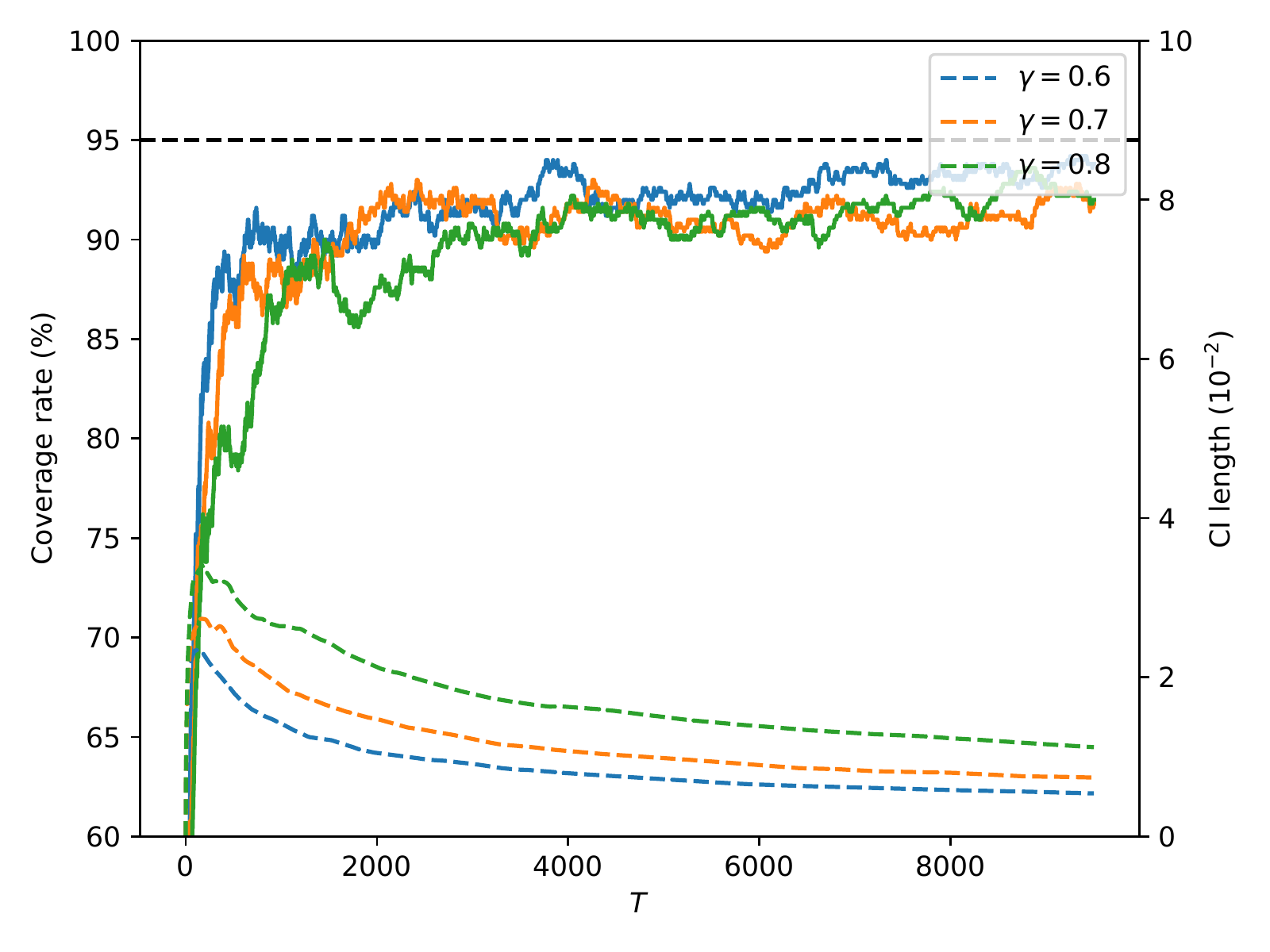}
	\caption{Left: log-log plots of the sample complexity $T(\varepsilon, \gamma)$ versus the discount complexity parameter $(1-\gamma)^{-1}$.
	Right: the coverage rate and the average length of the 95\% confidence interval for regularized Q-Learning.
		%		 (left) and versus the discount complexity parameter $(1-\gamma)^{-1}$ (right).
	}
\label{fig:2}
\end{figure}

\section{DETAILS OF EXPERIMENTS}
\label{appen:exp}

\paragraph{The setup of MDP.}
According to Theorem~\ref{thm:con-linear}, for sufficiently small error $\varepsilon>0$, we expect the sample complexity $T(\varepsilon, \gamma)$ is always upper bounded by $\|\diag(\VQ)\|_{\infty}$ and $\frac{1}{(1-\gamma)^3}$ at a worst case. 
To ensure Assumption~\ref{asmp:gap}, we consider a random MDP.
In particular, for each $(s,a)$ pair, the random reward $R(s,a)\sim\UM(0,1)$ is the uniformly sampled from $(0,1)$ and the transition probability $P(s'|s,a)=u(s')/\sum_{s}u(s)$, where $u(s)\overset{i.i.d.}{\sim}\UM(0,1)$. 
The size of the MDP we choose is $|\SM|=4$, $|\AM|=3$.
We consider 30 different values of $\gamma$ equispaced between $0.6$ and $0.9$. 
For a given $\gamma$, we run Q-learning algorithm for $10^5$ steps (which already ensures convergence) and repeat the process independently for $10^3$ times. 
Finally, we average the $\ell_\infty$ error $\|\bar{\Q}_T-\Q^*\|_{\infty}$ of the $10^3$ independent trials as an approximation of $\EB\|\bar{\Q}_T-\Q^*\|_{\infty}$ and compute $T(\varepsilon, \gamma)$ by definition.
The polynomial step size $\eta_t=t^{-\alpha}$ uses $\alpha\in\{0.51, 0.55, 0.60\}$ and the resacled linear step size is $\eta_t=(1+(1-\gamma)t)^{-1}$.
In Figure~\ref{fig:exp}, we choose $\varepsilon=e^{-4}$ and plot the results on a log-log scale. 
We then plot the least-squares fits through these points and the slopes of these lines are also provided in the legend.
%In particular, we run linear regression of pairs $\left(\log\frac{1}{1-\gamma_i}, \log T(\varepsilon, \gamma_i)\right)$ and obtain the coefficient $k$ on $\log\frac{1}{1-\gamma}$ in the legend of Figure~\ref{fig:exp}. 

\paragraph{Confirming the theoretical predictions.}
In the body, we show the least-squares fits through the points $\{ (\log\|\diag(\VQ)\|_\infty, \log T(\varepsilon, \gamma)) \}_{\gamma \in \Gamma}$.
As a complementary, we also show the fits through $\{ (\log(1-\gamma)^{-1}, \log T(\varepsilon, \gamma)) \}_{\gamma \in \Gamma}$ in Figure~\ref{fig:2}.

\paragraph{Online inference experiments.}
We visualize the empirical coverage rate and confidence interval lengths of averaged Q-Learning in Figure~\ref{fig:coverage}.
We use the random scaling method (Algorithm 1 in~\citep{lee2021fast}) to compute the weighting matrix $\W_T \in \RB^{D \times D}$ where $\W_T = \int_0^1 \bar{\Bphi}_T(r) \bar{\Bphi}_T(r)^\top \mathrm{d} r$ and $\bar{\Bphi}_T(r)={\Bphi}_T(r)- r \cdot {\Bphi}_T(1)$.
We focus on the inference of the optimal value function on the first state $s_0$ and the first action $a_0$, i.e., $Q^*(s_0, a_0)$.
We use $10^4$ steps of value iteration to compute the optimal value function $\Q^*$.
From~\citep{lee2021fast,li2021statistical}, the asymptotic confidence interval is given by
\[
\left[ \bar{Q}_T(s_0, a_0) - 6.753 \sqrt{\frac{\W_T((s_0, a_0), (s_0, a_0))}{T}}, \bar{Q}_T(s_0, a_0) + 6.753 \sqrt{\frac{\W_T((s_0, a_0), (s_0, a_0))}{T}} \right].
\]
We set $T=10^4$ and discard the first $5\%$ samples as a warm-up.
This warm-up is quite important; otherwise $\W_T$ would change rapidly (as a result of fast convergence of $\Q_T$) and deteriorate the performance.
The performance is measured by two statistics: the coverage rate and the average length of the 95\% confidence interval.
We also provide similar results for regularized Q-Learning in Figure~\ref{fig:2}.

\end{appendix}

%
%\newpage

\end{document}